%% file: arXiv2.tex
\documentclass[12pt]{article}

\usepackage[margin=1in]{geometry}

\usepackage{setspace}

\usepackage{graphicx}

\usepackage{enumitem}

\usepackage{microtype}
\usepackage{subfigure}
\usepackage{booktabs}
\usepackage{comment}
\usepackage{wrapfig}
\usepackage{arydshln}
\usepackage{enumitem}
\usepackage{multirow}
\usepackage{url}
\usepackage{natbib}
\usepackage{authblk}

\RequirePackage[colorlinks,citecolor=blue,linkcolor=blue,urlcolor=blue,pagebackref]{hyperref}

\usepackage{amsmath}
\usepackage{amssymb}
\usepackage{mathtools}
\usepackage{amsfonts}
\usepackage{bm}
\usepackage{bbm}

\input{Macros/algorithm.tex}
\input{Macros/math.tex}

\input{Macros/theorems.tex}
\input{Macros/misc.tex}

\usepackage[capitalize,noabbrev]{cleveref}

\allowdisplaybreaks

\input{Head/Title}

\begin{document}

\maketitle

\input{Head/Abstract}

\input{Main2.tex}

\end{document}

%% file: Macros/algorithm.tex
\usepackage{algorithm}
\usepackage{algorithmic}

%% file: Macros/math.tex
\def\eqref#1{equation~\ref{#1}}

\def\1{\bm{1}}

\def\rmC{{\mathrm{C}}}

\def\rmT{{\mathrm{T}}}
\def\rmU{{\mathrm{U}}}

\def\rmd{{\mathrm{d}}}

\def\rmp{{\mathrm{p}}}

\DeclareMathAlphabet{\mathsfit}{\encodingdefault}{\sfdefault}{m}{sl}
\SetMathAlphabet{\mathsfit}{bold}{\encodingdefault}{\sfdefault}{bx}{n}

\def\calC{{\mathcal{C}}}
\def\calD{{\mathcal{D}}}
\def\calE{{\mathcal{E}}}

\def\calH{{\mathcal{H}}}
\def\calI{{\mathcal{I}}}
\def\calJ{{\mathcal{J}}}
\def\calK{{\mathcal{K}}}
\def\calL{{\mathcal{L}}}

\def\calP{{\mathcal{P}}}

\def\calT{{\mathcal{T}}}

\def\calX{{\mathcal{X}}}

\def\bbE{{\mathbb{E}}}

\def\bbP{{\mathbb{P}}}

\def\bbR{{\mathbb{R}}}

\DeclareMathOperator*{\argmin}{arg\,min}

\newcommand{\p}[1]{\left(#1\right)}
\newcommand{\sqb}[1]{\left[#1\right]}
\newcommand{\cb}[1]{\left\{#1\right\}}

\newcommand{\Bigp}[1]{\Big(#1\Big)}
\newcommand{\Bigsqb}[1]{\Big[#1\Big]}

\newcommand{\Biggp}[1]{\Bigg(#1\Bigg)}
\newcommand{\Biggsqb}[1]{\Bigg[#1\Bigg]}

\newcommand{\bigp}[1]{\big(#1\big)}
\newcommand{\bigsqb}[1]{\big[#1\big]}
\newcommand{\bigcb}[1]{\big\{#1\big\}}

\newcommand{\indep}{\rotatebox[origin=c]{90}{$\models$}}

%% file: Macros/theorems.tex
\usepackage{amsthm}

\theoremstyle{plain}

\newtheorem{theorem}{Theorem}[section]
\newtheorem{lemma}[theorem]{Lemma}

\newtheorem{corollary}[theorem]{Corollary}

\newtheorem{assumption}[theorem]{Assumption}

\newtheorem*{remark}{Remark}

%% file: Macros/misc.tex
\usepackage[textsize=tiny]{todonotes}
\usepackage{multirow}

\renewcommand{\eqref}[1]{(\ref{#1})}

\usepackage{xcolor}
\newcount\Comments  
\newcommand{\kibitz}[2]{\ifnum\Comments=1\textcolor{#1}{#2}\fi}

%% file: Head/Title.tex
\title{PUATE: Efficient ATE Estimation\\
from Treated (Positive) and Unlabeled Units}

\author{Masahiro Kato\thanks{Email: \texttt{mkato-csecon@g.ecc.u-tokyo.ac.jp}}$\,$}
\author{Fumiaki Kozai}
\author{Ryo Inokuchi}

\affil{Mizuho-DL Financial Technology, Co., Ltd.}

\date{\today}

%% file: Head/Abstract.tex
\begin{abstract}
The estimation of \emph{average treatment effects} (ATEs), defined as the difference in expected outcomes between treatment and control groups, is a central topic in causal inference. This study develops semiparametric efficient estimators for ATE in a setting where only a treatment group and an unlabeled group—consisting of units whose treatment status is unknown—are observed. This scenario constitutes a variant of learning from positive and unlabeled data (\emph{PU learning}) and can be viewed as a special case of ATE estimation with missing data. For this setting, we derive the \emph{semiparametric efficiency bounds}, which characterize the lowest achievable asymptotic variance for regular estimators. We then construct semiparametric \emph{efficient ATE estimators} that attain these bounds. Our results contribute to the literature on causal inference with missing data and weakly supervised learning.
\end{abstract}

%% file: Main2.tex
\section{Introduction}
\label{sec:introduction}
The estimation of \emph{average treatment effects} (ATEs), defined as the difference in expected outcomes between treatment and control groups, is a fundamental problem in causal inference \citep{Imbens2015causalinference}. Estimating ATEs enables researchers to quantify the causal impact of a treatment, intervention, or policy on an outcome of interest. This problem is of paramount importance across various fields, including economics, epidemiology, and the computer sciences.

Standard ATE estimation typically assumes access to both treatment and control groups, along with complete information on treatment assignment. However, in many practical situations, this assumption does not hold. In some cases, only a \emph{treatment group} and an \emph{unknown group}—comprising units for which treatment assignment is unobserved—are available. Such scenarios arise in various applications, including recommendation systems with implicit feedback, electronic health records, and marketing campaigns, where the absence of explicit treatment information poses significant challenges for causal inference.

For example, in a recommendation system, if a user buys a product from a website, we can infer that the user visited the site. However, if the purchase occurs in a physical store, we cannot determine whether the user visited the website. If building an online website is regarded as the treatment, this situation implies that users who bought the product in-store belong to an unknown group.

\paragraph{Content of this study.}
We address the problem of ATE estimation using only a treatment group and an unknown group. This setting is closely related to learning from positive and unlabeled data \citep[\emph{PU learning},][]{Sugiyama2022machinelearning}, where the goal is to train a classifier using only positive and unlabeled instances. In our context, the challenge lies in efficiently estimating ATEs using the treatment (positive) and unknown groups. We refer to our setting and methodology as \emph{PUATE}.

For this problem, we first derive \emph{semiparametric efficiency bounds}, which are theoretical lower bounds on the asymptotic variance of regular estimators under the given data-generating processes (DGPs).\footnote{For regular estimators, see p.366 in \citet{Vaart1998asymptoticstatistics}.} These bounds serve as benchmarks for evaluating estimator performance. As part of this derivation, we compute the \emph{efficient influence function}, which provides insight into the construction of \emph{efficient estimators}.

Using the efficient influence function, we develop semiparametric efficient ATE estimators that are $\sqrt{n}$-consistent and whose asymptotic variance achieves the efficiency bounds. These estimators are thus optimal under the semiparametric framework.

In this study, we consider two DGPs relevant to the PU setup: the censoring setting and the case-control setting \citep{Elkan2008learningclassifiers,duPlessis2015convexformulation}. In the censoring setting, we are given a single dataset in which some units have missing treatment information while others are confirmed to have received treatment. In the case-control setting, we are provided with two datasets: one containing treated units and another comprising units with unknown treatment status.

Specifically, our contributions are as follows:
\begin{itemize}[noitemsep, topsep=0pt, leftmargin=0.40cm]
    \item We formulate the ATE estimation problem with missing data using the PU learning framework.
    \item We derive efficiency bounds under both the censoring and case-control settings.
    \item We propose novel efficient estimators.
    \item We establish connections between ATE estimation with missing data and PU learning.
    \item We also propose alternative candidate estimators.
\end{itemize}

This study is organized as follows. Section~\ref{sec:problem} formulates our problem. Section~\ref{sec:exampleate} introduces simple ATE estimators, which are later shown to be inefficient. In Section~\ref{sec:censoring}, we derive efficiency bounds, propose an efficient estimator, and establish its asymptotic properties under the censoring setting. Due to the space limitation, we show the ATE estimation in the case-control setting briefly in Section~\ref{sec:case-control} and mainly in Appendix~\ref{appdx:case-control}. Section~\ref{sec:experiments} presents simulation studies. We introduce related work in Appendix~\ref{appdx:related}. The details of PU learning methods in Appendix~\ref{appdx:pulearning}


\section{Problem setting}
\label{sec:problem}

\subsection{Potential outcomes and parameter of interest}
We consider binary treatments, $1$ and $0$. For each treatment $d \in \{1, 0\}$, there exists a potential outcome $Y(d) \in \bbR$. The outcome is observed only when the corresponding treatment is assigned to a unit. Each unit has $p$-dimensional covariates $X \in \calX \subset \bbR^p$. This setting is called the Neyman-Rubin causal model \citep{Neyman1923surapplications,Rubin1974estimatingcausal}. 

We denote by $P_0 \in \calP$ a distribution of $Y(d)$, $X$, and other random variables introduced below, which we call the true distribution, where $\calP$ is the set of distributions. For simplicity, we assume that $P_0$ has a density. We denote the conditional density of $Y(d)$ given $X$ by $p_{Y(d), 0}(y(d) \mid X)$, and the marginal density of $X$ by $\zeta_0(x)$. 

We consider $n$ units, each of which receives either treatment or control. Let $Y_i(d)$ and $X_i$ be i.i.d. copies of $Y(d)$ and $X$. Throughout this study, for a random variable $R$, let $R_i$ be its i.i.d. copy under $P_0$. If unit $i$ receives treatment $d$, we observe $Y_i(d)$ but not the counterfactual outcome.

We refer to the group of units who receive treatment $1$ as the \emph{treatment group} and those who receive treatment $0$ as the \emph{control group}. As formulated in the next subsection, we consider a scenario where part of the treatment group and a \emph{mixture} of the treatment and control groups are observable, where the treatment indicator is unobservable. We refer to this mixed group as the \emph{unknown group}. 

Our objective is to estimate the ATE under $P_0$ using observed data, defined as $\tau_0 \coloneqq \bbE\sqb{Y(1) - Y(0)}$, 
where $\bbE$ denotes the expectation under $P_0$.

\subsection{Observations with two DGPs}
In our setting, the observations are non-standard. We can only observe part of the treatment group and the unknown group, a mixture of the treatment and control groups. This setting is a variant of PU learning. PU learning encompasses two settings: the \emph{censoring setting} and the \emph{case-control setting}. In the censoring setting, we consider a single dataset with i.i.d. observations, where treatment labels contain missing values. Specifically, while part of the treatment group is observed, the mixture of the treated and control groups is also present. In the case-control setting, we observe two independent datasets: one consisting solely of the treatment group and the other comprising the unknown group. The censoring and case-control settings are also referred to as \emph{one-sample} and \emph{two-sample} settings, respectively \citep{Niu2016theoreticalcomparisons}. The case-control setting can also be regarded as a form of stratified sampling, studied in \citet{Imbens1996efficientestimation} and \citet{Wooldridge2001asymptoticproperties}.  

\subsection{Censoring setting}
\label{sec:censoring_dgp}
In the censoring setting, we observe a single dataset $\calD$, defined as follows: 
\begin{align*}
    \calD \coloneqq \cb{\bigp{X_i, O_i, Y_i}}^n_{i=1}\ \ \text{with}\ \ \bigp{X_i, O_i, Y_i} \sim p_0(x, o, y),
\end{align*}
where $O_i \in \{1, 0\}$ is an observation indicator with the \emph{observation probability} $\pi_0(O \mid X)$, $Y_i$ is defined as
\[Y_i \coloneqq \mathbbm{1}[O_i = 1] Y_i(1) + \mathbbm{1}[O_i = 0] \widetilde{Y},\] 
$\widetilde{Y}_i$ is the observation of the unknown groups and defined as 
\[\widetilde{Y}_i \coloneqq \mathbbm{1}[D_i = 1] Y_i(1) + \mathbbm{1}[D_i = 0] Y_i(0),\] 
$D_i \in \{1, 0\}$ is a (unobserved) treatment indicator, and we denote the conditional probability of $D_i$ given $X$ and $O_i = 0$ as $g_0(D \mid X) = \bbP(D \mid X, O = 0)$. We refer to $g_0(d \mid X)$ as the \emph{censoring propensity score}. 
Here, the density $p_0(x, o, y)$ is given as $p_0(x, o, y) = \zeta_0(x)\pi_0(o \mid x)p_{Y,0}(y \mid x)$, 
where $p_{Y, 0}$ is the density of $Y$.

\subsection{Case-control setting}
\label{sec:case_control_dgp}
In the case-control setting, we observe two stratified datasets, $\calD_{\rmT}$ and $\calD_{\rmU}$: 
\begin{align*}
    &\calD_{\rmT} \coloneqq \cb{\bigp{X_{\rmT, j}, Y_j(1)}}^m_{j=1}\ \ \text{with}\ \ \bigp{X_{\rmT, j}, Y_j(1)} \sim p_{\rmT, 0}(x, y(1))\ \ \mathrm{and}\\
    &\calD_{\rmU} \coloneqq \cb{\bigp{X_k, Y_{\rmU, k}}}^l_{k=1}\ \ \text{with}\ \ \bigp{X_k, Y_{\rmU, k}} \sim p_{\rmU, 0}(x, y_{\rmU}),
\end{align*}
where $m$ and $l$ are fixed sample sizes of each dataset such that $m + l = n$, $X_{\rmT, j}$ represents the covariates of the treatment group, $Y_{\rmU, k}$ is the observed outcome defined as 
\[Y_{\rmU, k} = \mathbbm{1}[D_k = 1]Y_k(1) + \mathbbm{1}[D_k = 0]Y_k(0),\] 
and  $D_k \in \{1, 0\}$ is a treatment indicator with probability $e_0(D \mid X)$. We refer to $e_0(d \mid X = x)$ as the \emph{case-control propensity score}.
The densities $p_{\rmT, 0}(x, y(1))$ and $p_{\rmU, 0}(x, y_{\rmU})$ satisfy $p_{\rmT, 0}(x, y(1))= \zeta_{\rmT, 0}(x)p_{Y(1), 0}(y(1)\mid X = x)$ and $p_{\rmU, 0}(x, y_{\rmU}) = \zeta_0(x)p_{Y_\rmU, 0}(y_{\rmU}\mid X = x)$ respectively, 
where $p_{Y_\rmU, 0}(y_{\rmU} \mid X = x)$ denotes the density of $Y_{\rmU, k}$ given $X = x$, and 
$\zeta_{\rmT, 0}(x)$ represents the density of the covariates $X$ in the treatment group.

Although the ATE estimation in the case-control setting is also crucially important, due to the space limitation, we show the main results almost in Appendix~\ref{sec:case-control}. 

\subsection{Difference between the two settings}

\begin{figure}[t]
 \centering
 \vspace{-5mm}
   \includegraphics[width=0.7\textwidth]{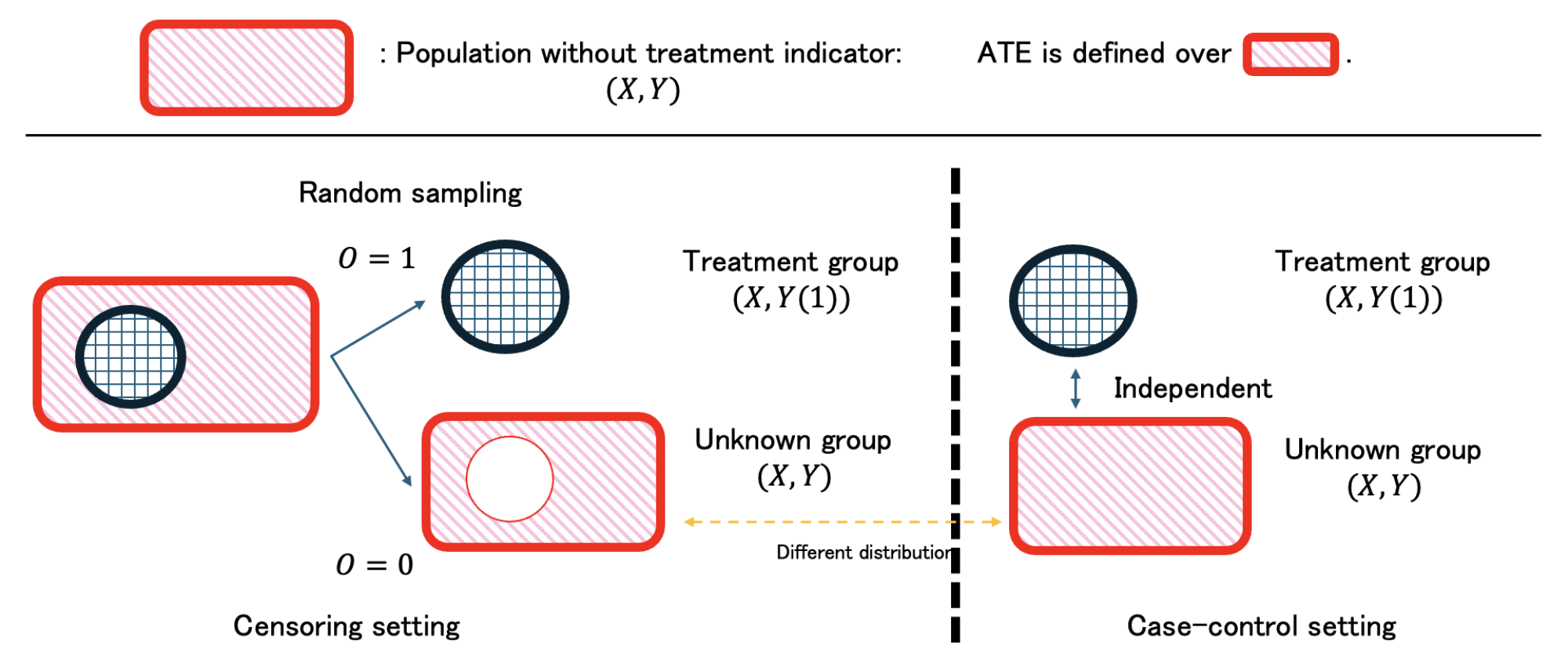}
 \vspace{-3mm}
 \caption{Illustration of the censoring and case-control settings}
 \vspace{-2mm}
 \label{fig:cens_case_control}
\end{figure}

We illustrate the concept of the censoring and case-control settings in Figure~\ref{fig:cens_case_control}. 
A summary of the differences is provided below:
\begin{description}[noitemsep, topsep=0pt, leftmargin=0.40cm]
    \item[Censoring setting:] A single dataset is observed, containing partial treatment information and a mixture of treated and control groups.
    \item[Case-control setting:] Two stratified datasets are observed—one consisting of the treatment group and the other comprising the unknown group.
\end{description}
The key distinction lies in the randomness of treatment group observations. In the censoring setting, the observation of the treatment group is a random event, where the observation indicator $O_i$ follows a probability $P(O\mid X)$. In contrast, in the case-control setting, the label observation is deterministic, and the treatment and unknown groups are drawn independently. This difference impacts the estimator design and efficiency bounds.

Note that the definitions and meanings of the propensity scores in the censoring and case-control settings are also different, and its difference stems from the definition of the observations indicators.

\paragraph{Notation.}
We summarize the notations above and introduce new notations. Let $\bbE$, $\bbP$, and $\mathrm{Var}$ be an expectation operator, a probability law, and a variance operator. For both settings, let us define $\mu_{\rmT, 0}(X) \coloneqq \bbE[Y(1)\mid X]$, $\mu_{\rmC, 0}(X) \coloneqq \bbE[Y(0)\mid X]$, and let $\tau_0(X) \coloneqq \mu_{\rmT, 0}(X) - \mu_{\rmC, 0}(X)$ be the conditional ATE. If a function $f$ depends on the true distribution $P_0$, we denote it by $f_0$

In the censoring setting, we use $\pi_0(o \mid X) \coloneqq \bbP(O = o\mid X)$, $g_0(d\mid X) \coloneqq \bbP(D = d\mid X, O = 0)$, and $\nu_0 \coloneqq \bbE[\widetilde{Y}\mid X, O = 0]$. Here, $\bbE[\mathbbm{1}[O = 1]Y\mid X] = \pi_0(O = 1 \mid X)\mu_{\rmT, 0}(X)$ and $\bbE[\mathbbm{1}[O = 0]Y\mid X] = \pi_0(0 \mid X)\nu_0(X) = \pi_0(0 \mid X)g_0(1\mid X)\mu_{\rmT, 0}(X) + \pi_0(0 \mid X)g_0(0\mid X)\mu_{\rmC, 0}(X)$ hold under Assumption~\ref{asm:unconfoundedness_censoring}, defined later.

In the case-control setting, we use $e_0(d\mid X) \coloneqq \bbP(D = d \mid X)$ and $\mu_{\rmU, 0}(X) \coloneqq \bbE[Y_{\rmU}\mid X]$. Here, $\mu_{\rmU, 0}(X) = e_0(1\mid X)\mu_{\rmT, 0}(X) + e_0(0\mid X)\mu_{\rmC, 0}(X)$ holds under Assumption~\ref{asm:unconfoundedness_casecontrol}. Let $r_0(X) \coloneqq \frac{\zeta_0(X)}{\zeta_{\rmT, 0}(X)}$ be the density ratio between the covariate densities. 


\section{Example of ATE estimators in the censoring setting}
\label{sec:exampleate}
In this section, as a preliminary, we suggest a simple ATE estimator in the censoring. In Appendix~\ref{appdx:case-control}, we show the detailed result of the estimator and also propose a simple ATE estimator in the case-control setting. Note that as discussed in the subsequent subsections, the estimators are not efficient, i.e., there exist ATE estimators whose asymptotic variance is smaller. 

\subsection{The Inverse probability weighting estimator}
We first consider the censoring setting. We begin by the arguments from the estimation of the propensity score. To estimate the propensity score, we employ the result in PU learning. \citet{Elkan2008learningclassifiers} addresses PU learning in the censoring setting. In that work, they show that under the Selected Completely At Random (SCAR) assumption defined below \citep{Elkan2008learningclassifiers}. This assumption is analogy of the Missing Completely A Random assumption (MCAR), which is common with the
missing data literature \citep{Little2002statisticalanalysis,Rubin1974estimatingcausal,Bekker2020learningfrom}. Several works such as \citet{Bekker2018learningfrom} attempt to relax the SCAR assumption, but to relax the assumption, we usually require additional assumptions.

\begin{assumption}[SCAR]
\label{asm:pu_learning_censoring}
It holds that $\bbP(D = 1, O = d \mid X) = \bbP(D = 1 \mid X) \bbP(O = d \mid D = 1)$. 
\end{assumption}

From the assumption, we have $\pi_0(1\mid x) = \bbP(D = 1 \mid X) \bbP(O = d \mid D = 1)$ since $\bbP(D = 0, O = d\mid X) = 0$ holds by definition of the DGP.
Thus, under this assumption, the propensity scores $g_0(d \mid x)$ and $e_0(d \mid x)$ can be estimated using PU learning methods \citep{duPlessis2015convexformulation,Elkan2008learningclassifiers}. If we know their true values, such assumptions are unnecessary. Note that the censoring PU learning studies aim to estimate $\bbP(D = d\mid X)$ not $g_0(d \mid X) = \bbP(D = d\mid X, O = 0)$. However, once we obtain an estimate of $\bbP(D = 1\mid X)$ and $\pi_0(0\mid X)$, we can obtain $g_0(1\mid X)$ from $g_0(1\mid X) = \bbP(O = 0 \mid D = 1)\bbP(D = 1\mid X) / \pi_0(0 \mid X)$.

We further make the following unconfoundedness and common support assumptions, which are common in ATE estimation. 
\begin{assumption}[Unconfoundedness in the censoring setting]
\label{asm:unconfoundedness_censoring}
The potential outcomes $(Y(1), Y(0))$ are independent of treatment assignment given covariates: $(Y(1), Y(0)) \indep (O, D) \mid X$.      
\end{assumption}

\begin{assumption}[Common support in the censoring setting]
\label{asm:overlap_censoring}
There exists a constant $c$ independent of $n$ such that for all $x \in \calX$, $\pi_0(d\mid x), g_0(d\mid x), \zeta_{\rmT, 0}(x), \zeta_0(x) > c$ hold.  
\end{assumption}

Under these assumptions, the ATE $\tau_0$ is estimable by replacing the following two quantities with sample approximation: $\bbE[Y(1)] = \bbE\sqb{\frac{\mathbbm{1}[O = 1]Y}{\pi_0(1 \mid X)}}$ and $\bbE[Y(0)] = \bbE\sqb{\frac{\mathbbm{1}[O = 0]Y}{g_0(0\mid X)\pi_0(0 \mid X)}} - \bbE\sqb{\frac{g_0(1\mid X)\mathbbm{1}[O = 1]Y}{g_0(0\mid X)\pi_0(1 \mid X)}}$, 
where $g_0$ and $\pi_0$ can be estimated using observations, and expectations can be approximated by sample averages. Such an estimator is a variant of the inverse probability weighting (IPW) estimator and shown in Remark~\ref{rem:IPW}.

\subsection{Toward efficient estimators}
Thus, we can estimate the ATE in the censoring setting (and in the case-control setting, as shown in Appendix~\ref{appdx:example_case_control}). However, it is unclear whether it is efficient; that is, the (asymptotic) variance is sufficiently small. In the subsequent subsections, we investigate efficient estimators and develop efficiency bounds, which work as a lower bound for the regular estimators. The efficiency bound also suggest the construction of efficient estimators, and we actually propose an estimator whose asymptotic variance aligns with the efficiency bound.

\section{Semiparametric efficient ATE estimation under the censoring setting}
\label{sec:censoring}
This section presents a method for ATE estimation under the censoring setting. First, we derive the efficiency bound in Section~\ref{sec:efficient_inf_censoring}. Then, we propose our estimator in Section~\ref{sec:semiparame_efficient_censoring} and show the asymptotic normality in Section~\ref{sec:asymp_prop}. Finally, in Section~\ref{sec:unknown_prop}, we discuss issues related to the estimation of the propensity score. 

\subsection{Efficient influence function and efficiency}
\label{sec:efficient_inf_censoring}
First, we derive the efficiency bound for regular estimators, which provides a lower bound on asymptotic variances. The efficiency bound is characterized via the efficient influence function \citep{Vaart1998asymptoticstatistics}, derived as follows (Proof is provided in Appendix~\ref{appdx:eb_censoring}):

\begin{lemma}
\label{lem:efficiency_bound}
    If Assumptions~\ref{asm:unconfoundedness_censoring}--\ref{asm:overlap_censoring}, then the efficient influence function is given as $\Psi^{\mathrm{cens}}(X, O, Y; \mu_{\rmT, 0}, \nu_0, \pi_0, g_0, \tau_0)$, where 
    \begin{align*}
    &\Psi^{\mathrm{cens}}(X, O, Y; \mu_{\rmT}, \nu_0, \pi_0, g_0, \tau_0) \coloneqq  S^{\mathrm{cens}}(X, O, Y; \mu_{\rmT, 0}, \nu_0, \pi_0, g_0) - \tau_0,\\
    &S^{\mathrm{cens}}(X, O, Y; \mu_{\rmT, 0}, \nu_0, \pi_0, g_0) \coloneqq \frac{\mathbbm{1}[O = 1]\bigp{Y - \mu_{\rmT, 0}(X)}}{\pi_0(1\mid X)} - \frac{\mathbbm{1}[O = 0]\bigp{Y - \nu_0(X)}}{g_0(0\mid X)\pi_0(0\mid X)}\\
    &\ \ \ +\frac{g_0(1\mid X)\mathbbm{1}[O = 1]\bigp{Y - \mu_{\rmT, 0}(X)}}{g_0(0\mid X)\pi_0(1\mid X)} +  \mu_{\rmT, 0}(X) - \frac{1}{g_0(0\mid X)}\nu_0(X) + \frac{g_0(1\mid X)}{g_0(0\mid X)}\mu_{\rmT, 0}(X). 
    \end{align*}
\end{lemma}
Here, note that the efficient influence function depends on unknown $\mu_{\rmT, 0}, \nu_0, \pi_0, g_0$, which are referred to as \emph{nuisance parameters}. Since the efficient influence function satisfies the equation $\bbE\sqb{\Psi^{\mathrm{cens}}(X, O, Y; \mu_{\rmT, 0}, \nu_0, \pi_0, g_0, \tau_0)} = 0$, if the nuisance parameters are known and the exact expectation is computed, we can obtain $\tau_0$ by solving for $\tau_0$ that satisfies $\bbE\sqb{\Psi^{\mathrm{cens}}(X, O, Y; \mu_{\rmT, 0}, \nu_0, \pi_0, g_0, \tau_0)} = 0$. Thus, the efficient influence function provides significant insights for constructing an efficient estimator. Furthermore, the accuracy of the estimation of the nuisance parameters affects the estimation of $\tau_0$, the parameter of interest.

From Theorem~25.20 in \citet{Vaart1998asymptoticstatistics}, Lemma~\ref{lem:efficiency_bound} yields the following result about the efficiency bound.

\begin{theorem}[Efficiency bound in the censoring setting]
If Assumptions~\ref{asm:unconfoundedness_censoring}--\ref{asm:overlap_censoring}, then the asymptotic variance of any regular estimator is lower bounded by 
\begin{align*}
    V^{\mathrm{cens}} &\coloneqq \bbE\sqb{\Psi^{\mathrm{cens}}(X, O, Y; \mu_0, \nu_0, \pi_0, g_0)^2}\\
    &= \bbE\sqb{\p{1 - \frac{g_0(1\mid X)}{g_0(0\mid X)}}^2\frac{\mathrm{Var}(Y(1)\mid X)}{\pi_0(1\mid X)} + \frac{\mathrm{Var}(\widetilde{Y}\mid X)}{g_0(0\mid X)^2\pi_0(1\mid X)} + \Bigp{\tau_0(X) - \tau_0}^2}.
\end{align*}
\end{theorem}

We say that an estimator is efficient if its asymptotic variance aligns with $V^{\mathrm{cens}}$. 

\subsection{Semiparametric efficient estimator}
\label{sec:semiparame_efficient_censoring}
Based on the efficient influence function, we propose an ATE estimator defined as $\widehat{\tau}^{\mathrm{cens}\mathchar`-\mathrm{eff}}_n \coloneqq \frac{1}{n}\sum^n_{i=1}S^{\mathrm{cens}}(X_i, O_i, Y_i; \widehat{\mu}_{\rmT, n, i}, \widehat{\nu}_{n, i}, \widehat{\pi}_{n, i}, \widehat{g}_{n, i})$, 
where $\widehat{\mu}_{\rmT, n, i}$, $\widehat{\nu}_{n, i}$, $\widehat{\pi}_{n, i}$ and $\widehat{g}_{n, i}$ are estimators of $\mu_{\rmT, 0}$, $\nu_0$, $\pi_0$, and $g_0$. Note that the estimators can depend on $i$. 
This estimator is an extension of the augmented inverse probability weighting estimator, also called a doubly robust estimator \citep{Bang2005doublyrobust}. 


\begin{remark}[Estimation equation]
There exist several intuitive explanations for $\widehat{\tau}^{\mathrm{cens}\mathchar`-\mathrm{eff}}_n$. One of the typical explanations is the one from the viewpoint of the estimation equation. 
Given $\widehat{\mu}_{\rmT, n, i}$, $\widehat{\nu}_{n, i}$, $\widehat{\pi}_{n, i}$, and $\widehat{g}_{n, i}$, the estimator $\widehat{\tau}^{\mathrm{cens}\mathchar`-\mathrm{eff}}_n$ is obtained by solving the following equation: $\frac{1}{n}\sum^n_{i=1}\Psi^{\mathrm{cens}}(X_i, O_i, Y_i; \widehat{\mu}_{\rmT, n, i}, \widehat{\nu}_{n, i}, \widehat{\pi}_{n, i}, \widehat{g}_{n, i}, \widehat{\tau}^{\mathrm{cens}\mathchar`-\mathrm{eff}}_n) = 0$. 
Such a derivation of the efficient estimator as the estimation equation approach \citep{Schuler2024introductionmodern}. 
\end{remark}

\subsection{Consistency and double robustness}
First, we prove the consistency result; that is, $\widehat{\tau}^{\mathrm{cens}\mathchar`-\mathrm{eff}}_n \xrightarrow{\rmp} \tau_0$ holds as $n\to \infty$. We can obtain this result relatively easily compared to the asymptotic normality. We make the following assumption that holds for most estimators of the nuisance parameters.
\begin{assumption}
\label{asm:consistency_censoring}
    As $n \to \infty$, $\big\|\widehat{g}_{n, i} - g_0\big\|_2 = o_p(1)$ holds. Additionally, either of the followings holds for all $i\in \{1,2,\dots,n\}$:
    \begin{itemize}[noitemsep, topsep=0pt, leftmargin=0.40cm]
        \item $\big\|\widehat{\pi}_{n, i} - \pi_0\big\|_2 = o_p(1)$.$\ \ \ \ \ \ \ $ {\labelitemi}$\ $ $\big\|\widehat{\nu}_{n, i} - \nu_0\big\|_2 = o_p(1)$ and $\big\|\widehat{\mu}_{\rmT, n, i} - \mu_{\rmT, 0}\big\|_2 = o_p(1)$.
    \end{itemize}
\end{assumption}
For the estimation of the censoring propensity score, we can employ the existing PU learning methods in the censoring setting, such as \citet{Elkan2008learningclassifiers}. Note that we can also apply methods for the case-control PU learning, such as \citet{duPlessis2015convexformulation}, since, as the classification problem, the case-control setting is more general than the censoring setting \citep{Niu2016theoreticalcomparisons}. Note that the case-control PU learning methods typically require the class prior $\bbP(D = 1)$, which can be estimated under several additional assumptions, even if we do not know it \citep{duPlessis2014classprior,Ramaswamy2016mixtureproportion,Kato2018alternateestimation}.

Then, the following consistency result holds.
\begin{theorem}[Consistency in the censoring setting]
    If Assumptions~\ref{asm:unconfoundedness_censoring}--\ref{asm:overlap_censoring}, and \ref{asm:consistency_censoring} hold, then $\widehat{\tau}^{\mathrm{cens}\mathchar`-\mathrm{eff}}_n \xrightarrow{\rmp} \tau_0$ holds as $n\to \infty$. 
\end{theorem}

\textbf{Double robustness.} There exists double-robustness structure such that given $\big\|\widehat{g}_{n, i} - g_0\big\|_2 = o_p(1)$, if either $\big\|\widehat{\pi}_{n, i} - \pi_0\big\|_2 = o_p(1)$ or $\big\|\widehat{\nu}_{n, i} - \nu_0\big\|_2 = o_p(1)$ and $\big\|\widehat{\mu}_{\rmT, n, i} - \mu_{\rmT, 0}\big\|_2 = o_p(1)$ holds, then $\widehat{\tau}^{\mathrm{cens}\mathchar`-\mathrm{eff}}_n \xrightarrow{\rmp} \tau_0$ holds. Here, note that we need to estimate the propensity score consistently to estimate the ATE and the double robustness holds between the estimators of the observation probability $\pi_0$ and the expected outcomes $\mu_{\rmT, 0}$ and $\nu_0$.\footnote{In the standard setting, the double robustness holds between the estimators of the propensity score and the expected outcome.} This is because, in our setting, the treatment indicator is unobservable. Under this setting, to identify the ATE, we need to use the propensity score and cannot avoid its estimation. 

\subsection{Asymptotic normality}
\label{sec:asymp_prop}
Next, we prove the asymptotic normality. Unlike consistency, we need to make a stronger assumption on the nuisance estimators, especially for the propensity score. 

To establish the asymptotic normality or $\sqrt{n}$-consistency of the estimator, it is necessary to control the complexity of the estimators of the nuisance parameter. One of the simplest approaches is to assume the Donsker condition, but it is well known that the Donsker condition does not hold in several cases, such as high-dimensional regression settings. In such cases, asymptotic normality can still be established through sample splitting, a technique in this field \citep{Klaassen1987consistentestimation}, which has been recently refined by \citet{Chernozhukov2018doubledebiased} as cross-fitting.

\textbf{Cross-fitting.}
Cross-fitting is a variant of sample splitting \citep{Chernozhukov2018doubledebiased}. We randomly partition $\calD$ into $L > 0$ folds (subsamples), and for each fold $\ell \in \calL \coloneqq \{1,2,\dots, L\}$, the nuisance parameters are estimated using all other folds. We estimate $\mu_{\rmT, 0}, \nu_0, \pi_0$, assuming that the propensity score $g_0$ is known. Let us denote the estimators in fold $\ell \in \calL$ as $\widehat{\mu}^{(\ell)}_{\rmT, n}, \widehat{\nu}^{(\ell)}_n, \widehat{\pi}^{(\ell)}_n$. Let $\calI^{(\ell)}$ be the set of the sample index belonging to fold $\ell$. 

Various estimation methods and models can be employed, including neural networks and Lasso, provided they satisfy the convergence rate conditions specified in Assumption~\ref{asm:conv_rate}. We later relax the assumption of a known propensity score. It is important to note that issues related to the propensity score estimation cannot be fully addressed even with cross-fitting. The pseudocode is in Algorithm~\ref{alg:psudo_censoring}. 

\begin{algorithm}[t]
\label{alg:psudo_censoring}
\caption{Cross-fitting in the censoring setting}
\begin{algorithmic}
\STATE \textbf{Input:} Observations $\calD \coloneqq \cb{\bigp{X_i, O_i, Y_i}}^n_{i=1}$, number of folds $L$, and estimation methods for $\mu_{\rmT, 0}, \nu_0, \pi_0$. Let $\calI = \{1, 2, \dots, n\}$ be the index set. 
\STATE Randomly split $\calI$ into $L$ roughly equal-sized folds, $(\mathcal{I}^{(\ell)})_{\ell\in\calL}$. Note that $\bigcup_{\ell \in \calL}\calI^{(\ell)} = \calI$. 
\FOR{$\ell \in \calL$}
    \STATE Set the training data as $\mathcal{I}^{(-\ell)} = \{1,2,\dots,n\} \setminus \mathcal{I}^{(\ell)}$.
    \STATE Construct estimators of nuisance parameters on $\mathcal{I}^{(-\ell)}$, denoted by $\widehat{\mu}^{(\ell)}_{\rmT, n}, \widehat{\nu}^{(\ell)}_n, \widehat{\pi}^{(\ell)}_n$.
\ENDFOR
\STATE \textbf{Output:}  Obtain an ATE estimate $\widehat{\tau}^{\mathrm{cens}\mathchar`-\mathrm{eff}}_{n}$ using $\widehat{\mu}^{(\ell)}_{\rmT, n}$, $\widehat{\nu}^{(\ell)}_n$, and $\widehat{\pi}^{(\ell)}_n$.
\end{algorithmic}
\end{algorithm}

\textbf{Asymptotic normality.}
We describe the results only for the case with cross-fitting, but similar results hold for the case when we assume the Donsker condition. 

We make the following assumptions.

\begin{assumption}
\label{asm:prop_known}
    The propensity score $g_0$ is known ($\widehat{g}_{n, i} = g_0$). 
\end{assumption}

\begin{assumption}
\label{asm:conv_rate}
For each $\ell \in \calL$, as $n\to\infty$, the followings hold:
    \begin{itemize}[noitemsep, topsep=0pt, leftmargin=0.40cm]
        \item $\big\|\pi_0(d\mid X) - \widehat{\pi}^{(\ell)}_n(d\mid X)\big\|_2 = o_p(1)$ for $d\in\{1, 0\}$, $\big\| \mu_{\rmT, 0}(X) - \widehat{\mu}^{(\ell)}_{\rmT, n}(X) \big\|_2 = o_p(1)$, and $\big\| \nu_{0}(X) - \widehat{\nu}^{(\ell)}_{n}(X) \big\|_2 = o_p(1)$. 
        \item $\big\|\pi_0(d\mid X) - \widehat{\pi}^{(\ell)}_n(d\mid X)\big\|_2\big\| \mu_{\rmT, 0}(X) - \widehat{\mu}^{(\ell)}_{\rmT, n}(X) \big\|_2 = o_p(n^{-1/2})$ for $d \in \{1, 0\}$.
         \item $\big\|\pi_0(0\mid X) - \widehat{\pi}^{(\ell)}_n(0\mid X)\big\|_2\big\| \nu_{0}(X) - \widehat{\nu}^{(\ell)}_{n}(X) \big\|_2 = o_p(n^{-1/2})$.
    \end{itemize}
\end{assumption}

Then, we construct the estimator as 
$\widehat{\tau}^{\mathrm{cens}\mathchar`-\mathrm{eff}}_n \coloneqq \frac{1}{n}\sum_{\ell \in \calL}\sum_{i \in \calI^{(\ell)}}S^{\mathrm{cens}}(X_i, O_i, Y_i; \widehat{\mu}^{(\ell)}_{\rmT, n}, \widehat{\nu}^{(\ell)}_n, \widehat{\pi}^{(\ell)}_n, g_0)$ and show the asymptotic normality holds as follows:
\begin{theorem}[Asymptotic normality in the censoring setting]
    \label{thm:asymp_normal_censoring}
    Consider the censoring setting. Suppose that Assumptions~\ref{asm:unconfoundedness_censoring}--\ref{asm:overlap_censoring}, \ref{asm:prop_known}, and \ref{asm:conv_rate} hold; that is, $\widehat{g}_{n, i} = g_0$, and  $\widehat{\mu}_{\rmT, n, i} = \widehat{\mu}^{(\ell)}_{\rmT, n}$, $\widehat{\nu}_{n, i} = \widehat{\nu}^{(\ell)}_n$, and $\widehat{\pi}_{n, i} = \widehat{\pi}^{(\ell)}_n$ are constructed via cross-fitting with certain convergence rates. Then, we have
    \begin{align*}
        \sqrt{n}\p{\widehat{\tau}^{\mathrm{cens}\mathchar`-\mathrm{eff}}_n - \tau_0} \xrightarrow{\rmd} \mathcal{N}(0, V^{\mathrm{cens}})\ \ \text{as}\ \ n\to \infty. 
    \end{align*}
\end{theorem}
The proof is provided in Appendix~\ref{appdx:normal_cens}. 
The asymptotic variance of $\widehat{\tau}^{\mathrm{cens}\mathchar`-\mathrm{eff}}_n$ matches the efficiency bound. Therefore, Theorem~\ref{thm:asymp_normal_censoring} also implies that the estimator $\widehat{\tau}^{\mathrm{cens}\mathchar`-\mathrm{eff}}_n$ is asymptotically efficient. 

We discuss the other candidates of ATE estimators below. 
\begin{remark}[Inefficiency of the Inverse Probability Weighting (IPW) estimator]
\label{rem:IPW}
    We can define the IPW estimator as 
    $\widehat{\tau}^{\mathrm{cens}\mathchar`-\mathrm{IPW}}_n \coloneqq \frac{1}{n}\sum^n_{i=1}\p{\frac{\mathbbm{1}[O_i = 1]Y_i}{\widehat{\pi}_{n, i}(1\mid X_i)} - \frac{\mathbbm{1}[O_i = 0]Y_i}{\widehat{g}_{n, i}(0\mid X_i)\widehat{\pi}_{n, i}(0\mid X_i)} +\frac{\widehat{g}_{n, i}(1\mid X_i)\mathbbm{1}[O_i = 1]Y_i}{\widehat{g}_{n, i}(0\mid X_i)\widehat{\pi}_{n, i}(1\mid X_i)}}$. 
    Compared to our proposed efficient estimator, this estimator does not use the conditional outcome estimators \citep{Horvitz1952generalization}. If $g_0$ and $\pi_0$ are known, this estimator is unbiased. However, it incurs a large asymptotic variance, given as $V^{\mathrm{IPW}} \coloneqq \bbE\Bigsqb{\p{1 - \frac{g_0(1\mid X)}{g_0(0\mid X)}}^2\frac{\bbE[Y(1)^2\mid X]}{\pi_0(1\mid X)}+ \frac{\bbE[\widetilde{Y}^2\mid X]}{g_0(0\mid X)^2\pi_0(1\mid X)}}$. 
    Here, it holds that $V^{\mathrm{IPW}} \geq V^{\mathrm{cens}}$, where the equality holds when $\mu_{\rmT, 0}(x) = 0$ and $\nu_0 = 0$ hold for all $x$. Thus, the IPW estimator is inefficient compared to $\widehat{\tau}^{\mathrm{cens}\mathchar`-\mathrm{eff}}_n$. Additionally, if $\pi_0$ is unknown, the IPW estimator requires more restrictive conditions for the asymptotic normality than $\widehat{\tau}^{\mathrm{cens}\mathchar`-\mathrm{eff}}_n$. 
\end{remark}
\begin{remark}[Direct Method (DM) estimator]
\label{rem:DM}
Another candidate is a DM estimator, defined as 
$
\widehat{\tau}^{\mathrm{cens}\mathchar`-\mathrm{IPW}}_n \coloneqq \widehat{\mu}_{\rmT, n, i}(X) - \frac{1}{\widehat{g}_{n, i}(0\mid X)}\nu(X) + \frac{\widehat{g}_{n, i}(1\mid X)}{\widehat{g}_{n, i}(0\mid X)}\widehat{\mu}_{\rmT, n, i}(X)$,
which is also referred to as a naive plug-in estimator. The asymptotic normality significantly depends on the properties of the estimators $\widehat{\mu}_{\rmT, n, i}$ and $\widehat{g}_{n, i}$. Additionally, the DM method is known to be sensitive to model misspecification.
\end{remark}

\subsection{Unknown propensity score}
\label{sec:unknown_prop}
We have assumed that the propensity score $g_0$ is known. This is because we cannot establish $\sqrt{n}$-consistency even if we assume the Donsker condition or employ cross-fitting if $g_0$ is estimated. However, this assumption can be relaxed by utilizing an additional dataset to estimate $g_0$. 

Several practical scenarios exist. For instance, consider that the following additional dataset is available:
$\calD^{\mathrm{aux}} \coloneqq \cb{\bigp{X_{i'}, O_{i'}}}^{n^{\mathrm{aux}}}_{i'=1},\ \bigp{X_{i'}, O_{i'}} \sim \zeta_0(x)\pi_0(o\mid x)$.

Such a dataset can be less costly since it does not have the outcome data.
Let $\widehat{g}_{n^{\mathrm{aux}}}$ be an estimator obtained from $\calD^{\mathrm{aux}}$ and consider the following assumption:

\begin{assumption}
\label{asm:prop_conv_rate}
It holds that $\|\widehat{g}_{n^{\mathrm{aux}}} - g_0\|_2 = o_p(1)$ as $n^{\mathrm{aux}} \to \infty$.
\end{assumption}

If $n^{\mathrm{aux}}$ approaches infinity independently of $n$, under Assumption~\ref{asm:pu_learning_censoring}, we can establish the asymptotic normality without assuming the propensity score is known.
\begin{corollary}[Asymptotic normality in the censoring setting]
    Consider the censoring setting. Suppose that Assumptions~\ref{asm:unconfoundedness_censoring}--\ref{asm:overlap_censoring}, \ref{asm:conv_rate}, and \ref{asm:prop_conv_rate} hold. Then, it holds that $\sqrt{n}\p{\widehat{\tau}^{\mathrm{cens}\mathchar`-\mathrm{eff}}_n - \tau_0} \xrightarrow{\rmd} \mathcal{N}(0, V^{\mathrm{cens}})$ as $n\to \infty$. 
\end{corollary}

We can also use $\cb{\p{X_i, O_i}}^n_{i=1}$ from $\calD$ to estimate $g_0$ with $\calD^{\mathrm{aux}}$. The inclusion of $\cb{\p{X_i, O_i}}^n_{i=1}$ can improve empirical performance.


Another practical scenario involves an auxiliary dataset with treatment indicators and missing outcomes, given as $\calD^{\mathrm{aux}'} \coloneqq \bigcb{\bigp{X_{i'}, D_{i'}}}^{n^{\mathrm{aux}}}_{i'=1},\  \bigp{X_{i'}, D_{i'}} \sim \zeta_0(x)g_0(o\mid x)$.

\begin{table}[t]
  \centering
  \caption{Experimental results. Left: censoring setting; Right: case‐control setting.}
  \label{tab:table_exp1}

  \begin{minipage}[t]{0.48\textwidth}
    \centering
    \scalebox{0.7}{
      \begin{tabular}{|l|rrr|rrr|}
        \hline
        \multirow{2}{*}{Censoring} & IPW & DM & Efficient & IPW & DM & Efficient \\
         & \multicolumn{3}{|c|}{(estimated $g_0$)}
         & \multicolumn{3}{|c|}{(true $g_0$)} \\ \hline
        MSE        & 0.31 & 0.08 & 0.06 & 0.06 & 0.01 & 0.01 \\
        Bias       &-0.26 & 0.16 & 0.12 &-0.06 & 0.03 & 0.00 \\
        Cov. ratio & 0.95 & 0.07 & 0.78 & 1.00 & 0.09 & 0.93 \\ \hline
      \end{tabular}
    }
  \end{minipage}
  \hfill
  \begin{minipage}[t]{0.48\textwidth}
    \centering
    \scalebox{0.7}{
      \begin{tabular}{|l|rrr|rrr|}
        \hline
        Case- & IPW & DM & Efficient & IPW & DM & Efficient \\
        control & \multicolumn{3}{|c|}{(estimated $e_0$)}
                 & \multicolumn{3}{|c|}{(true $e_0$)} \\ \hline
        MSE        & 10.85 & 0.07 & 0.06 & 0.03 & 0.01 & 0.00 \\
        Bias       &  1.44 & 0.11 & 0.07 & 0.00 & 0.03 & 0.00 \\
        Cov. ratio &  0.57 & 0.61 & 0.73 & 0.98 & 0.95 & 0.95 \\ \hline
      \end{tabular}
    }
  \end{minipage}

  \vspace{-3mm}
\end{table}

\begin{figure}[t]
    \centering
    \includegraphics[width=1.0\linewidth]{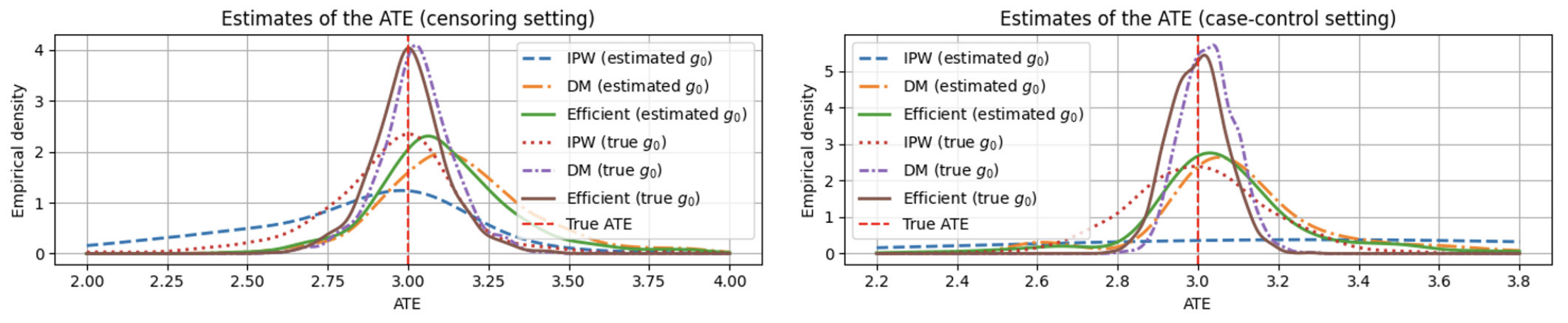}
    \vspace{-7mm}
    \caption{Empirical distributions of ATE estimates.}
    \vspace{-5mm}
    \label{fig:figure_exp1}
\end{figure}

\section{Semiparametric efficient ATE estimation under the case-control setting}
\label{sec:case-control}
Here, we briefly introduce the ATE estimator in the case-control setting. More detailed results are shown in Appendix~\ref{appdx:case-control}. 

We define
\begin{align*}
   \widehat{\tau}^{\mathrm{cc}\mathchar`-\mathrm{eff}}_{n} &\coloneqq  \frac{1}{m}\sum^m_{j=1}\p{1 - \frac{\widehat{e}_{n, i}(1\mid X_i)}{\widehat{e}_{n, i}(0\mid X_i)}}\Bigp{Y(1) - \widehat{\mu}_{\rmT, n, i}(X)}\widehat{r}_{n, i}(X_i)\\
   &+ \frac{1}{l}\sum^l_{k=1}\p{\frac{Y_{\rmU, i} - \widehat{\mu}_{\rmU, n, i}(X_i)}{\widehat{e}_{n, i}(0\mid X_i)} + \widehat{\mu}_{\rmT, n, i}(X_i) - \frac{\widehat{\mu}_{\rmU, n, i}(X_i)}{\widehat{e}_{n, i}(0\mid X_i)} + \frac{\widehat{e}_{n, i}(1\mid X_i)\widehat{\mu}_{\rmT, n, i}(X_i)}{\widehat{e}_{n, i}(0\mid X_i)}} 
\end{align*}
as an ATE estimator in the case-control setting. 
Here, $\widehat{\mu}_{\rmT, n, i}$, $\widehat{\mu}_{\rmU, n, i}$, $\widehat{e}_{n, i}$, and $\widehat{r}_{n, i}$ are estimators of $\mu_{\rmT, 0}$, $\mu_{\rmU, 0}$, $e_0$, and $r_0$, where $m$ and $l$ denote the dependence on each dataset.

For the estimator, we show the following theorem, which is an informal version of Theorem~\ref{thm:asymp_norm_case_control}
\begin{theorem}[Asymptotic normality in the case-control setting (Informal)]
\label{thm:informal_asymp_norm_case_control}
Fix $\alpha \in (0, 1)$. For $n > 0$,
consider the case-control setting with sample sizes $m, l$ such that $m = \alpha n$ and $l = (1-\alpha)n$. If the case-control propensity score $e_0$ and the density ratio are known ($r_0$ $\widehat{e}_{n, i} = e_0$ and $\widehat{r}_{n, i} = r_0$), and  $\widehat{\mu}_{\rmT, n, i} = \widehat{\mu}^{(\ell)}_{\rmT, m}$, $\widehat{\mu}_{\rmU, n, i} = \widehat{\mu}^{(\ell)}_{\rmU, l}$ are consistent estimators constructed via cross-fitting. Then, under regularity conditions (see Theorem~\ref{thm:asymp_norm_case_control}), we have $\sqrt{n}\p{\widehat{\tau}^{\mathrm{cc}\mathchar`-\mathrm{eff}}_{n} - \tau_0} \xrightarrow{\rmd} \mathcal{N}(0, V^{\mathrm{cc}})$ as $n\to \infty$, where $V^{\mathrm{cc}} > 0$ is the efficiency bound defined in Theorem~\ref{thm:efficy_bound_cc}. 
\end{theorem}

\section{Simulation studies}
\label{sec:experiments}
This section investigates the empirical performance of the proposed estimators. We also show the experimental results using semi-synthetic data in Appendix~\ref{sec:semi-synthetic}. 

\subsection{Censoring setting}
\label{sec:sim_censoring}
We generate synthetic data under the censoring setting, where the covariates $X$ are drawn from a multivariate normal distribution as $X \sim \zeta_0(x)$,
where $\zeta_0(x)$ is the density of $\mathcal{N}(0, I_p)$, and $I_p$ denotes the $(p\times p)$ identity matrix. We set $p = 3$. Set $\bbP(D\mid X) = \mathrm{trunc}(\mathrm{sigmoid}(X^\top \beta), 0.1, 0.9)$, 
where $\beta$ is a coefficient sampled from $\mathcal{N}(0, 0.5I_{p})$, and $\mathrm{trunc}(t, a, b)$ truncates $t$ by $a$ and $b$ ($a < b$). Treatment $D$ is sampled from the probability. The observation indicator $O$ is generated from a Bernoulli distribution with probability $c$ if $D_i = 1$ and $O_i = 0$ if $D_i = 0$. Here, $c$ is generated from a uniform distribution with support $[0, 1]$. 
The outcome is generated as $Y = X^\top \beta + 1.1 + \tau_0 \cdot D + \varepsilon$, 
where $\varepsilon \sim \mathcal{N}(0,1)$, where we set $\tau_0 = 3$. 

The nuisance parameters are estimated using linear regression and (linear) logistic regression. We compared our proposed estimator, $\widehat{\tau}^{\mathrm{cens}\mathchar`-\mathrm{eff}}_n$, with the other candidates, the IPW estimator $\widehat{\tau}^{\mathrm{cens}\mathchar`-\mathrm{IPW}}_n$ and the DM estimator $\widehat{\tau}^{\mathrm{cens}\mathchar`-\mathrm{DM}}_n$, defined in Remarks~\ref{rem:IPW} and \ref{rem:DM}, respectively. Note that all of these estimators are proposed by us, and our goal is not to confirm $\widehat{\tau}^{\mathrm{cens}\mathchar`-\mathrm{eff}}_n$ outperforms the others, while our recommendation is $\widehat{\tau}^{\mathrm{cens}\mathchar`-\mathrm{eff}}_n$. We consider both cases where the propensity score is either estimated using the method proposed by \citet{Elkan2008learningclassifiers} or assumed to be known.

We set $n = 3000$. 
We conduct $5000$ trials and report the empirical mean squared errors (MSEs) and biases for the true ATE and the coverage ratio (Cov. ratio) computed from the confidence intervals in Table~\ref{tab:table_exp1}. We also present the empirical distributions of the ATE estimates in Figure~\ref{fig:figure_exp1}.

As the theory suggests, $\widehat{\tau}^{\mathrm{cens}\mathchar`-\mathrm{eff}}_n$ exhibits smaller MSEs compared to other methods. Interestingly, when the propensity score is estimated, the MSEs decrease, a phenomenon reported in existing studies. The coverage ratio is also accurate. The empirical distribution of the ATE estimates demonstrates the asymptotic normality.

\subsection{Case-control setting}
\label{sec:sim_casecontrol}
In the case-control setting, covariates for the treatment and unknown groups are generated from different $p$-dimensional normal distributions: $X_{\rmT} \sim \zeta_{\rmT, 0}(x)$ and $X \sim \zeta_0(x) = e_0(1)\zeta_{\rmT, 0}(x) + e_0(0)\zeta_{\rmC}(x)$, where we set $p=3$, $\zeta_{\rmT, 0}(x)$ and $\zeta_{\rmC}(x)$ are the densities of normal distributions $\mathcal{N}(\mu_p\bm{1}_p, I_p)$ and $\mathcal{N}(\mu_n\bm{1}_p, I_p)$, $\mu_p = 0.5$ and $\mu_n = 0$, $\bm{1}_p = (1\ 1\ \cdots\ 1)^\top$, and $e_0(1)$ is the class prior set as $e_0(1) = 0.3$. By definition, the propensity score $e_0(d\mid x)$ is given as $e_0(1\mid x) = e_0(1) \zeta_{\rmT, 0}(x) / \zeta_0(x)$. The outcome is generated similarly to the censoring setting $Y = X^\top \beta + 1.1 + \tau_0 D + \varepsilon$, where $\tau_0 = 3$. 

We set $m = 1000$ and $l = 2000$ and compute the same evaluation metrics as in the censoring setting. Although logistic regression is used, the propensity score model is misspecified, while the expected conditional outcome follows a linear model.

Overall, $\widehat{\tau}^{\mathrm{cc}\mathchar`-\mathrm{eff}}_n$ demonstrates robust performance in terms of MSE, bias, and coverage ratio. The poor performance of the IPW estimator is attributed to model misspecification. 

We investigate non-linear settings in Appendix~\ref{appdx:exp}. 

\section{Conclusion}
\label{sec:conclusion}
In this study, we investigated PUATE, the problem of ATE estimation in the presence of missing treatment indicators. We formulated the problem using the censoring and case-control settings, inspired by PU learning. For each setting, we derived the efficiency bound and developed an efficient estimator. Our analysis revealed that achieving asymptotic normality and efficiency. Future research directions include extending our approach to the semi-supervised setting, handling additional missing values, and the relaxation of assumptions regarding the missingness mechanism.



\bibliography{Bibtex/others,Bibtex/causalinference,Bibtex/rdd,Bibtex/semiparametric,Bibtex/nonparametric,Bibtex/highdimensional,Bibtex/experimentaldesign,Bibtex/reinforcement,Bibtex/weaksupervised,Bibtex/machinelearning,Bibtex/neuralnet,Bibtex/conditionalaveragetreatmenteffect,Bibtex/policylearning,Bibtex/statistics,Bibtex/bandits,Bibtex/finance,Bibtex/conformalinference,Bibtex/optimaltransport,Bibtex/treatment_choice}

\bibliographystyle{tmlr.bst}

\onecolumn

\appendix
 
\section{Related work}
\label{appdx:related}
The ATE estimation problem has long been studied in statistics, epidemiology, economics, and machine learning \citep{Imbens2015causalinference}. While randomized controlled trials are considered the gold standard, it is extremely important to estimate the ATE in observational studies. In ATE estimation with observational data, one of the basic approaches is to employ the IPW estimator, which allows us to correct for selection bias \citep{Horvitz1952generalization}. 

Although the IPW estimator is a powerful tool, it is known that its asymptotic variance exceeds the efficiency bound even when the true propensity score is used \citep{Hahn1998ontherole}.\footnote{Under certain conditions, using an estimated propensity score can reduce the asymptotic variance, as shown by \citet{Hirano2003efficientestimation}. This phenomenon is known as the paradox of the propensity score \citep{Henmi2004aparadox,Kato2021adaptivedoubly}.}

Another powerful estimator in this context is the doubly robust estimator \citep{Bang2005doublyrobust}, which also plays an important role in the literature on missing data \citep{Yang2024doublyrobust}. The doubly robust estimator not only satisfies the double robustness property but also achieves asymptotic efficiency; that is, its asymptotic variance reaches the efficiency bound \citep{Hahn1998ontherole}. This property is closely related to the efficient influence function in the derivation of the efficiency bound \citep{Vaart1998asymptoticstatistics,Tsiatis2007semiparametrictheory}. The doubly robust estimator is defined as a sample average of the efficient influence function with estimated nuisance parameters.

The doubly robust estimator has been further refined by various studies. For example, \citet{vanderLaan2006targetedmaximum} propose the targeted maximum likelihood framework, which has the potential to improve the finite-sample performance of the efficient ATE estimator.

To construct efficient estimators, convergence rate conditions and complexity restrictions for the nuisance estimators are usually required \citep{Schuler2024introductionmodern}. In particular, to satisfy the complexity restrictions, researchers often assume the Donsker condition or apply sample splitting \citep{Vaart2002semiparametricstatistics,Klaassen1987consistentestimation,Zheng2011crossvalidatedtargeted}. These approaches are further developed in the double machine learning framework by \citet{Chernozhukov2018doubledebiased}, where convergence rate conditions are relaxed through the use of Neyman orthogonality, and complexity restrictions are addressed via sample splitting, known as cross-fitting. For discussions of the relationship between double machine learning and other frameworks, such as targeted maximum likelihood estimation, see \citet{Kennedy2016semiparametrictheory,Kennedy2023semiparametricdoubly}.

It is important to note that Neyman orthogonality and cross-fitting play different roles. Cross-fitting is used to ensure that the nuisance estimators (e.g., propensity score, outcome models) are independent of the observations to which they are applied. However, cross-fitting alone is not sufficient to guarantee asymptotic normality. The issue is that nuisance estimators typically converge at rates slower than $\sqrt{n}$. In doubly robust estimation, the convergence rate conditions for the nuisance estimators are relaxed due to the doubly robust structure, which is also referred to as Neyman orthogonality.

Our work builds upon these arguments. However, in our setting, we cannot apply the techniques from \citet{Chernozhukov2018doubledebiased} to mitigate the convergence rate condition for the propensity score. In other words, our estimator is sensitive to the accuracy of the propensity score estimation. Therefore, we assume that the propensity score is known in order to derive asymptotic normality, although consistency can still be achieved when the propensity score is estimated.

CATE estimation is also an important topic related to this study \citep{Heckman1997matchingas}. Various methods have been proposed for estimating CATE including methods using neural networks \citep{Johansson2016learningrepresentations,Shalit2017estimatingindividual,Shi2019adaptingneural,Hassanpour2020learningdisentangled,Curth2021nonparametricestimation}, gaussian process \citep{Alaa2017bayesianinference}, and tree-based approaches \citep{Wager2018estimationinference}. A critical perspective in recent literature is minimax optimality. \citet{Kennedy2024minimaxrates} proposes a minimax optimal CATE estimator based on the R-learner \citep{Nie2020quasioracle}, by deriving a minimax lower bound \citep{Tsybakov2008introductionnonparametric}. While several directions exist for extending our results to CATE estimation, deriving a minimax optimal CATE estimator would require further theoretical analysis, which is beyond the scope of this study.

\subsection{ATE estimation with missing data}
ATE estimation under missing values has been extensively studied, as the standard ATE estimation setting is itself closely related to the literature on missing data \citep{Rubin1976inferenceand,Bang2005doublyrobust}.

In this context, various assumptions can be made about how data is missing. For example, some studies consider settings with missing covariates \citep{Zhao2024toadjust}. This study, however, focuses on the case in which covariates are fully observed and treatment assignment is missing. In the problem of missing treatments, \citet{Molinari2010missingtreatments} presents several examples from survey analysis. \citet{Ahn2011missingexposure} investigates the effect of physical activity on colorectal cancer using data in which treatment is missing for about $20$\% of the units. \citet{Zhang2013causalinference} estimates infant weight outcomes where the treatment—mother's body mass index (BMI)—is missing for about half of the sample. \citet{Kennedy2020efficientnonparametric} develops a general framework for handling settings where both the observation indicator and the treatment indicator are separately observable. In contrast, in our case, we can observe only the product of the observation and treatment indicators, implying less available information than in \citet{Kennedy2020efficientnonparametric}. \citet{Kuzmanovic2023estimatingconditional} proposes a method for conditional ATE estimation under this weaker setting.

Our problem is also related to ATE estimation from misclassified data \citep{Lewbel2007estimationof}. Early econometric studies focused on continuous regressors \citep{Hausman2001mismeasuredvariables}. With regard to binary variables, \citet{Mahajan2006identificationestimation} analyzes misclassification in regression models, while \citet{Lewbel2007estimationof} develops methods for identifying and estimating ATEs under potentially misclassified treatment indicators. Researchers have also explored partial identification approaches when the exact misclassification process is unknown, providing bounds on parameters rather than point estimates \citep{Manski1993identificationproblems,Manski2010partialidentification}. In applied settings, validation data have been used to refine causal effect estimates under potential misclassification \citep{Black2008stayingclassroom}, demonstrating that even modest errors in treatment indicators can significantly impact policy conclusions. \citet{Yamane2018upliftmodeling} also addresses a related problem. 

Finally, we refer to semi-supervised treatment effect estimation \citep{chakrabortty2024generalframework}, which primarily considers a scenario where two datasets are available: one with complete data and the other with only treatment indicators $D$ and covariates $X$ but no outcome data. Although the setting is not directly related, integrating insights from both areas could enhance the applicability.

\subsection{Introduction of PU learning algorithms}
Another related body of work comes from the literature on PU learning. PU learning is a classification method primarily designed for binary classifiers (though it can be extended to multi-class settings) in the presence of missing data. Its origins trace back to case-control studies with contaminated controls \citep{Steinberg1992estimatinglogistic,Lancaster1996casecontrolstudies}, which are refined in \citet{duPlessis2015convexformulation} under the term case-control PU learning. In parallel, \citet{Elkan2008learningclassifiers} investigates PU learning in the context of the censoring setting. One of the main applications of PU learning is learning from implicit feedback, which commonly arises in marketing and recommender systems. In such settings, user actions—such as product purchases—are observable, but non-actions do not necessarily imply disinterest in the products; therefore, we might suffer bias in a classifier for predicting the users interests if we train it using such data with regarding the action and non-action as positive and negative data. As discussed in the introduction, we consider a similar application. However, our goal is to estimate treatment effects, rather than to train a classifier.

\section{Reformulation of the censoring and the case-control settings}
This section provides a reformulation of the case-control and censoring setting to deepen our understanding. Note that the formulation described in this section is mathematically equivalent to the ones in Sections~\ref{sec:censoring} and \ref{sec:case_control_dgp}. 

\subsection{Reformulation of the censoring setting}
We can introduce the censoring setting with the following story:
\begin{itemize}
    \item For each $i$, a sample $(X_i, {D}_i, Y_i)$ is generated.
    \item A coin is tossed, and $\widetilde{O}_i = 1$ if it lands heads, or $\widetilde{O}_i = 0$ if it lands tails.
    \begin{itemize}
        \item If $\widetilde{O}_i = 1$ and $D_i = 1$ (that is, $\widetilde{O}_iD_i = 1$), then we observe the treatment indicator $D_i = \widetilde{O}_i = 1$.
        \item Otherwise, the treatment indicator is not observed, and we only observe $(X_i, Y_i)$.
    \end{itemize}
    \item Finally, we observe $(X_i, \widetilde{O}_iD_i, Y_i)$.
\end{itemize}
By denoting $\widetilde{O}_iD_i$ by $O_i$, we can obtain the same formulation in Section~\ref{sec:censoring_dgp}.

\begin{remark}[\citet{Kennedy2020efficientnonparametric}]
    For example, \citet{Kennedy2020efficientnonparametric} considers the missing treatment information, which is essentially different from ours. \citet{Kennedy2020efficientnonparametric} considers the following setup:
    \begin{itemize}
        \item For each $i$, a sample $(X_i, D_i, Y_i)$ is generated, where $X_i$ denotes covariates, $D_i$ the treatment indicator, and $Y_i$ the outcome.
        \item A coin is tossed, and $\widetilde{O}_i = 1$ if it lands heads, or $\widetilde{O}_i = 0$ if it lands tails.
        \begin{itemize}
            \item If $\widetilde{O}_i = 1$, we observe the treatment indicator $D_i \in {1, 0}$ along with $(X_i, Y_i)$.
            \item If $\widetilde{O}_i = 0$, the treatment indicator is unobserved, and we only observe $(X_i, Y_i)$.
        \end{itemize}
        \item Finally, we observe $(X_i, \widetilde{O}_i, \widetilde{O}_iD_i, Y_i)$.

    \end{itemize}
    For each $i$, a sample $(X_i, D_i, Y_i)$ is generated, where $X_i$ denotes covariates, $D_i$ the treatment indicator, and $Y_i$ the outcome.

    Thus, while \citet{Kennedy2020efficientnonparametric} observes $(X_i, \widetilde{O}_i, \widetilde{O}_iD_i, Y_i)$, we observe only $(X_i, \widetilde{O}_iD_i, Y_i)$. In our case, $\widetilde{O}_i$ itself is also missing (note again that $(X_i, \widetilde{O}_iD_i, Y_i)$ is equivalent to $(X_i, O_i, Y_i)$, where $O_i = \widetilde{O}_iD_i$). That is, in our case, we can observe only a subset of treatment labels with $D_i = 1$, while the remaining labels are missing and consist of a mixture of $D_i = 1$ and $D_i = 0$. In contrast, \citet{Kennedy2020efficientnonparametric} can observe both $D_i = 1$ and $D_i = 0$ when the label is observed. Our setting is designed to be more suitable for applications in marketing and recommendation systems, where implicit feedback is common\footnote{Regarding the missingness mechanism of $\widetilde{O}_i$, \citet{Kennedy2020efficientnonparametric} considers a more general setting than ours by allowing $\widetilde{O}_i$ to depend on the outcome $Y_i$. In contrast, while we do not allow such dependence, we use less information about the treatment indicator than \citet{Kennedy2020efficientnonparametric}, as explained above. Thus, we cannot say which setting is more general.}.

    Note that in our study, we do not explicitly use $\widetilde{O}_i$ but instead denote $\widetilde{O}_iD_i$ by another random variable $O_i$; that is, $O_i = \widetilde{O}_iD_i$. 

    In machine learning terminology, we believe that the setup in \citet{Kennedy2020efficientnonparametric} is close to semi-supervised learning, where both a fully labeled dataset $(X_i, L_i)$ $(i = 1,2,..,n)$ and an unlabeled dataset $X_j$ $(j = 1,2,..,m)$ are available (Note, however, that in \citet{Kennedy2020efficientnonparametric}, whether a data point is labeled is itself a random event, whereas typical semi-supervised learning assumes a deterministic labeling process). PU learning is generally considered a distinct setting from semi-supervised learning though they are related. For a discussion of the relationship between these settings, see \citet{Sakai2017semisupervisedclassification}. 
\end{remark}

\subsection{Reformulation of the case-control setting}
We can introduce the case-control setting with the following story:
\begin{itemize}
\item There are two groups: the treatment group and the unknown group:
\begin{itemize}
    \item Treatment group:
    \begin{itemize}
        \item For each $j$, a sample $(X_{\rmT, j}, Y_j(1))$ is generated and observed by us.
    \end{itemize}
    \item Unknown group:
    \begin{itemize}
        \item For each $k$, a sample $(X_i, D_i, Y_i)$ is generated.
        \item We observe $(X_i, Y_i)$.
    \end{itemize}
\end{itemize}
\end{itemize}

\section{Details of example of ATE estimators in the case-control setting}
This section provides the details of the examples shown in Section~\ref{sec:exampleate}. 

\subsection{Example ATE estimator in the censoring setting}
In the censoring setting, as we explained in Section~\ref{sec:exampleate}, we can obtain an ATE estimator by replacing the following two quantities with sample approximation and nuisance estimators: 
\begin{align*}
    \bbE[Y(1)] &= \bbE\sqb{\frac{\mathbbm{1}[O = 1]Y}{\pi_0(1 \mid X)}},\\
    \bbE[Y(0)] &= \bbE\sqb{\frac{\mathbbm{1}[O = 0]Y}{g_0(0\mid X)\pi_0(0 \mid X)}} - \bbE\sqb{\frac{g_0(1\mid X)\mathbbm{1}[O = 1]Y}{g_0(0\mid X)\pi_0(1 \mid X)}}.
\end{align*}

Such an estimator can be defined as follows:
\[\widehat{\tau}^{\mathrm{cens}\mathchar`-\mathrm{IPW}}_n \coloneqq \frac{1}{n}\sum^n_{i=1}\p{\frac{\mathbbm{1}[O_i = 1]Y_i}{\widehat{\pi}_{n}(1\mid X_i)} - \frac{\mathbbm{1}[O_i = 0]Y_i}{\widehat{g}_{n}(0\mid X_i)\widehat{\pi}_{n, i}(0\mid X_i)} +\frac{\widehat{g}_{n, i}(1\mid X_i)\mathbbm{1}[O_i = 1]Y_i}{\widehat{g}_{n, i}(0\mid X_i)\widehat{\pi}_{n, i}(1\mid X_i)}},\]
where $\widehat{\pi}_{n, i}$ and $\widehat{g}_{n}$ are estimators of $\pi_0$ and $g_0$. We refer to this estimator as the inverse probability weighting (IPW) estimator, which is also shown in Remark~\ref{rem:IPW}

For this estimator, we can show the following theorem.
\begin{theorem}
    Suppose that Assumptions~\ref{asm:unconfoundedness_censoring}--\ref{asm:overlap_censoring}, and \ref{asm:consistency_censoring}. If $\| \widehat{\pi}_{n, i}(d\mid X) - \pi_0(d\mid X) \|_2 = o_p(1)$ and $\| \widehat{g}_{n}(d\mid X) - g_0(d\mid X) \|_2 = o_p(1)$ hold as $n\to\infty$ for $d\in\{1, 0\}$ then $\widehat{\tau}^{\mathrm{cens}\mathchar`-\mathrm{IPW}}_n \xrightarrow{\rmp} \tau_0$ holds as $n\to \infty$.
\end{theorem}
\begin{proof}
We have
    \begin{align*}
        &\widehat{\tau}^{\mathrm{cens}\mathchar`-\mathrm{IPW}}_n = \frac{1}{n}\sum^n_{i=1}\p{\frac{\mathbbm{1}[O_i = 1]Y_i}{\widehat{\pi}_{n}(1\mid X_i)} - \frac{\mathbbm{1}[O_i = 0]Y_i}{\widehat{g}_{n}(0\mid X_i)\widehat{\pi}_{n, i}(0\mid X_i)} +\frac{\widehat{g}_{n, i}(1\mid X_i)\mathbbm{1}[O_i = 1]Y_i}{\widehat{g}_{n, i}(0\mid X_i)\widehat{\pi}_{n, i}(1\mid X_i)}}\\
        &= \frac{1}{n}\sum^n_{i=1}\p{\frac{\mathbbm{1}[O_i = 1]Y_i}{\pi_0(1\mid X_i)} - \frac{\mathbbm{1}[O_i = 0]Y_i}{g_0(0\mid X_i)\pi_0(0\mid X_i)} +\frac{g_0(1\mid X_i)\mathbbm{1}[O_i = 1]Y_i}{g_0(0\mid X_i)\pi_0(1\mid X_i)}}\\
        &\ \ \ - \p{\frac{1}{n}\sum^n_{i=1}\p{\frac{\mathbbm{1}[O_i = 1]Y_i}{\pi_0(1\mid X_i)} - \frac{\mathbbm{1}[O_i = 0]Y_i}{g_0(0\mid X_i)\pi_0(0\mid X_i)} +\frac{g_0(1\mid X_i)\mathbbm{1}[O_i = 1]Y_i}{g_0(0\mid X_i)\pi_0(1\mid X_i)}}}\\
        &\ \ \ + \p{\frac{1}{n}\sum^n_{i=1}\p{\frac{\mathbbm{1}[O_i = 1]Y_i}{\widehat{\pi}_{n}(1\mid X_i)} - \frac{\mathbbm{1}[O_i = 0]Y_i}{\widehat{g}_{n}(0\mid X_i)\widehat{\pi}_{n, i}(0\mid X_i)} +\frac{\widehat{g}_{n, i}(1\mid X_i)\mathbbm{1}[O_i = 1]Y_i}{\widehat{g}_{n, i}(0\mid X_i)\widehat{\pi}_{n, i}(1\mid X_i)}}}\\
        &= \frac{1}{n}\sum^n_{i=1}\p{\frac{\mathbbm{1}[O_i = 1]Y_i}{\pi_0(1\mid X_i)} - \frac{\mathbbm{1}[O_i = 0]Y_i}{g_0(0\mid X_i)\pi_0(0\mid X_i)} +\frac{g_0(1\mid X_i)\mathbbm{1}[O_i = 1]Y_i}{g_0(0\mid X_i)\pi_0(1\mid X_i)}} + o_p(1).
    \end{align*}
    Here, from the law of large numbers, we have
     \begin{align*}
        &\frac{1}{n}\sum^n_{i=1}\frac{\mathbbm{1}[O_i = 1]Y_i}{\pi_0(1\mid X_i)} \xrightarrow{\rmp} \bbE\sqb{\frac{\mathbbm{1}[O = 1]Y}{\pi_0(1 \mid X)}} = \bbE[Y(1)]\\
        &\frac{1}{n}\sum^n_{i=1}\p{\frac{\mathbbm{1}[O_i = 0]Y_i}{g_0(0\mid X_i)\pi_0(0\mid X_i)} - \frac{g_0(1\mid X_i)\mathbbm{1}[O_i = 1]Y_i}{g_0(0\mid X_i)\pi_0(1\mid X_i)}}\\
        &\ \ \ \xrightarrow{\rmp} \bbE\sqb{\frac{\mathbbm{1}[O = 0]Y}{g_0(0\mid X)\pi_0(0 \mid X)}} - \bbE\sqb{\frac{g_0(1\mid X)\mathbbm{1}[O = 1]Y}{g_0(0\mid X)\pi_0(1 \mid X)}} = \bbE[Y(0)].
    \end{align*}
    Thus, the proof is complete. 
\end{proof}

\subsection{Example ATE estimator in the case-control setting}
\label{appdx:example_case_control}
In the case-control setting, PU learning methods have been investigated by \citet{Imbens1996efficientestimation} and \citet{duPlessis2015convexformulation}. In that works, we typically make the following assumption, which corresponds to the SCAR assumption in the censoring setting. 

\begin{assumption}
\label{asm:pu_learning_casecontrol}
It holds that $\zeta_{\rmT, 0}(x) = \zeta_0(x \mid D = 1)$, where $\zeta_0(x \mid D = d) = \frac{e_0(d \mid x)\zeta_0(x)}{e_0(d)}$.
\end{assumption}

This assumption is also attempted to be relaxed by existing work, such as \citet{Kato2019learningfrom} and \citet{Hsieh2019classificationfrom} introduce their approaches. 
As well as the censoring setting, various relaxations exist depending on the application, and there are trade-offs between the strengths of assumptions and identification \citep{Manski1993identificationproblems}.

We further make the following assumptions. 

\begin{assumption}[Unconfoundedness in the case-control setting]
\label{asm:unconfoundedness_casecontrol}
The potential outcomes $(Y(1), Y(0))$ are independent of treatment assignment given covariates: 
\[(Y(1), Y(0)) \indep D \mid X.\] 
\end{assumption}

\begin{assumption}[Common support in the case-control setting]
\label{asm:overlap_casecontrol}
There exists a constant $c$ independent of $n$ such that for all $x \in \calX$, $e_0(d\mid x), \zeta_{\rmT, 0}(x), \zeta_0(x) > c$ hold. 
\end{assumption}

Under these assumptions, the ATE $\tau_0$ is estimable by replacing the following two quantities with sample approximation: 
\[\bbE[Y(1)] = \bbE\sqb{Y r_0(X)}\] and 
\[\bbE[Y(0)] = \bbE\sqb{\frac{Y}{e_0(0\mid X)}} - \bbE\sqb{\frac{e_0(1\mid X)Y}{e_0(0 \mid X)}},\] 
where recall that $r_0(X) = \frac{\zeta_0(X)}{\zeta_{\rmT, 0}(X)}$, $e_0$ can be estimated using PU learning methods, and expectations can be approximated by sample averages.

\section{Semiparametric efficient ATE estimation under the case-control setting}
\label{appdx:case-control}
In this section, we consider efficient ATE estimation under the case-control setting. Similar to the censoring setting, we first derive the efficiency bound and then propose an efficient estimator, providing theoretical guarantees for its consistency and asymptotic normality. Throughout the arguments, we assume $m = \alpha n$ and $l = (1 - \alpha) n$, where $\alpha \in (0, 1)$. 

\subsection{Efficient influence function and efficiency bound}
Using efficiency arguments under the stratified sampling scheme \citep{Uehara2020offpolicy}, we derive the following efficient influence function (see Appendix~\ref{appdx:eb_cc} for the proof). 
\begin{lemma}
\label{lem:efficient_case_control}
    The efficient influence functions are given as $\Psi^{\mathrm{cc}~(\mathrm{T})}(X, Y(1); \mu_{\rmT, 0}, e_0, r_0)$ and $\Psi^{\mathrm{cc}~(\mathrm{U})}(X, Y_{\rmU}; \mu_{\rmT, 0}, \mu_{\rmU, 0}, e_0, \tau_0)$, where 
    \begin{align*}
    &\Psi^{\mathrm{cc}~(\mathrm{T})}(X, Y(1); \mu_{\rmT, 0}, e_0, r_0) \coloneqq S^{\mathrm{cc}~(\mathrm{T})}(X, Y(1); \mu_{\rmT, 0}, e_0, r_0),\\
    &\Psi^{\mathrm{cc}~(\mathrm{U})}(X, Y_{\rmU}; \mu_{\rmT, 0}, \mu_{\rmU, 0}, e_0, \tau_0) \coloneqq S^{\mathrm{cc}~(\mathrm{U})}(X, Y_{\rmU}; \mu_{\rmT, 0}, \mu_{\rmU, 0}, e) - \tau_0,\\
    &S^{\mathrm{cc}~(\mathrm{T})}(X, Y(1); \mu_{\rmT, 0}, e_0, r_0) \coloneqq \p{1 - \frac{e_0(1\mid X)}{e_0(0\mid X)}}\Bigp{Y(1) - \mu_{\rmT, 0}(X)}r_0(X),\\
    &S^{\mathrm{cc}~(\mathrm{U})}(X, Y_{\rmU}; \mu_{\rmT, 0}, \mu_{\rmU, 0}, e_0)\coloneqq  \frac{Y_{\rmU} - \mu_{\rmU, 0}(X)}{e_0(0\mid X)} +  \mu_{\rmT, 0}(X) - \frac{\mu_{\rmU, 0}(X)}{e_0(0\mid X)} + \frac{e_0(1\mid X)\mu_{\rmT, 0}(X)}{e_0(0\mid X)},
\end{align*}
and recall that $r_0(X) \coloneqq \frac{\zeta_0(X)}{\zeta_{\rmT, 0}(X)}$. 
\end{lemma}

Then, we obtain the result on the efficiency bound. 
\begin{theorem}[Efficiency bound in the case-control setting]
\label{thm:efficy_bound_cc}
The asymptotic variance of any regular estimator is lower bounded by
\begin{align*}
&V^{\mathrm{cc}} \coloneqq \frac{1}{\alpha}\bbE\sqb{\Psi^{\mathrm{cc}~(\mathrm{T})}(X, O, Y; \mu_{\rmT, 0}, e_0, r_0)^2} + \frac{1}{1- \alpha}\bbE\sqb{\Psi^{\mathrm{cc}~(\mathrm{U})}(X, O, Y; \mu_{\rmU, 0}, e_0)^2}\\
&= \frac{1}{\alpha} \bbE\sqb{\p{1 - \frac{e_0(1\mid X)}{e_0(0\mid X)}}^2\mathrm{Var}(Y(1)\mid X)r_0(X)} + \frac{1}{1- \alpha}\bbE\sqb{\frac{\mathrm{Var}(Y_{\rmU}\mid X)}{e_0(0\mid X)^2} + \Bigp{\tau_0(X) - \tau_0}^2}    
\end{align*}
where $\alpha = m/n$. 
\end{theorem}

\subsection{Semiparametric efficient estimator}
Based on the efficient influence function, we define 
\[\widehat{\tau}^{\mathrm{cc}\mathchar`-\mathrm{eff}}_{n} \coloneqq \frac{1}{m}\sum^m_{j=1}S^{\mathrm{cc}~(\mathrm{T})}(X_j, Y_j; \widehat{\mu}_{\rmT, n, i}, \widehat{e}_{n, i}, \widehat{r}_{n, i}) + \frac{1}{l}\sum^l_{k=1}S^{\mathrm{cc}~(\mathrm{U})}(X_k, Y_k; \widehat{\mu}_{\rmU, n, i}, \widehat{e}_{n, i}.\] 
Here, $\widehat{\mu}_{\rmT, n, i}$, $\widehat{\mu}_{\rmU, n, i}$, $\widehat{e}_{n, i}$, and $\widehat{r}_{n, i}$ are estimators of $\mu_{\rmT, 0}$, $\mu_{\rmU, 0}$, $e_0$, and $r_0$, where $m$ and $l$ denote the dependence on each dataset.  
Unlike the censoring setting, we do not use the observation indicator $O$, as it is deterministic whether a unit belongs to the treatment group or the control group. This distinction leads to differences in the theoretical analysis compared to the censoring setting.

\subsection{Consistency}
We make the following assumption.
\begin{assumption}
\label{asm:consistency_casecontrol}
    As $n \to \infty$, it holds that $\big\|\widehat{e}_{n, i} - e_0\big\|_2 = o_p(1)$ and $\big\|\widehat{r}_{n, i} - r_0\big\|_2 = o_p(1)$. 
\end{assumption}

Then, the following consistency result holds.
\begin{theorem}[Consistency in the case-control setting]
    If Assumption~\ref{asm:consistency_casecontrol} holds, then $\widehat{\tau}^{\mathrm{cc}\mathchar`-\mathrm{eff}}_{n} \xrightarrow{\rmp} \tau_0$ as $n \to \infty$. 
\end{theorem}

Interestingly, to achieve consistency, it is sufficient to obtain consistent $\widehat{e}_{n, i}$. Compared to Assumption~\ref{asm:consistency_casecontrol} in the censoring setting, consistency of the expected outcome estimators $\widehat{\mu}_{\rmT, m}$ and $\widehat{\nu}_{\ell}$ is not required. 
This is because, in this setting, the observation probability can be treated as known ($1$ and $0$ for each dataset). 

\subsection{Asymptotic normality}
\label{sec:asymp_norm_case_control}
Next, we establish the asymptotic normality of the estimator. Similar to the censoring setting, we assume that the propensity score $e_0$ is known and obtain estimators of $\mu_{\rmT, 0}$ and $\mu_{\rmU, 0}$ via cross-fitting. 
\begin{assumption}
\label{asm:prop_known_cc}
    The propensity score $e_0$ and the density ratio $r_0$ are known and used in constructing $\widehat{\tau}^{\mathrm{cc}\mathchar`-\mathrm{eff}}_{n}$, i.e., $\widehat{e}_{n, i} = e_0$ and $\widehat{r}_{n, i} = r_0$. 
\end{assumption}
\begin{assumption}
\label{asm:conv_rate_cc}
For each $\ell \in \calL$, the following conditions hold as $n \to \infty$: $\big\| \mu_{\rmT, 0}(X) - \widehat{\mu}^{(\ell)}_{\rmT, m}(X) \big\|_2 = o_p(1)$ and $\big\| \mu_{\rmU, 0}(X) - \widehat{\mu}^{(\ell)}_{\rmU, l}(X) \big\|_2 = o_p(1)$. 
\end{assumption}

We establish the asymptotic normality in the following theorem with the proof in Appendix~\ref{appdx:normal_cc}. In this result, we consider the scenario where the sample sizes $m$ and $l$ approach infinity while maintaining a fixed ratio $m: l = \alpha: (1-\alpha)$. 
\begin{theorem}[Asymptotic normality in the case-control setting]
\label{thm:asymp_norm_case_control}
Fix $\alpha \in (0, 1)$. For $n > 0$,
consider the case-control setting with sample sizes $m, l$ such that $m = \alpha n$ and $l = (1-\alpha)n$. Suppose that Assumptions~\ref{asm:unconfoundedness_casecontrol}--\ref{asm:overlap_casecontrol}, \ref{asm:prop_known_cc}, and \ref{asm:conv_rate_cc} hold; that is, $\widehat{e}_{n, i} = e_0$ and $\widehat{r}_{n, i} = r_0$, and  $\widehat{\mu}_{\rmT, n, i} = \widehat{\mu}^{(\ell)}_{\rmT, m}$, $\widehat{\mu}_{\rmU, n, i} = \widehat{\mu}^{(\ell)}_{\rmU, l}$ are consistent estimators constructed via cross-fitting. Then, we have 
\begin{align*}
    \sqrt{n}\p{\widehat{\tau}^{\mathrm{cc}\mathchar`-\mathrm{eff}}_{n} - \tau_0} \xrightarrow{\rmd} \mathcal{N}(0, V^{\mathrm{cc}})\ \ \text{as}\ \ n\to \infty
\end{align*}
\end{theorem}
Thus, the proposed estimator is efficient with respect to the efficiency bound derived in Theorem~\ref{thm:efficy_bound_cc}. 

\subsection{Comparison with the censoring setting}
Unlike the censoring setting, we do not require a specific convergence rate for the nuisance estimators if the propensity score is known. This is because, in the case-control setting, the group membership—whether a unit belongs to the treatment group or the unknown group—is deterministically known. This scenario can be interpreted as a case in which the observation probability is known, meaning that only the consistency of the expected outcome estimators is required \citep{Kato2020efficientadaptive,Kato2021adaptivedoubly}. 

In other words, we can intuitively consider that in the case-control case, the observation probability is given as one for the treatment group, while it is given as zero for the unknown group. Since we know the true probabilities, we can ignore the estimation error unlike the censoring setting. Note that this interpretation may be mathematically confusing since it gives us impression that the case-control setting is a special case of the censoring setting where the observation probability is given one or zero. This understanding is not correct because in the case-control setting, the treated and unknown groups are different datasets; that is, the sampling scheme is completely different. This sampling scheme has extensively studied as stratified sampling scheme from 1990s to 2000s by existing studies such as \citet{Imbens1996efficientestimation} and \citet{Wooldridge2001asymptoticproperties}.

\section{PU learning algorithms}
\label{appdx:pulearning}
We review representative PU learning methods. For all methods, the goal is not to obtain a conditional class probability (propensity score) but rather to obtain a better classifier. However, under specific loss functions, including logistic loss, the obtained classifiers can be interpreted as estimators of the probability \citep{Elkan2008learningclassifiers,Kato2019learningfrom,Kato2021nonnegativebregman}.  

\subsection{Censoring PU learning}
In \citet{Elkan2008learningclassifiers}, it is assumed that only a fraction of the truly positive instances are labeled as positive. 

Let $O$ denote the event ``labeled as positive,'' and let $D = 1$ indicate true positivity. In our study, $O$ is called an observation indicator, and $D$ is called a treatment indicator. 

First, we make the following assumption, which plays a central role in the method of \citet{Elkan2008learningclassifiers}:
\begin{align}
\label{asm:censorPU_c}
    \bbP(O = 1 \mid D=1, x) = c \Big(=\bbP(O = 1 \mid D=1)\Big)\ \ \ \forall x \in \calX,
\end{align}
where $c \in (0,1]$ is a constant (Assumption~\ref{asm:pu_learning_censoring}). Intuitively, $c$ represents the \emph{labeling probability} or \emph{censoring rate}, which denotes the fraction of positive instances that are observed (uncensored) in the labeled dataset. 
If we relax this assumption, we may not pointy identify the ATE without different assumptions. There are various approaches proposed to address the relaxation \citep{Bekker2018learningfrom}. 

The learning procedure proposed by \citet{Elkan2008learningclassifiers} consists of three main steps (for details, see \citet{Elkan2008learningclassifiers}): 
\begin{description}
    \item[Estimation of $\pi_0$:] First, the observation probability $\pi_0$ is estimated using standard regression methods, such as logistic regression. 
    \item[Estimation of $c$:] 
    Next, $c$ is estimated using an estimator $\widehat{\pi}_n$ of $\pi_0$. Under Assumption~\ref{asm:pu_learning_censoring}, $c$ can be estimated by taking the sample average of $\widehat{\pi}_n$ over positively labeled samples.
    \item[Correction of the observation probability:] From Assumption~\ref{asm:pu_learning_censoring}, we have $\pi_0(1\mid X) = c \bbP(D=1 \mid X)$. Using this relationship and the estimators of $\pi_0$ and $c$, $\bbP(D=1 \mid X)$ is estimated as $\widehat{\pi}_n(1\mid X) / \widehat{c}$.
\end{description}

\begin{remark}{Violation of the assumptions}
The impact of violating assumptions depends on how the data deviates from the assumed conditions. For example, if treatment labels are not missing at random, the original estimator of \citet{Elkan2008learningclassifiers} may no longer be valid. In such cases, alternative methods—such as those proposed by \citet{Bekker2018learningfrom} and \citet{Teisseyre2025learningfrom}—may be applicable, although they rely on different assumptions. In the case-control setting, \citet{duPlessis2014analysisof} provides a sensitivity analysis of the trained classifier when the class prior is misspecified.
\end{remark}

\begin{remark}[Time complexity]
The computational cost of estimating the conditional class probability via PU learning is comparable to that of standard logistic regression. For example, in the censoring setting, \citet{Elkan2008learningclassifiers} proposes a method based on logistic regression. In the case-control setting, \citet{duPlessis2015convexformulation} presents an unbiased PU learning approach which, under the log-loss, has the same complexity as logistic regression. Specifically, for linear-in-parameter models, logistic regression has a time complexity of order $O((dn^2 + n^3)\log(1/\epsilon))$, where $d$ is the feature dimension and $\epsilon$ is the target optimization accuracy. If Newton's method is used with $T$ iterations, the total time complexity becomes $O((dn^2 + n^3)T)$, which dominates the final averaging step.
\end{remark}
\subsection{Case-control PU learning}
A different perspective is provided by \citet{duPlessis2015convexformulation} and subsequent studies, often referred to as case-control PU learning. In this approach, the labeled positive data follow a distribution $\zeta_{\rmT, 0}(x)$, whereas the unlabeled data are drawn from $\zeta_0(x)$, a mixture of positive and negative instances.

Let $h$ be a classifier. In conventional supervised learning, the classification risk is defined as:
\[
R(h) = e_0(1) R_{+}(h, +1) + (1 - e_0(1)) R_{-}(h, -1),
\]
where $e_0(1)$ is the prior probability of being positive, and $R_{+}(h, +1)$ and $R_{-}(h, -1)$ denote the risks over the positive and negative distributions, respectively. Specifically, $R_{+}(h, +1)$ represents the expected loss when predicting class $-1$ while the true label is $+1$ in the positive distribution, and $R_{-}(h, -1)$ represents the expected loss when predicting class $+1$ while the true label is $-1$ in the negative distribution.

Since negative examples are unavailable, \citet{duPlessis2015convexformulation} re-express $R_{-}(h, -1)$ as:
\[
(1 - e_0(1))R_{-}(h, -1) = R_{\rmU}(h, -1) - e_0(1)R_{+}(h, -1),
\]
where $R_{\rmU}(h, -1)$ and $R_{+}(h, -1)$ denote the risks over the unlabeled and positive distributions, respectively. The term $R_{\rmU}(h, -1)$ corresponds to the expected loss when predicting class $+1$ while the true label is $-1$ in the unlabeled distribution, and $R_{+}(h, -1)$ is the expected loss under the same prediction and true label in the positive distribution. Note that $R_{+}(h, +1)$ and $R_{+}(h, -1)$ are distinct: they consider different true labels while expectations are taken over the same positive distribution.

Substituting the above expression into the original risk gives the following classification risk:
\[
R(h) = e_0(1) R_{+}(h, +1) + R_{\rmU}(h, -1) - e_0(1) R_{+}(h, -1).
\]
A sample-based approximation of this formulation is referred to as an \emph{unbiased risk estimator}.

For example, using the logistic loss, the unbiased risk estimator becomes:
\begin{align*}
    \widehat{R}(h) &= e_0(1) \frac{1}{m} \sum_{j=1}^{m} \log\bigp{1 + \exp(-h(X_j))} \\
    &\quad + \frac{1}{l} \sum_{k=1}^{l} \log\bigp{1 + \exp(h(X_k))} - e_0(1) \frac{1}{m} \sum_{j=1}^{m} \log\bigp{1 + \exp(h(X_j))}.
\end{align*}
Then, we can train a classifier as $\widehat{h} \coloneqq \argmin_{h \in \calH} \widehat{R}(h)$, where $\calH$ is a given hypothesis set.

Note that we assume the class prior $e_0(1)$ is known. Several methods have been proposed to estimate it under additional assumptions \citep{duPlessis2014classprior}.

\subsection{Density-ratio estimation}
Since the density ratio can be estimated in the case-control setting, we introduce related methods. 
Density-ratio estimation has emerged as a powerful technique in machine learning and statistics, providing a principled approach for estimating the ratio of two probability density functions \citep{Sugiyama2012densityratio}. Let $X_i, Z_i \in \calX$ be random variables. Specifically, if $\{X_i \}_{i=1}^{n}$ are drawn from $p_0(x)$ and $\{ Z_j \}_{j=1}^{m}$ are drawn from $q_0(z)$, the goal is to estimate
\[
r_0(x) = \frac{p_0(x)}{q_0(x)}
\]
directly, \emph{without} first estimating $p_0(x)$ and $q_0(z)$ separately. 

Estimating $p_0(x)$ and $q_0(z)$ individually can be challenging and may introduce unnecessary modeling complexities if only the ratio $r_0(x)$ is required. By \emph{directly} estimating the density ratio, more stable and accurate estimates can often be obtained, avoiding potential compounding errors from separately learned density models. 

Various algorithms have been proposed for direct density-ratio estimation, including the Kullback–Leibler Importance Estimation Procedure \citep[KLIEP,][]{Sugiyama2008directimportance} and Least-Squares Importance Fitting \citep[LSIF,][]{Kanamori2009leastsquaresapproach}. 
These methods typically optimize a criterion that ensures the estimated ratio closely approximates the true ratio in a specific divergence sense, such as the Kullback–Leibler divergence or squared error, which can be generalized as a Bregman divergence minimization problem \citep{Sugiyama2012densityratio}. 

From the Bregman divergence minimization perspective, PU learning methods can also be seen as variants of density-ratio estimation, as demonstrated in \citet{Kato2019learningfrom} and \citet{Kato2021nonnegativebregman}.

\section{Remark on the nuisance parameter estimation in the censoring setting}
In the censoring setting, by applying the method of \citet{Elkan2008learningclassifiers}, we can obtain an estimator of $\bbP(D = 1 \mid X)$ from an estimator of $\pi_0(x) = \bbP(O = 1 \mid X)$. However, our objective is to estimate $g_0(1\mid X) = \bbP(D = 1 \mid X, O = 0)$ rather than $\bbP(D = 1 \mid X)$. 

Let $\widehat{\kappa}_n(1\mid X)$ be an estimator of $\bbP(D = 1\mid X)$. We can then obtain an estimator of $g_0(1\mid X)$ as follows:
\[
\widehat{g}_n(1\mid X) = \frac{(1 - \widehat{c})\widehat{\kappa}_n(1\mid X)}{\widehat{\pi}_n(0 \mid X)} = \frac{(1 - \widehat{c})\widehat{\pi}_n(0 \mid X)}{\widehat{c}\widehat{\pi}_n(0 \mid X)},
\]
where $\widehat{c}$ is an estimate of $c = \bbP(O = 0 \mid D = 1)$. Notably, under Assumption~\ref{asm:pu_learning_censoring}, $c$ can be estimated by taking the mean of $\widehat{\pi}_n(0 \mid X)$ over the positively labeled sample ($O_i = 1$).

\section{Pseudo-code for ATE estimation in the case-control setting}
We explain how we construct the estimators of the nuisance parameters in the case-control setting.

We can estimate $\mu_{\rmT, 0}$ and $\mu_{\rmU, 0}$ using standard regression methods, including logistic regression and nonparametric regression. Specifically, for estimating $\mu_{\rmT, 0}$, we typically use the dataset $\cb{\bigp{X_{\rmT, j}, Y_j(1)}}^m_{j=1}$, while for estimating $\mu_{\rmU, 0}$, we use $\cb{\bigp{X_k, Y_{\rmU, k}}}^l_{k=1}$. 

To estimate $e_0$, we can apply case-control PU learning methods, such as convex PU learning proposed by \citet{duPlessis2015convexformulation}. For estimating $r_0$, density-ratio estimation methods can be employed \citep{Sugiyama2012densityratio}. 

Notably, if $\zeta_{\rmT, 0}(x) = \zeta_{0}(x\mid d=1)$, then $r_0$ can be estimated from an estimator of $e_0$ using the relationship $r_0 = 1/(e_0(1\mid x) e_0(1))$.

In cross-fitting, we split $\calD_{\rmT}$ and $\calD_{\rmU}$, respectively, as performed in \citet{Uehara2020offpolicy}. The pseudo-code is shown in Algorithm~\ref{alg:psudo_cc}. 

\begin{algorithm}[t]
\label{alg:psudo_cc}
\caption{Cross-fitting in the case-control setting}
\begin{algorithmic}
\STATE \textbf{Input:} Observations $\calD_{\rmT} = \cb{\bigp{X_{\rmT, j}, Y_j(1)}}^m_{j=1}$ and $\calD_{\rmU} = \cb{\bigp{X_k, Y_{\rmU, k}}}^l_{k=1}$, number of folds $L$, and estimation methods for $\mu_{\rmT, 0}, \mu_{\rmU, 0}, e_0, r_0$.
\STATE Randomly partition $\calD_{\rmT}$ into $L$ roughly equal-sized folds, $(\mathcal{J}^{(\ell)})_{\ell\in\calL}$. Note that $\bigcup_{\ell \in \calL}\calJ^{(\ell)} = \calD$. 
\STATE Randomly partition $\calD_{\rmU}$ into $L$ roughly equal-sized folds, $(\mathcal{K}^{(\ell)})_{\ell\in\calL}$. Note that $\bigcup_{\ell \in \calL}\calK^{(\ell)} = \calD$. 
\FOR{$\ell \in \calL$}
    \STATE Set the training data as $\mathcal{J}^{(-\ell)} = \{1,2,\dots,n\} \setminus \mathcal{J}^{(\ell)}$.
    \STATE Set the training data as $\mathcal{K}^{(-\ell)} = \{1,2,\dots,n\} \setminus \mathcal{K}^{(\ell)}$.
    \STATE Construct estimators of nuisance parameters on $\mathcal{J}^{(-\ell)}$ and $\mathcal{K}^{(-\ell)}$.
    \STATE Construct an ATE estimate $\widehat{\tau}^{\mathrm{cc}\mathchar`-\mathrm{eff}(\ell)}_{n}$ using $\mathcal{J}^{(\ell)}$, $\mathcal{K}^{(\ell)}$and the nuisance estimates $\widehat{\mu}^{(\ell)}_{\rmT, n}, \widehat{\mu}^{(\ell)}_{\rmU, n}, \widehat{e}^{(\ell)}_n, \widehat{r}^{(\ell)}_n$.
\ENDFOR
\STATE \textbf{Output:}  Combine $(\widehat{\tau}^{\mathrm{cc}\mathchar`-\mathrm{eff}(\ell)}_{n})_{\ell\in\calL}$ to form $\widehat{\tau}^{\mathrm{cc}\mathchar`-\mathrm{eff}}_n$.
\end{algorithmic}
\end{algorithm}

\section{Proof of Lemma~\ref{lem:efficiency_bound}}
\label{appdx:eb_censoring}
We prove Lemma~\ref{lem:efficiency_bound}. Our proof procedure is inspired by the one in \citet{Hahn1998ontherole}.

\begin{proof}
Recall that the density function for $(X, O, Y)$ is given as
\begin{align*}
    &p_0(x, o, y) = \zeta_0(x)\Bigp{ \pi_0(1\mid x) p_{Y(1), 0}(y\mid x)}^{\mathbbm{1}[o = 1]} \Bigp{\pi_0(0\mid x)p_{\widetilde{Y}, 0}(y\mid x)}^{\mathbbm{1}[o = 0]},
\end{align*}
where $p_{Y(1), 0}(y\mid x)$, $p_{Y(0), 0}(y\mid x)$, and $p_{\widetilde{Y}, 0}(y\mid x)$ are the conditional densities of $Y(1)$, $Y(0)$, and $\widetilde{Y}$ in the censoring seting. 

For this density function, we consider the parametric submodels:
\[\calP^{\mathrm{sub}} \coloneqq \{P_\theta \in \calP \colon \theta \in \bbR\},\]
where $P_\theta$ has the following density:
\begin{align*}
    p_0(x, o, y; \theta) = \zeta_0(x; \theta)\Bigp{ \pi_0(1\mid x; \theta) p_{Y(1), 0}(y\mid x; \theta)}^{\mathbbm{1}[o = 1]} \Bigp{\pi_0(0\mid x; \theta)p_{\widetilde{Y}}(y\mid x; \theta)}^{\mathbbm{1}[o = 0]},
\end{align*}
while there exists $\theta_0 \in \bbR$ such that
\[p(x, o, y; \theta_0) = p_0(x, o, y).\]

Then, we define scores as follows:
\begin{align*}
    &S(x, o, y; \theta) \coloneqq \frac{\partial}{\partial \theta} \log p(x, o, y; \theta)\\
    &= S_X(x; \theta) + \mathbbm{1}[o = 1]\p{S_{Y(1)}(y\mid x; \theta) + \frac{\dot{\pi}(1\mid x; \theta)}{\pi(1\mid x; \theta)}} + \mathbbm{1}[o = 0]\p{S_{\widetilde{Y}}(y\mid x; \theta) + \frac{\dot{\pi}(0\mid x; \theta)}{\pi(0\mid x; \theta)}},
\end{align*}
where
\begin{align*}
    S_X(x; \theta) &\coloneqq  \frac{d}{d \theta} \log \zeta(x; \theta),\\
    S_{Y(1)}(y\mid x; \theta) &\coloneqq \frac{d}{d \theta} \log p_{Y(1)}(y\mid x; \theta),\\
    S_{\widetilde{Y}}(y\mid x; \theta) &\coloneqq \frac{d}{d \theta} \log p_{\widetilde{Y}}(y\mid x; \theta),\\
    \dot{\pi}(0\mid x; \theta) &\coloneqq \frac{d}{d \theta} \pi(o\mid x; \theta).
\end{align*}

Let $\calT \coloneqq \{S(x, o, y; \theta)\}$ be the tangent space.

Here, note that 
\[p_{\widetilde{Y}}(y\mid x; \theta) = g_0(1\mid x)p_{Y(1)}(y\mid x; \theta) + g_0(0\mid x)p_{Y(0)}(y\mid x; \theta).\]
We have
\[p_{\rmC}(y\mid x; \theta) = \frac{1}{g_0(0\mid x)}\Bigp{p_{\widetilde{Y}}(y\mid x; \theta) - g_0(1\mid x)p_{Y(1)}(y\mid x; \theta)}\]

Using this relationship, we write the ATE under the parametric submodels as
\begin{align*}
    \tau(\theta) &\coloneqq \iint y(1) p_{Y(1)}(y(1)\mid x; \theta) \zeta(x; \theta) \rmd y(1) \rmd x - \iint y(0) p_{Y(0)}(y(0)\mid x; \theta) \zeta(x; \theta) \rmd y(0) \rmd x\\
    &= \iint y(1) p_{Y(1)}(y(1)\mid x; \theta) \zeta(x; \theta) \rmd y(1) \rmd x - \iint y(0) \frac{1}{g_0(0\mid x)}p_{\widetilde{Y}}(y(0)\mid x; \theta) \zeta(x; \theta) \rmd y(0) \rmd x\\
    &\ \ \ \ \ \ \ \ \ \ \ \ \ \ \ \ \ \ \ \ \ \ \ \ \ \ \ \ \ \ \ \ \ \ \ \ \ \ \ \ \ \ \ \ \ \ \ \ \ \ \ \ \ \ \ \ \ \ \ \ + \iint y(0) \frac{g_0(1\mid x)}{g_0(0\mid x)}p_{Y(1)}(y(0)\mid x; \theta) \zeta(x; \theta) \rmd y(0) \rmd x.
\end{align*}

Them, the derivative is given as 
\begin{align*}
    \frac{\partial \tau(\theta)}{\partial \theta} &= \bbE_{\theta}\Bigsqb{Y(1) S_{Y(1)}(Y(1) \mid X; \theta)} - \bbE_{\theta}\sqb{\frac{1}{g_0(0\mid X)}\widetilde{Y} S_{\widetilde{Y}}(\widetilde{Y} \mid X; \theta)}\\
    &\ \ \ + \bbE_{\theta}\sqb{\frac{g_0(1\mid X)}{g_0(0\mid X)}Y(1) S_{Y(1)}(Y(1) \mid X; \theta)}\\
    &\ \ \ + \bbE_{\theta}\Bigsqb{\tau(X; \theta)S_X(X; \theta)},
\end{align*}
where 
\[\tau(X; \theta) \coloneqq \mu(1\mid X) - \frac{1}{g_0(0\mid X)}\mu(U\mid X) + \frac{g_0(1\mid X)}{g_0(0\mid X)}\mu_{\rmT}(X) = \mu(1\mid X) - \mu_{\rmC}(X)\]

From the Riesz representation theorem, there exists a function $\Psi$ such that
\begin{align}
\label{eq:riesz_cens}
\frac{\partial \tau(\theta)}{\partial \theta}\Big|_{\theta = \theta_0} = \bbE\bigsqb{\Psi(X, O, Y)S(X, O, Y; \theta_0)}.    
\end{align}

There exists a unique function $\Psi^{\mathrm{cens}}$ such that $\Psi^{\mathrm{cens}} \in \calT$, called the efficient influence function. 
We specify the efficient influence function as
\begin{align*}
    \Psi^{\mathrm{cens}}(X, O, Y; \mu_{\rmT, 0}, \nu_0, \pi_0, g_0) &= S^{\mathrm{cens}}(X, O, Y; \mu_{\rmT, 0}, \nu_0, \pi_0, g_0) - \tau_0,\\
    &= \frac{\mathbbm{1}[O = 1]\Bigp{Y - \mu_{\rmT, 0}(X)}}{\pi_0(1\mid X)} - \frac{\mathbbm{1}[O = 0]\Bigp{Y - \nu_0(X)}}{g_0(0\mid X)\pi_0(0\mid X)}\\
    &\ \ \ +\frac{g_0(1\mid X)\mathbbm{1}[O = 1]\Bigp{Y - \mu_{\rmT, 0}(X)}}{g_0(0\mid X)\pi_0(1\mid X)}\\
    &\ \ \ +  \mu_{\rmT, 0}(X) - \frac{1}{g_0(0\mid X)}\nu(X) + \frac{g_0(1\mid X)}{g_0(0\mid X)}\mu_{\rmT, 0}(X) - \tau_0. 
\end{align*}

We prove that $\Psi^{\mathrm{cens}}(X, O, Y; \mu_{\rmT, 0}, \nu_0, \pi_0, g_0)$ is actually the unique efficient influence function by verifying that $\Psi^{\mathrm{cens}}$ satisfies \eqref{eq:riesz_cens} and $\Psi^{\mathrm{cens}} \in \calT$.

\paragraph{Proof of \eqref{eq:riesz_cens}:}
First, we confirm that $\Psi^{\mathrm{cens}}$ satisfies \eqref{eq:riesz_cens}. We have
\begin{align*}
    &\bbE\sqb{\Psi^{\mathrm{cens}}(X, O, Y; \mu_{\rmT, 0}, \nu_0, \pi_0, g_0)S(X, O, Y; \theta_0)}\\
    &= \bbE\Biggsqb{\Psi^{\mathrm{cens}}(X, O, Y; \mu_{\rmT, 0}, \nu_0, \pi_0, g_0)\\
    &\ \ \ \ \cdot \Bigp{S_X(X; \theta) + \mathbbm{1}[O = 1]\p{S_{Y(1)}(Y\mid X; \theta_0) + \frac{\dot{\pi}(1\mid X; \theta)}{\pi(1\mid X; \theta_0)}}\\
    &\ \ \ \ \ \ \ \ \ \ \ \ \ \ \ \ \ \ \ \ \ \ \ \ \ \ \ \ \ \ \ \ \ \ \ \ \ \ \ \ \ \ \ \ \ \ \ \ + \mathbbm{1}[O = 0]\p{S_{\widetilde{Y}}(Y\mid X; \theta) + \frac{\dot{\pi}(0\mid X; \theta_0)}{\pi(0\mid X; \theta_0)}}}}\\
    &= \bbE\Biggsqb{\Biggp{\frac{\mathbbm{1}[O = 1]\Bigp{Y - \mu_{\rmT, 0}(X)}}{\pi_0(1\mid X)} - \frac{\mathbbm{1}[O = 0]\Bigp{Y - \nu_0(X)}}{g_0(0\mid X)\pi_0(0\mid X)}\\
    &\ \ \ \ \ \ +\frac{g_0(1\mid X)\mathbbm{1}[O = 1]\Bigp{Y - \mu_{\rmT, 0}(X)}}{g_0(0\mid X)\pi_0(1\mid X)}\\
    &\ \ \ \ \ \ +  \mu_{\rmT, 0}(X) - \frac{1}{g_0(0\mid X)}\nu(X) + \frac{g_0(1\mid X)}{g_0(0\mid X)}\mu_{\rmT, 0}(X) - \tau_0}\\
    &\ \ \ \ \cdot \Biggp{S_X(X; \theta) + \mathbbm{1}[O = 1]\p{S_{Y(1)}(Y\mid X; \theta_0) + \frac{\dot{\pi}(1\mid X; \theta)}{\pi(1\mid X; \theta_0)}}\\
    &\ \ \ \ \ \ \ \ \ \ \ \ \ \ \ \ \ \ \ \ \ \ \ \ \ \ \ \ \ \ \ \ \ \ \ \ \ \ \ \ \ \ \ \ \ \ \ \  + \mathbbm{1}[O = 0]\p{S_{\widetilde{Y}}(Y\mid X; \theta) + \frac{\dot{\pi}(0\mid X; \theta_0)}{\pi(0\mid X; \theta_0)}}}}\\
    &= \bbE\Biggsqb{\Biggp{\mu_{\rmT, 0}(X) - \frac{1}{g_0(0\mid X)}\nu(X) + \frac{g_0(1\mid X)}{g_0(0\mid X)}\mu_{\rmT, 0}(X) - \tau_0}S_X(X; \theta_0)\\
    &\ \ \ \ + \Biggp{\frac{\mathbbm{1}[O = 1]\Bigp{Y - \mu_{\rmT, 0}(X)}}{\pi_0(1\mid X)} +\frac{g_0(1\mid X)\mathbbm{1}[O = 1]\Bigp{Y - \mu_{\rmT, 0}(X)}}{g_0(0\mid X)\pi_0(1\mid X)}\\
    &\ \ \ \ \ \ +  \mu_{\rmT, 0}(X) - \frac{1}{g_0(0\mid X)}\nu(X) + \frac{g_0(1\mid X)}{g_0(0\mid X)}\mu_{\rmT, 0}(X) - \tau_0}\mathbbm{1}[O = 1]\p{S_{Y(1)}(Y\mid X; \theta_0) + \frac{\dot{\pi}(1\mid X; \theta_0)}{\pi(1\mid X; \theta_0)}}\\
    &\ \ \ \ \ \ + \Biggp{ - \frac{\mathbbm{1}[O = 0]\Bigp{Y - \nu_0(X)}}{g_0(0\mid X)\pi_0(0\mid X)} +  \mu_{\rmT, 0}(X) - \frac{1}{g_0(0\mid X)}\nu(X) + \frac{g_0(1\mid X)}{g_0(0\mid X)}\mu_{\rmT, 0}(X) - \tau_0}\\
    &\ \ \ \ \ \ \ \ \ \ \ \cdot \mathbbm{1}[O = 0]\p{S_{\widetilde{Y}}(Y\mid X; \theta) + \frac{\dot{\pi}(0\mid X; \theta_0)}{\pi(0\mid X; \theta_0)}}},
\end{align*}
where we used $\mathbbm{1}[O = 1]\mathbbm{1}[O = 0] = 0$, and 
\begin{align*}
    &\bbE\Biggsqb{\frac{\mathbbm{1}[O = 1]\Bigp{Y - \mu_{\rmT, 0}(X)}}{\pi_0(1\mid X)}} = \bbE\Biggsqb{\frac{\mathbbm{1}[O = 1]\Bigp{Y(1) - \mu_{\rmT, 0}(X)}}{\pi_0(1\mid X)}}\\
    &\ \ \ \ \ \ = \bbE\Biggsqb{\frac{\pi_0(1\mid X)\Bigp{\mu_{\rmT, 0}(X) - \mu_{\rmT, 0}(X)}}{\pi_0(1\mid X)}} = 0,\\
    &\bbE\Biggsqb{\frac{\mathbbm{1}[O = 0]\Bigp{Y - \nu_0(X)}}{g_0(0\mid X)\pi_0(0\mid X)}} = \bbE\Biggsqb{\frac{\mathbbm{1}[O = 0]\Bigp{\widetilde{Y} - \nu_0(X)}}{g_0(0\mid X)\pi_0(0\mid X)}} = \bbE\Biggsqb{\frac{\pi_0(0\mid X)\Bigp{\nu_0(X) - \nu_0(X)}}{g_0(0\mid X)\pi_0(0\mid X)}} = 0,\\
    &\bbE\Biggsqb{\frac{g_0(1\mid X)\mathbbm{1}[O = 1]\Bigp{Y - \mu_{\rmT, 0}(X)}}{g_0(0\mid X)\pi_0(1\mid X)}} = \bbE\Biggsqb{\frac{g_0(1\mid X)\pi_0(1\mid X)\Bigp{\mu_{\rmT, 0}(X) - \mu_{\rmT, 0}(X)}}{g_0(0\mid X)\pi_0(1\mid X)}} = 0.
\end{align*}

We have
\begin{align*}
    &\bbE\Biggsqb{\Biggp{\mu_{\rmT, 0}(X) - \frac{1}{g_0(0\mid X)}\nu(X) + \frac{g_0(1\mid X)}{g_0(0\mid X)}\mu_{\rmT, 0}(X) - \tau_0}S_X(X; \theta)\\
    &\ \ \ \ + \Biggp{\frac{\mathbbm{1}[O = 1]\Bigp{Y - \mu_{\rmT, 0}(X)}}{\pi_0(1\mid X)} +\frac{g_0(1\mid X)\mathbbm{1}[O = 1]\Bigp{Y - \mu_{\rmT, 0}(X)}}{g_0(0\mid X)\pi_0(1\mid X)}\\
    &\ \ \ \ \ \ +  \mu_{\rmT, 0}(X) - \frac{1}{g_0(0\mid X)}\nu(X) + \frac{g_0(1\mid X)}{g_0(0\mid X)}\mu_{\rmT, 0}(X) - \tau_0}\mathbbm{1}[O = 1]\p{S_{Y(1)}(Y\mid X; \theta_0) + \frac{\dot{\pi}(1\mid X; \theta)}{\pi(1\mid X; \theta_0)}}\\
    &\ \ \ \ \ \ + \Biggp{ - \frac{\mathbbm{1}[O = 0]\Bigp{Y - \nu_0(X)}}{g_0(0\mid X)\pi_0(0\mid X)} +  \mu_{\rmT, 0}(X) - \frac{1}{g_0(0\mid X)}\nu(X) + \frac{g_0(1\mid X)}{g_0(0\mid X)}\mu_{\rmT, 0}(X) - \tau_0}\\
    &\ \ \ \ \ \ \ \ \ \ \ \cdot \mathbbm{1}[O = 0]\p{S_{\widetilde{Y}}(Y\mid X; \theta) + \frac{\dot{\pi}(0\mid X; \theta_0)}{\pi(0\mid X; \theta_0)}}}\\
    &= \bbE\Biggsqb{\Biggp{\mu_{\rmT, 0}(X) - \frac{1}{g_0(0\mid X)}\nu(X) + \frac{g_0(1\mid X)}{g_0(0\mid X)}\mu_{\rmT, 0}(X)}S_X(X; \theta_0)\\
    &\ \ \ \ + \Biggp{\frac{\mathbbm{1}[O = 1]\Bigp{Y - \mu_{\rmT, 0}(X)}}{\pi_0(1\mid X)} +\frac{g_0(1\mid X)\mathbbm{1}[O = 1]\Bigp{Y - \mu_{\rmT, 0}(X)}}{g_0(0\mid X)\pi_0(1\mid X)}}S_{Y(1)}(Y\mid X; \theta_0)\\
    &\ \ \ \ \ \  - \frac{\mathbbm{1}[O = 0]\Bigp{Y - \nu_0(X)}}{g_0(0\mid X)\pi_0(0\mid X)}S_{\widetilde{Y}}(Y\mid X; \theta)}\\
    &= \bbE\Biggsqb{\Biggp{\mu_{\rmT, 0}(X) - \frac{1}{g_0(0\mid X)}\nu(X) + \frac{g_0(1\mid X)}{g_0(0\mid X)}\mu_{\rmT, 0}(X)}S_X(X; \theta_0)\\
    &\ \ \ \ + \Biggp{\frac{\mathbbm{1}[O = 1]Y(1)}{\pi_0(1\mid X)} +\frac{g_0(1\mid X)\mathbbm{1}[O = 1]Y(1)}{g_0(0\mid X)\pi_0(1\mid X)}}S_{Y(1)}(Y(1)\mid X; \theta)  - \frac{\mathbbm{1}[O = 0]\widetilde{Y}}{g_0(0\mid X)\pi_0(0\mid X)}S_{\widetilde{Y}}(\widetilde{Y}\mid X; \theta_0)},
\end{align*}
where we used 
\begin{align*}
    &\bbE\Biggsqb{\tau_0 S_X(X; \theta)} = 0\\
    &\bbE\Biggsqb{\Bigp{\mu_{\rmT, 0}(X) - \frac{1}{g_0(0\mid X)}\nu(X) + \frac{g_0(1\mid X)}{g_0(0\mid X)}\mu_{\rmT, 0}(X) - \tau_0}\mathbbm{1}[O = 1]\p{S_{Y(1)}(Y\mid X; \theta_0) + \frac{\dot{\pi}(1\mid X; \theta_0)}{\pi(1\mid X; \theta_0)}}}\\
    &\ \ \ \ \ \ = 0.
\end{align*}

Finally, we have
\begin{align*}
    &\bbE\Biggsqb{\Biggp{\mu_{\rmT, 0}(X) - \frac{1}{g_0(0\mid X)}\nu(X) + \frac{g_0(1\mid X)}{g_0(0\mid X)}\mu_{\rmT, 0}(X)}S_X(X; \theta_0)\\
    &\ \ \ \ + \Biggp{\frac{\mathbbm{1}[O = 1]Y(1)}{\pi_0(1\mid X)} +\frac{g_0(1\mid X)\mathbbm{1}[O = 1]Y(1)}{g_0(0\mid X)\pi_0(1\mid X)}}S_{Y(1)}(Y(1)\mid X; \theta_0)  - \frac{\mathbbm{1}[O = 0]\widetilde{Y}}{g_0(0\mid X)\pi_0(0\mid X)}S_{\widetilde{Y}}(\widetilde{Y}\mid X; \theta_0)}\\
    &= \bbE\Bigsqb{Y(1) S_{Y(1)}(Y(1) \mid X; \theta_0)} - \bbE\sqb{\frac{1}{g_0(0\mid X)}\widetilde{Y} S_{\widetilde{Y}}(\widetilde{Y} \mid X; \theta_0)}\\
    &\ \ \ + \bbE\sqb{\frac{g_0(1\mid X)}{g_0(0\mid X)}Y(1) S_{Y(1)}(Y(1) \mid X; \theta_0)}\\
    &\ \ \ + \bbE_{\theta_0}\Bigsqb{\tau(X; \theta)S_X(X; \theta_0)}\\
    &= \frac{\partial \tau(\theta)}{\partial \theta}\Big|_{\theta = \theta_0}
\end{align*}

\paragraph{Proof of $\Psi^{\mathrm{cens}} \in \calT$:}
Set
    \begin{align*}
        &S_{Y(1)}(y\mid x) = \frac{y - \bbE\bigsqb{Y(1)\mid X = x}}{\pi_0(1\mid x)},\\
        &S_{\widetilde{Y}}(y\mid x) = \frac{y - \bbE\bigsqb{\widetilde{Y}\mid X = x}}{\pi_0(0\mid x)},\\
        &S_X(X; \theta) = \mu_{\rmT, 0}(X) - \frac{1}{g_0(0\mid X)}\nu(X) + \frac{g_0(1\mid X)}{g_0(0\mid X)}\mu_{\rmT, 0}(X) - \tau_0.
    \end{align*}
Then,  $\Psi^{\mathrm{cens}} \in \calT$ holds. 
\end{proof}

\section{Proof of Theorem~\ref{thm:asymp_normal_censoring}: Semiparametric efficient ATE estimator under the censoring setting}
\label{appdx:normal_cens}
For simplicity, we consider two-fold cross-fitting; that is, $L = 2$. Without loss of generality, we assume that the sample size $n$ is even, and let $\overline{n} = n/2$. For each $\ell \in \{1, 2\}$, we denote the subset of the dataset in cross-fitting as
\[\calD^{(\ell)} \coloneqq \{(\widetilde{X}^{\ell}_i, \widetilde{O}^{(\ell)}_i, \widetilde{Y}^{(\ell)}_i)\}^{\overline{n}}_{i=1}.\]

We defined the estimator as
\begin{align*}
    \widehat{\tau}^{\mathrm{cens}\mathchar`-\mathrm{eff}}_n \coloneqq \frac{1}{n}\sum^n_{i=1}S^{\mathrm{cens}}(X_i, O_i, Y_i; \widehat{\mu}_{\rmT, n, i}, \widehat{\nu}_{n, i}, \widehat{\pi}_{n, i}, g_0),
\end{align*}
where recall that
\begin{align*}
    &S^{\mathrm{cens}}(X, O, Y; \widehat{\mu}_{\rmT, n, i}, \widehat{\nu}_{n, i}, \widehat{\pi}_{n, i}, g_0)\\
    &= \frac{\mathbbm{1}[O = 1]\Bigp{Y - \widehat{\mu}_{\rmT, n, i}(X)}}{\widehat{\pi}_{n, i}(1\mid X)} - \frac{\mathbbm{1}[O = 0]\Bigp{Y - \widehat{\nu}_{n, i}(X)}}{g_0(0\mid X)\widehat{\pi}_{n, i}(0\mid X)} +\frac{g_0(1\mid X)\mathbbm{1}[O = 1]\Bigp{Y - \widehat{\mu}_{\rmT, n, i}(X)}}{g_0(0\mid X)\widehat{\pi}_{n, i}(1\mid X)}\\
    &\ \ \ +  \widehat{\mu}_{\rmT, n, i}(X) - \frac{1}{g_0(0\mid X)}\nu(X) + \frac{g_0(1\mid X)}{g_0(0\mid X)}\widehat{\mu}_{\rmT, n, i}(X). 
\end{align*}

We have
\begin{align*}
    \widehat{\tau}^{\mathrm{cens}\mathchar`-\mathrm{eff}}_n &= \frac{1}{n}\sum^n_{i=1}S^{\mathrm{cens}}(X_i, O_i, Y_i; \widehat{\mu}_{\rmT, n, i}, \widehat{\nu}_{n, i}, \widehat{\pi}_{n, i}, g_0)\\
    &= \frac{1}{n}\sum^n_{i=1}S^{\mathrm{cens}}(X_i, O_i, Y_i; \mu_{\rmT, 0}, \nu_0, \pi_0, g_0) - \frac{1}{n}\sum^n_{i=1}S^{\mathrm{cens}}(X_i, O_i, Y_i; \mu_{\rmT, 0}, \nu_0, \pi_0, g_0)\\
    &\ \ \ \ \ \ + \frac{1}{n}\sum^n_{i=1}S^{\mathrm{cens}}(X_i, O_i, Y_i; \widehat{\mu}_{\rmT, n, i}, \widehat{\nu}_{n, i}, \widehat{\pi}_{n, i}, g_0).
\end{align*}

Here, if it holds that
\begin{align}
\label{eq:target_main_cens}
    &\frac{1}{n}\sum^n_{i=1}S^{\mathrm{cens}}(X_i, O_i, Y_i; \mu_{\rmT, 0}, \nu_0, \pi_0, g_0) - \frac{1}{n}\sum^n_{i=1}S^{\mathrm{cens}}(X_i, O_i, Y_i; \widehat{\mu}_{\rmT, n, i}, \widehat{\nu}_{n, i}, \widehat{\pi}_{n, i}, g_0) = o_p(1/\sqrt{n})
\end{align}
then we have
\begin{align*}
    \sqrt{n}\Bigp{\widehat{\tau}^{\mathrm{cens}\mathchar`-\mathrm{eff}}_n - \tau_0} 
    &= \frac{1}{\sqrt{n}}\sum^n_{i=1}S^{\mathrm{cens}}(X_i, O_i, Y_i; \mu_{\rmT, 0}, \nu_0, \pi_0, g_0) + o_p(1)\\
    &\xrightarrow{\rmd} \mathcal{N}(0, V^{\mathrm{cens}}),
\end{align*}
from the central limit theorem for i.i.d. random variables. 

Therefore, we prove Theorem~\ref{thm:asymp_normal_censoring} by showing \eqref{eq:target_main_cens}. We decompose the LHS of \eqref{eq:target_main_cens} as
\begin{align*}
        &\frac{1}{n}\sum^n_{i=1}S^{\mathrm{cens}}(X_i, O_i, Y_i; \mu_{\rmT, 0}, \nu_0, \pi_0, g_0) - \frac{1}{n}\sum^n_{i=1}S^{\mathrm{cens}}(X_i, O_i, Y_i; \widehat{\mu}_{\rmT, n, i}, \widehat{\nu}_{n, i}, \widehat{\pi}_{n, i}, g_0)\\
        &=  \frac{\overline{n}}{n}\sum_{\ell \in \{1, 2\}}\Biggp{\frac{1}{\overline{n}}\sum^m_{i=1}S^{\mathrm{cens}}(\widetilde{X}^{(\ell)}_i, \widetilde{O}^{(\ell)}_i, \widetilde{Y}^{(\ell)}_i; \mu_{\rmT, 0}, \nu_0, \pi_0, g_0)\\
        &\ \ \ \ \ \ \ \ \ - \frac{1}{\overline{n}}\sum^m_{i=1}S^{\mathrm{cens}}(\widetilde{X}^{(\ell)}_i, \widetilde{O}^{(\ell)}_i, \widetilde{Y}^{(\ell)}_i; \widehat{\mu}^{(\ell)}_{\rmT, n}, \widehat{\nu}^{(\ell)}_{n}, \widehat{\pi}^{(\ell)}_{n}, g_0)}.
\end{align*}

Here, we have
\begin{align*}
        &\frac{1}{\overline{n}}\sum^m_{i=1}S^{\mathrm{cens}}(\widetilde{X}^{(\ell)}_i, \widetilde{O}^{(\ell)}_i, \widetilde{Y}^{(\ell)}_i; \mu_{\rmT, 0}, \nu_0, \pi_0, g_0) - \frac{1}{\overline{n}}\sum^m_{i=1}S^{\mathrm{cens}}(\widetilde{X}^{(\ell)}_i, \widetilde{O}^{(\ell)}_i, \widetilde{Y}^{(\ell)}_i; \widehat{\mu}^{(\ell)}_{\rmT, n}, \widehat{\nu}^{(\ell)}_{n}, \widehat{\pi}^{(\ell)}_{n}, g_0)\\
        &= \frac{1}{\overline{n}}\sum^m_{i=1}S^{\mathrm{cens}}(\widetilde{X}^{(\ell)}_i, \widetilde{O}^{(\ell)}_i, \widetilde{Y}^{(\ell)}_i; \mu_{\rmT, 0}, \nu_0, \pi_0, g_0) - \frac{1}{\overline{n}}\sum^m_{i=1}S^{\mathrm{cens}}(\widetilde{X}^{(\ell)}_i, \widetilde{O}^{(\ell)}_i, \widetilde{Y}^{(\ell)}_i; \widehat{\mu}^{(\ell)}_{\rmT, n}, \widehat{\nu}^{(\ell)}_{n}, \widehat{\pi}^{(\ell)}_{n}, g_0)\\
        &\ \ \ - \Biggp{\bbE\Bigsqb{S^{\mathrm{cens}}(\widetilde{X}^{(\ell)}_i, \widetilde{O}^{(\ell)}_i, \widetilde{Y}^{(\ell)}_i; \mu_{\rmT, 0}, \nu_0, \pi_0, g_0)\mid \calC_{(\ell)}}\\
        &\ \ \ \ \ \ \ \ \  - \bbE\Bigsqb{S^{\mathrm{cens}}(\widetilde{X}^{(\ell)}_i, \widetilde{O}^{(\ell)}_i, \widetilde{Y}^{(\ell)}_i; \widehat{\mu}^{(\ell)}_{\rmT, n}, \widehat{\nu}^{(\ell)}_{n}, \widehat{\pi}^{(\ell)}_{n}, g_0)\mid \calC_{(\ell)}}}\\
        &\ \ \ + \Biggp{\bbE\Bigsqb{S^{\mathrm{cens}}(\widetilde{X}^{(\ell)}_i, \widetilde{O}^{(\ell)}_i, \widetilde{Y}^{(\ell)}_i; \mu_{\rmT, 0}, \nu_0, \pi_0, g_0)\mid \calC_{(\ell)}}\\
        &\ \ \ \ \ \ \ \ \  - \bbE\Bigsqb{S^{\mathrm{cens}}(\widetilde{X}^{(\ell)}_i, \widetilde{O}^{(\ell)}_i, \widetilde{Y}^{(\ell)}_i; \widehat{\mu}^{(\ell)}_{\rmT, n}, \widehat{\nu}^{(\ell)}_{n}, \widehat{\pi}^{(\ell)}_{n}, g_0)\mid \calC_{(\ell)}}}.
\end{align*}

To show \eqref{eq:target_main_cens}, we show the following two inequalities separately:
\begin{align}
        &\frac{1}{\overline{n}}\sum^m_{i=1}S^{\mathrm{cens}}(\widetilde{X}^{(\ell)}_i, \widetilde{O}^{(\ell)}_i, \widetilde{Y}^{(\ell)}_i; \mu_{\rmT, 0}, \nu_0, \pi_0, g_0) - \frac{1}{\overline{n}}\sum^m_{i=1}S^{\mathrm{cens}}(\widetilde{X}^{(\ell)}_i, \widetilde{O}^{(\ell)}_i, \widetilde{Y}^{(\ell)}_i; \widehat{\mu}^{(\ell)}_{\rmT, n}, \widehat{\nu}^{(\ell)}_{n}, \widehat{\pi}^{(\ell)}_{n}, g_0)\nonumber\\
        &\ \ \ - \Biggp{\bbE\Bigsqb{S^{\mathrm{cens}}(\widetilde{X}^{(\ell)}_i, \widetilde{O}^{(\ell)}_i, \widetilde{Y}^{(\ell)}_i; \mu_{\rmT, 0}, \nu_0, \pi_0, g_0)\mid \calC_{(\ell)}}\nonumber\\
        &\ \ \ \ \ \ \ \ \  - \bbE\Bigsqb{S^{\mathrm{cens}}(\widetilde{X}^{(\ell)}_i, \widetilde{O}^{(\ell)}_i, \widetilde{Y}^{(\ell)}_i; \widehat{\mu}^{(\ell)}_{\rmT, n}, \widehat{\nu}^{(\ell)}_{n}, \widehat{\pi}^{(\ell)}_{n}, g_0)\mid \calC_{(\ell)}}}\nonumber\\
        \label{eq:target_proof_target_cens1}
        &= o_p(1/\sqrt{n}),\\
        &\bbE\Bigsqb{S^{\mathrm{cens}}(\widetilde{X}^{(\ell)}_i, \widetilde{O}^{(\ell)}_i, \widetilde{Y}^{(\ell)}_i; \mu_{\rmT, 0}, \nu_0, \pi_0, g_0)\mid \calC_{(\ell)}}\nonumber\\
        &\ \ \ \ \ \ \ \ \  - \bbE\Bigsqb{S^{\mathrm{cens}}(\widetilde{X}^{(\ell)}_i, \widetilde{O}^{(\ell)}_i, \widetilde{Y}^{(\ell)}_i; \widehat{\mu}^{(\ell)}_{\rmT, n}, \widehat{\nu}^{(\ell)}_{n}, \widehat{\pi}^{(\ell)}_{n}, g_0)\mid \calC_{(\ell)}}\nonumber\\
        \label{eq:target_proof_target_cens2}
        &= o_p(1/\sqrt{n}).
\end{align}
Here, the LHS of the first inequality is referred to as the empirical process term, while the LHS of the second inequality is referred to as the second-order remainder term.

\subsection{Proof of \eqref{eq:target_proof_target_cens1}}
We aim to show that for any $\varepsilon > 0$, 
\begin{align}
\label{eq:target_prob_chebyshev}
        &\lim_{n\to\infty}\bbP\Biggp{\sqrt{\overline{n}}\Bigg|\frac{1}{\overline{n}}\sum^m_{i=1}S^{\mathrm{cens}}(\widetilde{X}^{(\ell)}_i, \widetilde{O}^{(\ell)}_i, \widetilde{Y}^{(\ell)}_i; \mu_{\rmT, 0}, \nu_0, \pi_0, g_0)\nonumber\\
        &\ \ \ \ \ \ \ \ \ \  - \frac{1}{\overline{n}}\sum^m_{i=1}S^{\mathrm{cens}}(\widetilde{X}^{(\ell)}_i, \widetilde{O}^{(\ell)}_i, \widetilde{Y}^{(\ell)}_i; \widehat{\mu}^{(\ell)}_{\rmT, n}, \widehat{\nu}^{(\ell)}_{n}, \widehat{\pi}^{(\ell)}_{n}, g_0)\nonumber\\
        &- \Biggp{\bbE\Bigsqb{S^{\mathrm{cens}}(\widetilde{X}^{(\ell)}_i, \widetilde{O}^{(\ell)}_i, \widetilde{Y}^{(\ell)}_i; \mu_{\rmT, 0}, \nu_0, \pi_0, g_0)\mid \calC_{(\ell)}}\nonumber\\
        &\ \ \ \ \ \ \ \ \ \ - \bbE\Bigsqb{S^{\mathrm{cens}}(\widetilde{X}^{(\ell)}_i, \widetilde{O}^{(\ell)}_i, \widetilde{Y}^{(\ell)}_i; \widehat{\mu}^{(\ell)}_{\rmT, n}, \widehat{\nu}^{(\ell)}_{n}, \widehat{\pi}^{(\ell)}_{n}, g_0) \mid \calC_{(\ell)}}}\Bigg| > \varepsilon}\nonumber\\
        &= 0.
\end{align}

We show \eqref{eq:target_prob_chebyshev} by showing that for any $\varepsilon > 0$, 
\begin{align}
\label{eq:target_prob_chebyshev2}
        &\lim_{n\to\infty}\bbP\Biggp{\sqrt{\overline{n}}\Bigg|\frac{1}{\overline{n}}\sum^m_{i=1}S^{\mathrm{cens}}(\widetilde{X}^{(\ell)}_i, \widetilde{O}^{(\ell)}_i, \widetilde{Y}^{(\ell)}_i; \mu_{\rmT, 0}, \nu_0, \pi_0, g_0)\nonumber\\
        &\ \ \ \ \ \ \ \ \ \ \ \  - \frac{1}{\overline{n}}\sum^m_{i=1}S^{\mathrm{cens}}(\widetilde{X}^{(\ell)}_i, \widetilde{O}^{(\ell)}_i, \widetilde{Y}^{(\ell)}_i; \widehat{\mu}^{(\ell)}_{\rmT, n}, \widehat{\nu}^{(\ell)}_{n}, \widehat{\pi}^{(\ell)}_{n}, g_0)\nonumber\\
        &\ \ \ - \Biggp{\bbE\Bigsqb{S^{\mathrm{cens}}(\widetilde{X}^{(\ell)}_i, \widetilde{O}^{(\ell)}_i, \widetilde{Y}^{(\ell)}_i; \mu_{\rmT, 0}, \nu_0, \pi_0, g_0)\mid \calC_{(\ell)}}\nonumber\\
        &\ \ \ \ \ \ \ \ \ \ \ \ - \bbE\Bigsqb{S^{\mathrm{cens}}(\widetilde{X}^{(\ell)}_i, \widetilde{O}^{(\ell)}_i, \widetilde{Y}^{(\ell)}_i; \widehat{\mu}^{(\ell)}_{\rmT, n}, \widehat{\nu}^{(\ell)}_{n}, \widehat{\pi}^{(\ell)}_{n}, g_0)\mid \calC_{(\ell)}}} \Bigg| \geq \varepsilon\mid \calC_{(\ell)}}\nonumber\\
        &= 0.
\end{align}

If \eqref{eq:target_prob_chebyshev2} holds, then \eqref{eq:target_prob_chebyshev} also holds from dominated convergence theorem.

We prove \eqref{eq:target_prob_chebyshev2} using Chebychev's inequality. From Chebychev's inequality we have 
\begin{align*}
        &\bbP\Biggp{\sqrt{\overline{n}}\Bigg|\frac{1}{\overline{n}}\sum^m_{i=1}S^{\mathrm{cens}}(\widetilde{X}^{(\ell)}_i, \widetilde{O}^{(\ell)}_i, \widetilde{Y}^{(\ell)}_i; \mu_{\rmT, 0}, \nu_0, \pi_0, g_0) - \frac{1}{\overline{n}}\sum^m_{i=1}S^{\mathrm{cens}}(\widetilde{X}^{(\ell)}_i, \widetilde{O}^{(\ell)}_i, \widetilde{Y}^{(\ell)}_i; \widehat{\mu}^{(\ell)}_{\rmT, n}, \widehat{\nu}^{(\ell)}_{n}, \widehat{\pi}^{(\ell)}_{n}, g_0)\nonumber\\
        &- \Biggp{\bbE\Bigsqb{S^{\mathrm{cens}}(\widetilde{X}^{(\ell)}_i, \widetilde{O}^{(\ell)}_i, \widetilde{Y}^{(\ell)}_i; \mu_{\rmT, 0}, \nu_0, \pi_0, g_0)\mid \calC_{(\ell)}}\\
        &\ \ \ \ \ \ \ \ \ \ \ \ - \bbE\Bigsqb{S^{\mathrm{cens}}(\widetilde{X}^{(\ell)}_i, \widetilde{O}^{(\ell)}_i, \widetilde{Y}^{(\ell)}_i; \widehat{\mu}^{(\ell)}_{\rmT, n}, \widehat{\nu}^{(\ell)}_{n}, \widehat{\pi}^{(\ell)}_{n}, g_0)\mid \calC_{(\ell)}}} \Bigg| \geq \varepsilon\mid \calC_{(\ell)}}\\
        &\leq \frac{\overline{n}}{\varepsilon}\mathrm{Var}\Biggp{\frac{1}{\overline{n}}\sum^m_{i=1}S^{\mathrm{cens}}(\widetilde{X}^{(\ell)}_i, \widetilde{O}^{(\ell)}_i, \widetilde{Y}^{(\ell)}_i; \mu_{\rmT, 0}, \nu_0, \pi_0, g_0) - \frac{1}{\overline{n}}\sum^m_{i=1}S^{\mathrm{cens}}(\widetilde{X}^{(\ell)}_i, \widetilde{O}^{(\ell)}_i, \widetilde{Y}^{(\ell)}_i; \widehat{\mu}^{(\ell)}_{\rmT, n}, \widehat{\nu}^{(\ell)}_{n}, \widehat{\pi}^{(\ell)}_{n}, g_0)\nonumber\\
        &- \Biggp{\bbE\Bigsqb{S^{\mathrm{cens}}(\widetilde{X}^{(\ell)}_i, \widetilde{O}^{(\ell)}_i, \widetilde{Y}^{(\ell)}_i; \mu_{\rmT, 0}, \nu_0, \pi_0, g_0)\mid \calC_{(\ell)}}\\
        &\ \ \ \ \ \ \ \ \ \ \ \  - \bbE\Bigsqb{S^{\mathrm{cens}}(\widetilde{X}^{(\ell)}_i, \widetilde{O}^{(\ell)}_i, \widetilde{Y}^{(\ell)}_i; \widehat{\mu}^{(\ell)}_{\rmT, n}, \widehat{\nu}^{(\ell)}_{n}, \widehat{\pi}^{(\ell)}_{n}, g_0)\mid \calC_{(\ell)}}}  
 \mid \calC_{(\ell)}}.
\end{align*}

Since observations are i.i.d. and the conditional mean of the target part is zero, we have
\begin{align}
        &m\mathrm{Var}\Biggp{\frac{1}{\overline{n}}\sum^m_{i=1}S^{\mathrm{cens}}(\widetilde{X}^{(\ell)}_i, \widetilde{O}^{(\ell)}_i, \widetilde{Y}^{(\ell)}_i; \mu_{\rmT, 0}, \nu_0, \pi_0, g_0) - \frac{1}{\overline{n}}\sum^m_{i=1}S^{\mathrm{cens}}(\widetilde{X}^{(\ell)}_i, \widetilde{O}^{(\ell)}_i, \widetilde{Y}^{(\ell)}_i; \widehat{\mu}^{(\ell)}_{\rmT, n}, \widehat{\nu}^{(\ell)}_{n}, \widehat{\pi}^{(\ell)}_{n}, g_0)\nonumber\\
        &- \Biggp{\bbE\Bigsqb{S^{\mathrm{cens}}(\widetilde{X}^{(\ell)}_i, \widetilde{O}^{(\ell)}_i, \widetilde{Y}^{(\ell)}_i; \mu_{\rmT, 0}, \nu_0, \pi_0, g_0)\mid \calC_{(\ell)}}\nonumber\\
        &\ \ \ \ \ \ \ \ \ \ \ \  - \bbE\Bigsqb{S^{\mathrm{cens}}(\widetilde{X}^{(\ell)}_i, \widetilde{O}^{(\ell)}_i, \widetilde{Y}^{(\ell)}_i; \widehat{\mu}^{(\ell)}_{\rmT, n}, \widehat{\nu}^{(\ell)}_{n}, \widehat{\pi}^{(\ell)}_{n}, g_0)\mid \calC_{(\ell)}}}  
        \mid \calC_{(\ell)}}\nonumber\\
        &=\mathrm{Var}\Biggp{S^{\mathrm{cens}}(\widetilde{X}^{(\ell)}_i, \widetilde{O}^{(\ell)}_i, \widetilde{Y}^{(\ell)}_i; \mu_{\rmT, 0}, \nu_0, \pi_0, g_0) - S^{\mathrm{cens}}(\widetilde{X}^{(\ell)}_i, \widetilde{O}^{(\ell)}_i, \widetilde{Y}^{(\ell)}_i; \widehat{\mu}^{(\ell)}_{\rmT, n}, \widehat{\nu}^{(\ell)}_{n}, \widehat{\pi}^{(\ell)}_{n}, g_0)\nonumber\\
        &- \Biggp{\bbE\Bigsqb{S^{\mathrm{cens}}(\widetilde{X}^{(\ell)}_i, \widetilde{O}^{(\ell)}_i, \widetilde{Y}^{(\ell)}_i; \mu_{\rmT, 0}, \nu_0, \pi_0, g_0)\mid \calC_{(\ell)}}\nonumber\\
        &\ \ \ \ \ \ \ \ \ \ \ \  - \bbE\Bigsqb{S^{\mathrm{cens}}(\widetilde{X}^{(\ell)}_i, \widetilde{O}^{(\ell)}_i, \widetilde{Y}^{(\ell)}_i; \widehat{\mu}^{(\ell)}_{\rmT, n}, \widehat{\nu}^{(\ell)}_{n}, \widehat{\pi}^{(\ell)}_{n}, g_0)\mid \calC_{(\ell)}}} \mid \calC_{(\ell)}}\nonumber\\
        \label{eq:last_step_proof}
        &=\bbE\Biggsqb{\Biggp{S^{\mathrm{cens}}(\widetilde{X}^{(\ell)}_i, \widetilde{O}^{(\ell)}_i, \widetilde{Y}^{(\ell)}_i; \mu_{\rmT, 0}, \nu_0, \pi_0, g_0) - S^{\mathrm{cens}}(\widetilde{X}^{(\ell)}_i, \widetilde{O}^{(\ell)}_i, \widetilde{Y}^{(\ell)}_i; \widehat{\mu}^{(\ell)}_{\rmT, n}, \widehat{\nu}^{(\ell)}_{n}, \widehat{\pi}^{(\ell)}_{n}, g_0)\\
        &- \Biggp{\bbE\Bigsqb{S^{\mathrm{cens}}(\widetilde{X}^{(\ell)}_i, \widetilde{O}^{(\ell)}_i, \widetilde{Y}^{(\ell)}_i; \mu_{\rmT, 0}, \nu_0, \pi_0, g_0)\mid \calC_{(\ell)}}\nonumber\\
        &\ \ \ \ \ \ \ \ \ \ \ \  - \bbE\Bigsqb{S^{\mathrm{cens}}(\widetilde{X}^{(\ell)}_i, \widetilde{O}^{(\ell)}_i, \widetilde{Y}^{(\ell)}_i; \widehat{\mu}^{(\ell)}_{\rmT, n}, \widehat{\nu}^{(\ell)}_{n}, \widehat{\pi}^{(\ell)}_{n}, g_0)\mid \calC_{(\ell)}}} }^2\mid \calC_{(\ell)}}.\nonumber
\end{align}

The term \eqref{eq:last_step_proof} converges to zero in probability as $n\to \infty$ if $\big\|\mu_{\rmT, 0} - \widehat{\mu}^{(\ell)}_{\rmT, n}\big\|_2 = o_p(1)$, $\big\|\nu_0 - \widehat{\nu}^{(\ell)}_{n}\big\|_2 = o_p(1)$, and $\big\|\pi_0 - \widehat{\pi}^{(\ell)}_{n}\big\|_2 = o_p(1)$ as $n\to \infty$. Thus, we complete the proof.

\subsection{Proof of \eqref{eq:target_proof_target_cens2}}
We have
\begin{align*}
        &\bbE\Bigsqb{S^{\mathrm{cens}}(\widetilde{X}^{(\ell)}_i, \widetilde{O}^{(\ell)}_i, \widetilde{Y}^{(\ell)}_i; \mu_{\rmT, 0}, \nu_0, \pi_0, g_0)\mid \calC_{(\ell)}} - \bbE\Bigsqb{S^{\mathrm{cens}}(\widetilde{X}^{(\ell)}_i, \widetilde{O}^{(\ell)}_i, \widetilde{Y}^{(\ell)}_i; \widehat{\mu}^{(\ell)}_{\rmT, n}, \widehat{\nu}^{(\ell)}_{n}, \widehat{\pi}^{(\ell)}_{n}, g_0)\mid \calC_{(\ell)}}\\
        &= \bbE\Biggsqb{\frac{\mathbbm{1}[O = 1]\Bigp{Y - \mu_{\rmT, 0}(X)}}{\pi_0(1\mid X)} - \frac{\mathbbm{1}[O = 0]\Bigp{Y - \nu_0(X)}}{g_0(0\mid X)\pi(0\mid X)} +\frac{g_0(1\mid X)\mathbbm{1}[O = 1]\Bigp{Y - \mu_{\rmT, 0}(X)}}{g_0(0\mid X)\pi_0(0\mid X)}\\
        &\ \ \ \ \ \ \ \ \ +  \mu_{\rmT, 0}(X) - \frac{1}{g_0(0\mid X)}\nu_0(X) + \frac{g_0(1\mid X)}{g_0(0\mid X)}\mu_{\rmT, 0}(X)\mid \calC_{(\ell)}}\\
        &\ \ \ - \bbE\Biggsqb{\frac{\mathbbm{1}[O = 1]\Bigp{Y - \widehat{\mu}^{(\ell)}_{\rmT, n}(X)}}{\widehat{\pi}^{(\ell)}_n(1\mid X)} - \frac{\mathbbm{1}[O = 0]\Bigp{Y - \widehat{\nu}^{(\ell)}_n(X)}}{g_0(0\mid X)\widehat{\pi}^{(\ell)}_n(0\mid X)}\\
        &\ \ \ \ \ \ \ \ \  +\frac{g_0(1\mid X)\mathbbm{1}[O = 1]\Bigp{Y - \widehat{\mu}^{(\ell)}_{\rmT, n}(X)}}{g_0(0\mid X)\widehat{\pi}^{(\ell)}_n(0\mid X)}\\
        &\ \ \ \ \ \ \ \ \ +  \widehat{\mu}^{(\ell)}_{\rmT, n}(X) - \frac{1}{g_0(0\mid X)}\widehat{\nu}^{(\ell)}_n(X) + \frac{g_0(1\mid X)}{g_0(0\mid X)}\widehat{\mu}^{(\ell)}_{\rmT, n}(X)\mid \calC_{(\ell)}}\\
        &= \bbE\Biggsqb{ \mu_{\rmT, 0}(X) - \frac{1}{g_0(0\mid X)}\nu_0(X) + \frac{g_0(1\mid X)}{g_0(0\mid X)}\mu_{\rmT, 0}(X)}\\
        &\ \ \ - \bbE\Biggsqb{\frac{\pi_0(1\mid X)\Bigp{\mu_{\rmT, 0}(X) - \widehat{\mu}^{(\ell)}_{\rmT, n}(X)}}{\widehat{\pi}^{(\ell)}_n(1\mid X)} - \frac{\pi_0(0\mid X)\Bigp{\nu_0(X) - \widehat{\nu}^{(\ell)}_n(X)}}{g_0(0\mid X)\widehat{\pi}^{(\ell)}_n(0\mid X)}\\
        &\ \ \ \ \ \ \ \ \  +\frac{g_0(1\mid X)\pi_0(1\mid X)\Bigp{\mu_{\rmT, 0}(X) - \widehat{\mu}^{(\ell)}_{\rmT, n}(X)}}{g_0(0\mid X)\widehat{\pi}^{(\ell)}_n(0\mid X)}\\
        &\ \ \ \ \ \ \ \ \ +  \widehat{\mu}^{(\ell)}_{\rmT, n}(X) - \frac{1}{g_0(0\mid X)}\widehat{\nu}^{(\ell)}_n(X) + \frac{g_0(1\mid X)}{g_0(0\mid X)}\widehat{\mu}^{(\ell)}_{\rmT, n}(X)}\\
        &=  \bbE\Biggsqb{\p{1 - \frac{\pi_0(1\mid X)}{\widehat{\pi}^{(\ell)}_n(1\mid X)}}\Bigp{\mu_{\rmT, 0}(X) - \widehat{\mu}^{(\ell)}_{\rmT, n}(X)}}\\
        &\ \ \ + \bbE\sqb{\frac{1}{g_0(0\mid X)}\widehat{\nu}^{(\ell)}_n(X)  - \frac{1}{g_0(0\mid X)}\nu_0(X)  - \frac{\pi_0(0\mid X)\Bigp{\widehat{\nu}^{(\ell)}_n(X) - \nu_0(X)}}{g_0(0\mid X)\widehat{\pi}^{(\ell)}_n(0\mid X)}}\\
        &\ \ \ + \bbE\sqb{\frac{g_0(1\mid X)}{g_0(0\mid X)}\mu_{\rmT, 0}(X) - \frac{g_0(1\mid X)}{g_0(0\mid X)}\widehat{\mu}^{(\ell)}_{\rmT, n}(X) - \frac{g_0(1\mid X)\pi_0(1\mid X)\Bigp{\mu_{\rmT, 0}(X) - \widehat{\mu}^{(\ell)}_{\rmT, n}(X)}}{g_0(0\mid X)\widehat{\pi}^{(\ell)}_n(0\mid X)}}\\
        &=  \bbE\Biggsqb{\p{1 - \frac{\pi_0(1\mid X)}{\widehat{\pi}^{(\ell)}_n(1\mid X)}}\Bigp{\mu_{\rmT, 0}(X) - \widehat{\mu}^{(\ell)}_{\rmT, n}(X)}}\\
        &\ \ \ + \bbE\sqb{\frac{1}{g_0(0\mid X)}\p{ 1 - \frac{\pi_0(0\mid X)}{\widehat{\pi}^{(\ell)}_n(0\mid X)}}\p{\widehat{\nu}^{(\ell)}_n(X) - \nu_0(X)}}\\
        &\ \ \ + \bbE\sqb{\frac{g_0(1\mid X)}{g_0(0\mid X)}\p{1 - \frac{\pi_0(1\mid X)}{\widehat{\pi}^{(\ell)}_n(0\mid X)}}\Bigp{\mu_{\rmT, 0}(X) - \widehat{\mu}^{(\ell)}_{\rmT, n}(X)}}\\
        &\leq C\sqrt{\bbE\Biggsqb{\p{\widehat{\pi}^{(\ell)}_n(1\mid X) - \pi_0(1\mid X)}^2}\bbE\Biggsqb{\Bigp{\mu_{\rmT, 0}(X) - \widehat{\mu}^{(\ell)}_{\rmT, n}(X)}^2}}\\
        &\ \ \ + C\sqrt{\bbE\sqb{\p{\widehat{\pi}^{(\ell)}_n(0\mid X) - \pi_0(0\mid X)}^2}\bbE\sqb{\p{\widehat{\nu}^{(\ell)}_n(X) - \nu_0(X)}^2}}\\
        &\ \ \ + C\sqrt{\bbE\sqb{\p{\widehat{\pi}^{(\ell)}_n(0\mid X) - \pi_0(1\mid X)}^2}\bbE\sqb{\Bigp{\mu_{\rmT, 0}(X) - \widehat{\mu}^{(\ell)}_{\rmT, n}(X)}^2}}\\
        &= o_p(1/\sqrt{n}),
\end{align*}
where we used H\"older's inequality.

\section{Proof of Lemma~\ref{lem:efficient_case_control}}
\label{appdx:eb_cc}
Our proof is inspired by those in \citet{Uehara2020offpolicy} and \citet{Kato2024activeadaptive}. \citet{Uehara2020offpolicy} revisits the efficiency bound under the stratified sampling scheme, a generalization of the case-control setting, studied by \citet{Wooldridge2001asymptoticproperties} and \citet{Imbens2009recentdevelopments}. In the stratified sampling, we define an efficiency bound by regarding $\Big(\calD_{\rmT}, \calD_{\rmU}\Big)$ as one sample.

Their proof considers a nonparametric model for the distribution of potential outcomes and defines regular subparametric models. Then, (i) we characterize the tangent set for all regular parametric submodels, (ii) verify that the parameter of interest is pathwise differentiable, (iii) verify that a guessed semiparametric 
efficient influence function lies in the tangent set, and (iv) calculate the expected square of the influence function.

In the case-control setting, the observations are generated as follows:
\begin{align*}
    &\calD_{\rmT} \coloneqq \cb{\bigp{X_{\rmT, j}, Y_j(1)}}^m_{j=1}, \quad \bigp{X_{\rmT, j}, Y_j(1)} \sim  p_{\rmT, 0}(x, y(1)) = \zeta_{\rmT, 0}(x)p_{Y(1), 0}(y(1)\mid x),\\
    &\calD_{\rmU} \coloneqq \cb{\bigp{X_k, Y_{\rmU, k}}}^l_{k=1}, \quad \bigp{X_k, Y_{\rmU, k}} \sim p_{\rmU, 0}(x, y_{\rmU}) = \zeta_0(x)p_{Y_\rmU, 0}(y_{\rmU}\mid x).
\end{align*}

We derive the efficiency bound by regarding \[\calE=\Big(\calD_{\rmT}, \calD_{\rmU}\Big)\]
as one observation.

We define regular parametric submodels
\begin{align*}
    \calP^{\mathrm{sub}} \coloneqq \{P_{\rmT, \theta}, P_{\rmU, \theta}\colon \theta \in \bbR\},
\end{align*}
where $P_{\rmT, \theta}$ is a parametric submodel for the distribution of $X_{\rmT, j}, Y_j(1)$ and $P_{\rmU, \theta}$ is a parametric submodel for the distribution of $X_k, Y_{\rmU, k}$. 

We denote the probability densities under $P_{\rmT, \theta}$ and $P_{\rmU, \theta}$ by 
\begin{align*}
    &p_{\rmT}(x, y; \theta) = \zeta_{\rmT}(x; \theta)p_{Y(1), 0}(y(1)\mid x; \theta),\\
    &p_{\rmU}(x, y; \theta) = \zeta(x; \theta)p_{Y_\rmU}(y_{\rmU}\mid x; \theta).
\end{align*}

We consider the joint log-likelihood of $\mathcal{D}_T$ and $\mathcal{D}_S$, which is defined as
\begin{align*}
    \sum^m_{j=1}\log \p{p_{\rmT}(X_{\rmT, j}, Y_j(1); \theta)} + \sum^l_{k=1}\log\p{p_{\rmU}(X_k, Y_{\rmU, k}; \theta)}. 
\end{align*}

By taking the derivative of $\sum^m_{j=1}\log \p{p_{\rmT}(X_{\rmT, j}, Y_j(1); \theta)} + \sum^l_{k=1}\log\p{p_{\rmU}(X_k, Y_{\rmU, k}; \theta)}$ with respect to $\beta$, we can obtain the corresponding score as
\begin{align*}
    &S(\calE; \theta) \coloneqq \frac{\mathrm{d}}{\mathrm{d}\theta} \p{\sum^m_{j=1}\log \p{p_{\rmT}(X_{\rmT, j}, Y_j(1); \theta)} + \sum^l_{k=1}\log\p{p_{\rmU}(X_k, Y_{\rmU, k}; \theta)}}\\
    &= \sum^m_{j=1}S_{X_{\rmT}}(X_{\rmT, j}; \theta) + \sum^m_{j=1}S_{Y(1)}(Y_j(1)\mid X_{\rmT, j}; \theta) + \sum^l_{k=1} S_X(X_k; \theta) + \sum^l_{k=1} S_{Y_{\rmU}}(Y_{\rmU, k}\mid X_k; \theta).
\end{align*}
where 
\begin{align*}
    S_{X_{\rmT}}(x; \theta) &\coloneqq  \frac{d}{d \theta} \log \zeta_{\rmT}(x; \theta),\\
    S_{Y(1)}(y\mid x; \theta) &\coloneqq \frac{d}{d \theta} \log p_{Y(1)}(y\mid x; \theta),\\
    S_X(x; \theta) &\coloneqq  \frac{d}{d \theta} \log \zeta(x; \theta),\\
    S_{Y_{\rmU}}(y\mid x; \theta) &\coloneqq \frac{d}{d \theta} \log p_{Y_{\rmU}}(y\mid x; \theta),
\end{align*}

Let us also define 
\begin{align*}
    &S^{\mathrm{cc}~(\mathrm{T})}(x, y; \mu_{\rmT}, e, r) = S_{X_{\rmT}}(x; \theta) + S_{Y(1)}(y\mid x; \theta),\\
    &S^{\mathrm{cc}~(\mathrm{U})}(X, Y_{\rmU}; \mu_{\rmT}, \mu_{\rmU}, e) = S_X(x; \theta) + S_{Y_{\rmU}}(y\mid x; \theta).
\end{align*}

Here, note that 
\[p_{Y_{\rmU}}(y\mid x; \theta) = e_0(1\mid x)p_{Y(1)}(y\mid x; \theta) + e_0(0\mid x)p_{Y(0)}(y\mid x; \theta).\]
We have
\[p_{Y(0)}(y\mid x; \theta) = \frac{1}{e_0(0\mid x)}\Bigp{p_{Y_{\rmU}}(y\mid x; \theta) - e_0(1\mid x)p_{Y(1)}(y\mid x; \theta)}\]

Using this relationship, we write the ATE under the parametric submodels as
\begin{align*}
    \tau(\theta) &\coloneqq \iint y(1) p_{Y(1)}(y(1)\mid x; \theta) \zeta(x; \theta) \rmd y(1) \rmd x - \iint y(0) p_{Y(0)}(y(0)\mid x; \theta) \zeta(x; \theta) \rmd y(0) \rmd x\\
    &= \iint y(1) p_{Y(1)}(y(1)\mid x; \theta) \zeta(x; \theta) \rmd y(1) \rmd x - \iint y(0) \frac{1}{e_0(0\mid x)}p_{Y_{\rmU}}(y(0)\mid x; \theta) \zeta(x; \theta) \rmd y(0) \rmd x\\
    &\ \ \ \ \ \ \ \ \ \ \ \ \ \ \ \ \ \ \ \ \ \ \ \ \ \ \ \ \ \ \ \ \ \ \ \ \ \ \ \ \ \ \ \ \ \ \ \ \ \ \ \ \ \ \ \ \ \ \ \ + \iint y(0) \frac{e_0(1\mid x)}{e_0(0\mid x)}p_{Y(1)}(y(0)\mid x; \theta) \zeta(x; \theta) \rmd y(0) \rmd x.
\end{align*}

The tangent space for this parametric submodel at $\theta = \theta_0$ is given as
\begin{align*}
&\mathcal{T} \coloneqq \\
&\cb{\sum^m_{j=1}S_{X_{\rmT}}(X_{\rmT, j}; \theta_0) + \sum^m_{j=1}S_{Y(1)}(Y_j(1)\mid X_{\rmT, j}; \theta_0) + \sum^l_{k=1} S_X(X_k; \theta_0) + \sum^l_{k=1} S_{Y_{\rmU}}(Y_{\rmU, k}\mid X_k; \theta_0)  \in L_2(\calE)}.
\end{align*}

From the Riesz representation theorem, there exists a function $\widetilde{\Psi}$ such that
\begin{align}
\label{eq:riesz_cc}
\frac{\partial \tau(\theta)}{\partial \theta}\Big|_{\theta = \theta_0} = \bbE\bigsqb{\widetilde{\Psi}(\calE)S(\calE; \theta_0)}.    
\end{align}

There exists a unique function $\Psi^{\mathrm{cc}}$ such that $\Psi^{\mathrm{cc}} \in \calT$, called the efficient influence function. 
We specify the efficient influence function as
\begin{align*}
    &\widetilde{\Psi}^{\mathrm{cc}}(\calE; \mu_{\rmT, 0}, \mu_{\rmU, 0}, e_0, r_0),\\
    &= \frac{1}{m}\sum^m_{j=1}\p{\p{1 - \frac{e_0(1\mid X_{\rmT, j})}{e_0(0\mid X_{\rmT, j})}}\Bigp{Y_j(1) - \mu_{\rmT, 0}(X)}}r_0(X_{\rmT, j}),\\
    &\ \ \ + \frac{1}{l}\sum^l_{k=1}\p{\frac{\Bigp{Y_{\rmU, k} - \mu_{\rmU, 0}(X_k)}}{e_0(0\mid X_k)} +  \mu_{\rmT, 0}(X_k) - \frac{1}{e_0(0\mid X_k)}\mu_{\rmU, 0}(X) + \frac{e_0(1\mid X_k)}{e_0(0\mid X_k)}\mu_{\rmT, 0}(X_k)}- \tau_0.
\end{align*}

We prove that $\widetilde{\Psi}^{\mathrm{cc}}(X, O, Y; \mu_{\rmT, 0}, \nu_0, \pi_0, g_0)$ is actually the unique efficient influence function by verifying that $\widetilde{\Psi}^{\mathrm{cc}}$ satisfies \eqref{eq:riesz_cc} and $\widetilde{\Psi}^{\mathrm{cc}} \in \calT$.

\paragraph{Proof of \eqref{eq:riesz_cc}:}
First, we confirm that $\widetilde{\Psi}^{\mathrm{cc}}$ satisfies \eqref{eq:riesz_cc}. We have
\begin{align*}
    &\bbE\sqb{\widetilde{\Psi}^{\mathrm{cc}}(\calE; \mu_{\rmT, 0}, \nu_0, \pi_0, g_0)S(\calE; \theta_0)}\\
    &= \bbE\Biggsqb{\Psi^{\mathrm{cc}}(\calE; \mu_{\rmT, 0}, \nu_0, \pi_0, g_0)\\
    &\ \ \ \ \cdot \p{\sum^m_{j=1}S_{X_{\rmT}}(X_{\rmT, j}; \theta_0) + \sum^m_{j=1}S_{Y(1)}(Y_j(1)\mid X_{\rmT, j}; \theta_0) + \sum^l_{k=1} S_X(X_k; \theta_0) + \sum^l_{k=1} S_{Y_{\rmU}}(Y_{\rmU, k}\mid X_k; \theta_0)}}\\
    &= \bbE\Biggsqb{\Biggp{\frac{1}{m}\sum^m_{j=1}\p{\p{1 - \frac{e_0(1\mid X_{\rmT, j})}{e_0(0\mid X_{\rmT, j})}}\Bigp{Y_j(1) - \mu_{\rmT}(X_{\rmT, j})}}r_0(X_{\rmT, j})\\
    &\ \ \ + \frac{1}{l}\sum^l_{k=1}\p{\frac{\Bigp{Y_{\rmU, k} - \mu_{\rmU, 0}(X_k)}}{e_0(0\mid X_k)} +  \mu_{\rmT, 0}(X_k) - \frac{1}{e_0(0\mid X_k)}\mu_{\rmU}(X_k) + \frac{e_0(1\mid X_k)}{e_0(0\mid X_k)}\mu_{\rmT, 0}(X_k)}- \tau_0}\\
    &\ \ \ \ \cdot \p{\sum^m_{j=1}S_{X_{\rmT}}(X_{\rmT, j}; \theta_0) + \sum^m_{j=1}S_{Y(1)}(Y_j(1)\mid X_{\rmT, j}; \theta_0) + \sum^l_{k=1} S_X(X_k; \theta_0) + \sum^l_{k=1} S_{Y_{\rmU}}(Y_{\rmU, k}\mid X_k; \theta_0)}}.
\end{align*}
Since $\calD_{\rmT}$ and $\calD_{\rmU}$ are independent and observations are i.i.d., we have
\begin{align*}
    &\bbE\Biggsqb{\Biggp{\frac{1}{m}\sum^m_{j=1}\p{\p{1 - \frac{e_0(1\mid X_{\rmT, j})}{e_0(0\mid X_{\rmT, j})}}\Bigp{Y_j(1) - \mu_{\rmT}(X_{\rmT, j})}}r_0(X_{\rmT, j})\\
    &\ \ \ + \frac{1}{l}\sum^l_{k=1}\p{\frac{\Bigp{Y_{\rmU, k} - \mu_{\rmU, 0}(X_k)}}{e_0(0\mid X_k)} +  \mu_{\rmT, 0}(X_k) - \frac{1}{e_0(0\mid X_k)}\mu_{\rmU}(X_k) + \frac{e_0(1\mid X_k)}{e_0(0\mid X_k)}\mu_{\rmT, 0}(X_k)- \tau_0}}\\
    &\ \ \ \ \cdot \p{\sum^m_{j=1}S_{X_{\rmT}}(X_{\rmT, j}; \theta_0) + \sum^m_{j=1}S_{Y(1)}(Y_j(1)\mid X_{\rmT, j}; \theta_0) + \sum^l_{k=1} S_X(X_k; \theta_0) + \sum^l_{k=1} S_{Y_{\rmU}}(Y_{\rmU, k}\mid X_k; \theta_0)}}\\
    &= \bbE\Biggsqb{\p{1 - \frac{e_0(1\mid X_{\rmT, j})}{e_0(0\mid X_{\rmT, j})}}\Bigp{Y_j(1) - \mu_{\rmT, 0}(X_{\rmT, j})}r_0(X_{\rmT, j}) \p{S_{X_{\rmT}}(X_{\rmT, j}; \theta_0) + S_{Y(1)}(Y_j(1)\mid X_{\rmT, j}; \theta_0)}}\\
    &\ \ \ + \bbE\Biggsqb{\p{\frac{\Bigp{Y_{\rmU, k} - \mu_{\rmU, 0}(X_k)}}{e_0(0\mid X_k)} +  \mu_{\rmT, 0}(X_k) - \frac{1}{e_0(0\mid X_k)}\mu_{\rmU, 0}(X_k) + \frac{e_0(1\mid X_k)}{e_0(0\mid X_k)}\mu_{\rmT, 0}(X_k) - \tau_0}\\
    &\ \ \ \ \ \ \ \ \ \ \ \ \cdot \p{S_X(X_k; \theta_0) + S_{Y_{\rmU}}(Y_{\rmU, k}\mid X_k; \theta_0)}}.
\end{align*}

Because the density ratio allows us to change the measure, we have
\begin{align*}
    &\bbE\Biggsqb{\p{1 - \frac{e_0(1\mid X_{\rmT, j})}{e_0(0\mid X_{\rmT, j})}}\Bigp{Y_j(1) - \mu_{\rmT, 0}(X_{\rmT, j})}r_0(X_{\rmT, j}) \p{S_{X_{\rmT}}(X_{\rmT, j}; \theta_0) + S_{Y(1)}(Y_j(1)\mid X_{\rmT, j}; \theta_0)}}\\
    &=\bbE\Biggsqb{\p{1 - \frac{e_0(1\mid X)}{e_0(0\mid X)}}\Bigp{Y(1) - \mu_{\rmT, 0}(X)}\p{S_{X}(X_{\rmT, j}; \theta_0) + S_{Y(1)}(Y(1)\mid X; \theta_0)}}
\end{align*}

Finally, we have
\begin{align*}
    &\bbE\Biggsqb{\p{1 - \frac{e_0(1\mid X)}{e_0(0\mid X)}}\Bigp{Y(1) - \mu_{\rmT}(X)}\p{S_{X}(X_{\rmT, j}; \theta_0) + S_{Y(1)}(Y(1)\mid X; \theta_0)}}\\
    &\ \ \ + \bbE\Biggsqb{\p{\frac{\Bigp{Y_{\rmU} - \mu_{\rmU, 0}(X_k)}}{e_0(0\mid X)} +  \mu_{\rmT, 0}(X) - \frac{1}{e_0(0\mid X)}\mu_{\rmU}(X) + \frac{e_0(1\mid X)}{e_0(0\mid X_k)}\mu_{\rmT, 0}(X) - \tau_0}\\
    &\ \ \ \ \ \ \ \ \ \ \ \ \cdot \p{S_X(X; \theta_0) + S_{Y_{\rmU}}(Y_{\rmU}\mid X; \theta_0)}}\\
    &= \bbE_{\theta}\Bigsqb{Y(1) S_{Y(1)}(Y(1) \mid X; \theta_0)} - \bbE_{\theta}\sqb{\frac{1}{e_0(0\mid X)}Y_{\rmU} S_{Y_{\rmU}}(Y_{\rmU} \mid X; \theta_0)}\\
    &\ \ \ + \bbE_{\theta}\sqb{\frac{e_0(1\mid X)}{e_0(0\mid X)}Y(1) S_{Y(1)}(Y(1) \mid X; \theta_0)}\\
    &\ \ \ + \bbE_{\theta}\Bigsqb{\tau(X; \theta)S_X(X; \theta)}\\
    &= \frac{\partial \tau(\theta)}{\partial \theta}\Big|_{\theta = \theta_0}
\end{align*}

\paragraph{Proof of $\widetilde{\Psi}^{\mathrm{cc}} \in \calT$:}
Set
\begin{align*}
    &S^{\mathrm{cc}~(\mathrm{T})}(X, Y(1); \mu_{\rmT, 0}, e_0, r_0) = \p{\p{1 - \frac{e_0(1\mid X)}{e_0(0\mid X)}}\Bigp{Y(1) - \mu_{\rmT, 0}(X)}}r(X),\\
    &S^{\mathrm{cc}~(\mathrm{U})}(X, Y_{\rmU}; \mu_{\rmT, 0}, \mu_{\rmU, 0}, e_0) =  \frac{\Bigp{Y_{\rmU} - \mu_{\rmU, 0}(X)}}{e_0(0\mid X)} +  \mu_{\rmT, 0}(X) - \frac{1}{e_0(0\mid X)}\mu_{\rmU, 0}(X) + \frac{e_0(1\mid X)}{e_0(0\mid X)}\mu_{\rmT, 0}(X),
\end{align*}
Then,  $\widetilde{\Psi}^{\mathrm{cc}} \in \calT$ holds.

\section{Proof of Theorem~\ref{thm:asymp_norm_case_control}: : Semiparametric efficient ATE estimator under the case-control setting}
\label{appdx:normal_cc}
Recall that we have defined the ATE estimators as
\begin{align*}
    &\widehat{\tau}^{\mathrm{cc}\mathchar`-\mathrm{eff}}_{n} = \frac{1}{m}\sum^m_{j=1}S^{\mathrm{cc}~(\mathrm{T})}(X_j, Y_j; \widehat{\mu}^{(\ell)}_{\rmT, n}, \widehat{e}^{(\ell)}_{n}, \widehat{r}^{(\ell)}_{n}) + \frac{1}{l}\sum^l_{k=1}S^{\mathrm{cc}~(\mathrm{U})}(X_k, Y_k; \widehat{\mu}^{(\ell)}_{\rmT, n}, \widehat{\mu}^{(\ell)}_{\rmU, n}, \widehat{e}^{(\ell)}_{n}).
\end{align*}

We aim to show
\begin{align*}
    \sqrt{n}\p{\widehat{\tau}^{\mathrm{cc}\mathchar`-\mathrm{eff}}_{n} - \tau_0} \xrightarrow{\rmd} \mathcal{N}(0, V^{\mathrm{cc}})\quad \mathrm{as}\ n\to \infty.
\end{align*}

Recall that
\begin{align*}
    &S^{\mathrm{cc}~(\mathrm{T})}(X, Y(1); \widehat{\mu}^{(\ell)}_{\rmT, n}, \widehat{e}^{(\ell)}_{n}, \widehat{r}^{(\ell)}_{n})\\
    &\ \ \ = \p{1 - \frac{\widehat{e}^{(\ell)}_{n}(1\mid X)}{\widehat{e}^{(\ell)}_{n}(0\mid X)}}\Bigp{Y(1) - \widehat{\mu}^{(\ell)}_{\rmT, n}(X)}\widehat{r}^{(\ell)}_{n}(X),\\
    &S^{\mathrm{cc}~(\mathrm{U})}(X, Y_{\rmU}; \widehat{\mu}^{(\ell)}_{\rmT, n}, \widehat{\mu}^{(\ell)}_{\rmU, n}, \widehat{e}^{(\ell)}_{n})\\
    &\ \ \ =  \frac{\Bigp{Y_{\rmU} - \widehat{\mu}^{(\ell)}_{\rmU, n}(X)}}{\widehat{e}^{(\ell)}_{n}(0\mid X)} +  \widehat{\mu}^{(\ell)}_{\rmT, n}(X) - \frac{1}{\widehat{e}^{(\ell)}_{n}(0\mid X)}\widehat{\mu}^{(\ell)}_{\rmU, n}(X) + \frac{\widehat{e}^{(\ell)}_{n}(1\mid X)}{\widehat{e}^{(\ell)}_{n}(0\mid X)}\widehat{\mu}^{(\ell)}_{\rmT, n}(X).
\end{align*}

We have
\begin{align*}
    \widehat{\tau}^{\mathrm{cc}\mathchar`-\mathrm{eff}}_n &= \frac{1}{m}\sum^m_{j=1}S^{\mathrm{cc}~(\mathrm{T})}(X_j, Y_j; \widehat{\mu}^{(\ell)}_{\rmT, n}, \widehat{e}^{(\ell)}_{n}, \widehat{r}^{(\ell)}_{n}) + \frac{1}{l}\sum^l_{k=1}S^{\mathrm{cc}~(\mathrm{U})}(X_k, Y_k; \widehat{\mu}^{(\ell)}_{\rmU, n}, \widehat{e}^{(\ell)}_{n})\\
    &= \frac{1}{m}\sum^m_{j=1}S^{\mathrm{cc}~(\mathrm{T})}(X_j, Y_j; \mu_{\rmT, 0}, e_{0}, r_{0}) + \frac{1}{l}\sum^l_{k=1}S^{\mathrm{cc}~(\mathrm{U})}(X_k, Y_k; \mu_{\rmU, 0}, e_{0})\\
    &\ \ \ - \frac{1}{m}\sum^m_{j=1}S^{\mathrm{cc}~(\mathrm{T})}(X_j, Y_j; \mu_{\rmT, 0}, e_{0}, r_{0}) - \frac{1}{l}\sum^l_{k=1}S^{\mathrm{cc}~(\mathrm{U})}(X_k, Y_k; \mu_{\rmU, 0}, e_{0})\\
    &\ \ \ + \frac{1}{m}\sum^m_{j=1}S^{\mathrm{cc}~(\mathrm{T})}(X_j, Y_j; \widehat{\mu}^{(\ell)}_{\rmT, n}, \widehat{e}^{(\ell)}_{n}, \widehat{r}^{(\ell)}_{n}) + \frac{1}{l}\sum^l_{k=1}S^{\mathrm{cc}~(\mathrm{U})}(X_k, Y_k; \widehat{\mu}^{(\ell)}_{\rmU, n}, \widehat{e}^{(\ell)}_{n}).
\end{align*}

Here, if it holds that
\begin{align}
\label{eq:target_main_cens_cc1}
    &\frac{1}{m}\sum^m_{j=1}S^{\mathrm{cc}~(\mathrm{T})}(X_j, Y_j; \widehat{\mu}^{(\ell)}_{\rmT, n}, \widehat{e}^{(\ell)}_{n}, \widehat{r}^{(\ell)}_{n}) - \frac{1}{m}\sum^m_{j=1}S^{\mathrm{cc}~(\mathrm{T})}(X_j, Y_j; \mu_{\rmT, 0}, e_{0}, r_{0}) = o_p(1/\sqrt{m}),\\
    \label{eq:target_main_cens_cc2}
    &\frac{1}{l}\sum^l_{k=1}S^{\mathrm{cc}~(\mathrm{U})}(X_k, Y_k; \widehat{\mu}^{(\ell)}_{\rmU, n}, \widehat{e}^{(\ell)}_{n}) - \frac{1}{l}\sum^l_{k=1}S^{\mathrm{cc}~(\mathrm{U})}(X_k, Y_k; \mu_{\rmU, 0}, e_{0}) = o_p(1/\sqrt{l}). 
\end{align}
then we have
\begin{align*}
    &\sqrt{n}\Bigp{\widehat{\tau}^{\mathrm{cc}\mathchar`-\mathrm{eff}}_n - \tau_0}\\
    &=  \sqrt{n}\frac{1}{m}\sum^m_{j=1}S^{\mathrm{cc}~(\mathrm{T})}(X_j, Y_j; \mu_{\rmT, 0}, e_{0}, r_{0}) + \sqrt{n}\frac{1}{l}\sum^l_{k=1}S^{\mathrm{cc}~(\mathrm{U})}(X_k, Y_k; \mu_{\rmU, 0}, e_{0}) + o_p(1)\\
    &=  \frac{1}{\sqrt{\alpha m}}\sum^m_{j=1}S^{\mathrm{cc}~(\mathrm{T})}(X_j, Y_j; \mu_{\rmT, 0}, e_{0}, r_{0}) + \frac{1}{\sqrt{(1-\alpha)l}}\sum^l_{k=1}S^{\mathrm{cc}~(\mathrm{U})}(X_k, Y_k; \mu_{\rmU, 0}, e_{0}) + o_p(1)\\
    &\xrightarrow{\rmd} \mathcal{N}(0, V^{\mathrm{cc}}),
\end{align*}
from the central limit theorem for i.i.d. random variables. 

Therefore, we prove Theorem~\ref{thm:asymp_norm_case_control} by establishing \eqref{eq:target_main_cens_cc1} and \eqref{eq:target_main_cens_cc2}. These inequalities can be proved in the same manner as the proof of Theorem~\ref{thm:asymp_normal_censoring} and the analysis of double machine learning under the stratified scheme presented in \citet{Uehara2020offpolicy}. Since the procedure is nearly identical, we omit further details.

\begin{table}[t]
    \centering
    \caption{Experimental results. The upper and lower tables are results in the censoring and case-control settings, respectively.}
    \label{tab:table_exp2}
    
\begin{tabular}{|l|rrr|rrr|}
\hline
\multirow{2}{*}{Censoring} & IPW & DM & Efficient & IPW & DM  & Efficient \\
 & \multicolumn{3}{|c|}{(estimated $g_0$)} & \multicolumn{3}{|c|}{(true $g_0$)}  \\
\hline
MSE & 6.86 & 0.51 & 0.28 & 2.30 & 0.17 & 0.21 \\
Bias & -1.60 & 0.40 & 0.22 & 0.33 & 0.10 & 0.04 \\
Cov. ratio & 0.81 & 0.18 & 0.76 & 0.96 & 0.29 & 0.94 \\
\hline
\end{tabular}

\begin{tabular}{|l|rrr|rrr|}
\hline
Case- & IPW & DM & Efficient & IPW & DM  & Efficient \\
control & \multicolumn{3}{|c|}{(estimated $e_0$)} & \multicolumn{3}{|c|}{(true $e_0$)} \\
\hline
MSE & 1.06 & 0.09 & 0.10 & 0.35 & 0.03 & 0.03 \\
Bias & -0.03 & 0.19 & 0.18 & -0.00 & -0.01 & -0.01 \\
Cov. ratio & 0.93 & 0.40 & 0.61 & 0.97 & 0.77 & 0.91 \\
\hline
\end{tabular}
\centering
\caption{Experimental results. The upper and lower tables are results in the censoring and case-control settings, respectively.}
\label{tab:table_exp3}
\begin{tabular}{|l|rrr|rrr|}
\hline
\multirow{2}{*}{Censoring} & IPW & DM & Efficient & IPW & DM  & Efficient \\
 & \multicolumn{3}{|c|}{(estimated $g_0$)} & \multicolumn{3}{|c|}{(true $g_0$)}  \\
\hline
MSE & 5.03 & 0.23 & 0.13 & 1.25 & 0.07 & 0.09 \\
Bias & -1.32 & 0.24 & 0.18 & 0.17 & 0.07 & 0.04 \\
Cov. ratio & 0.91 & 0.22 & 0.82 & 0.99 & 0.34 & 0.98 \\
\hline
\end{tabular}
\vspace{3mm}
\begin{tabular}{|l|rrr|rrr|}
\hline
Case- & IPW & DM & Efficient & IPW & DM  & Efficient \\
control & \multicolumn{3}{|c|}{(estimated $e_0$)} & \multicolumn{3}{|c|}{(true $e_0$)} \\
\hline
MSE & 0.40 & 0.03 & 0.03 & 0.23 & 0.01 & 0.01 \\
Bias & -0.09 & 0.10 & 0.11 & 0.00 & -0.02 & -0.00 \\
Cov. ratio & 0.99 & 0.69 & 0.82 & 0.99 & 0.92 & 0.98 \\
\hline
\end{tabular}
\end{table}

\begin{figure}[t]
    \centering
    \includegraphics[width=0.6\linewidth]{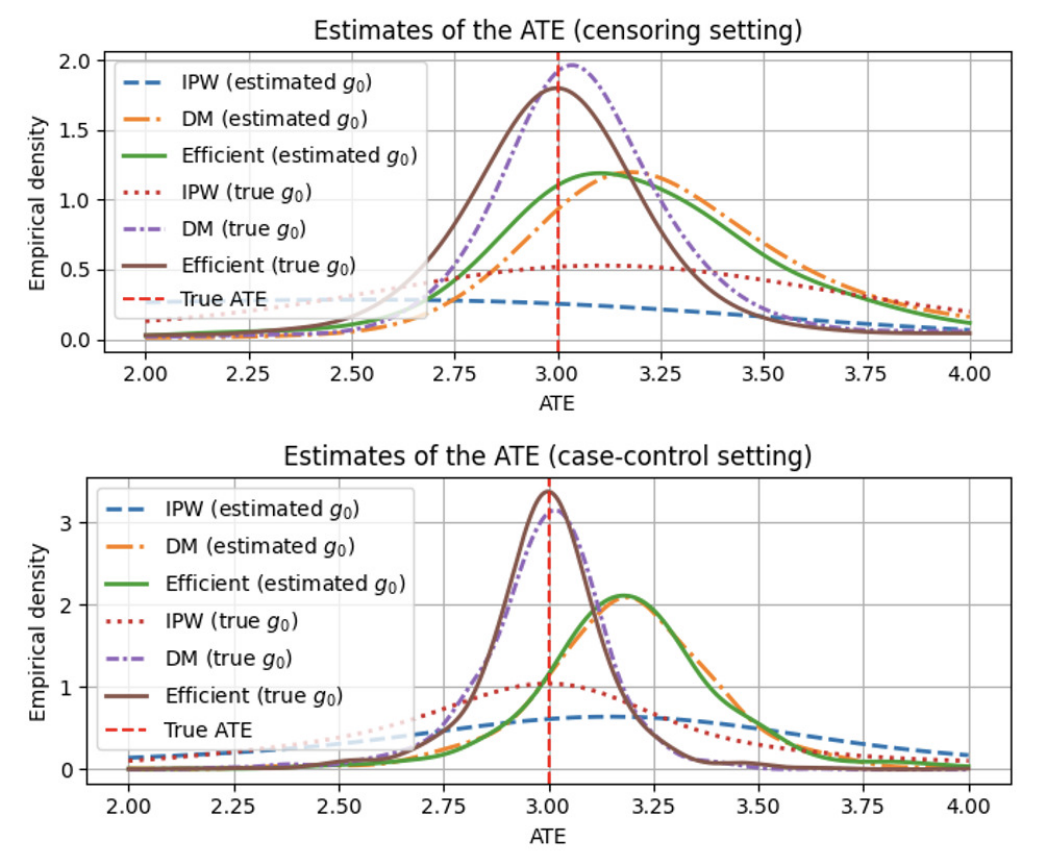}
    \caption{Empirical distributions of ATE estimates.}
    \label{fig:figure_exp2}
    \vspace{5mm}
        \centering
    \includegraphics[width=0.6\linewidth]{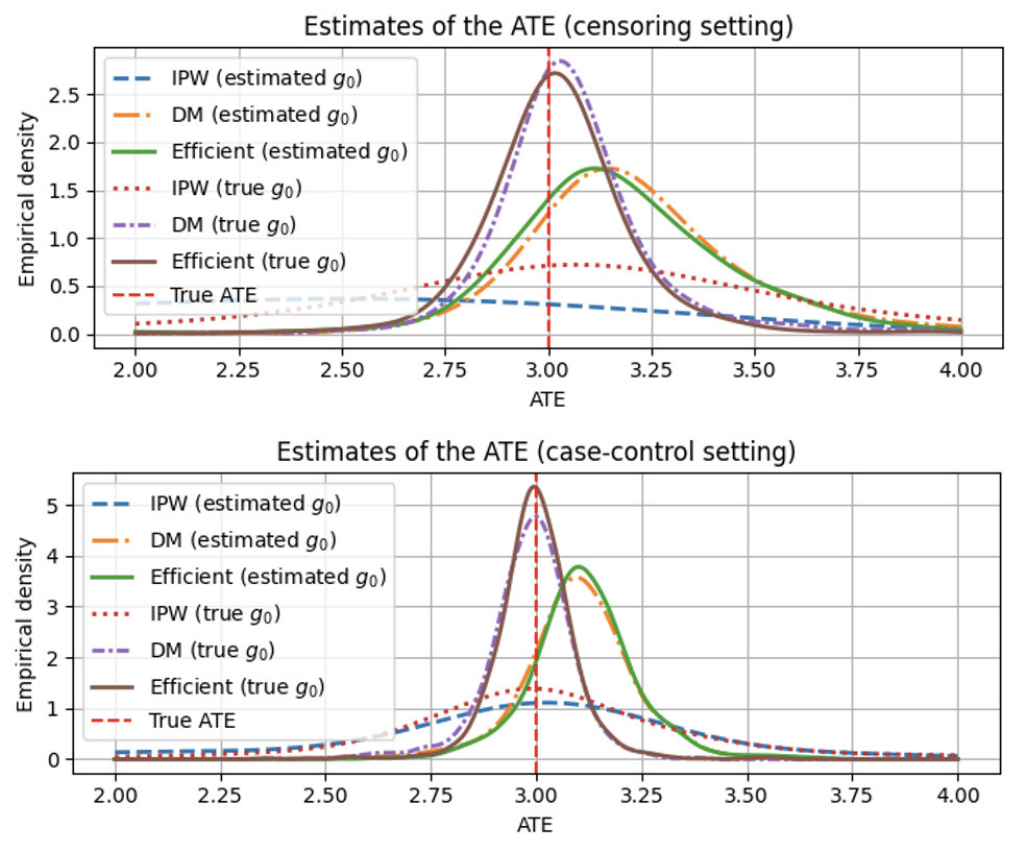}
    \caption{Empirical distributions of ATE estimates.}
    \label{fig:figure_exp3}
\end{figure}

\begin{table}[t]
    \centering
     \begin{minipage}[t]{0.48\textwidth}
     \centering
    \begin{tabular}{lrrr}
\hline
Censoring & IPW & DM & Efficient \\
    \hline
    MSE & 297.34 & 6.38 & 5.19 \\
    Bias & -16.36 & -0.58 & 0.56 \\
    Cov. ratio & 0.00 & 0.10 & 0.22 \\
    \hline
    \end{tabular}
    \end{minipage}
\hfill
 \begin{minipage}[t]{0.48\textwidth}
 \centering
    \begin{tabular}{|l|rrr|}
\hline
Case-control & IPW & DM & Efficient \\
\midrule
MSE & 26.58 & 1.18 & 1.49 \\
Bias & 2.36 & 0.52 & 0.77 \\
Cov. ratio & 0.42 & 0.29 & 0.40 \\
\hline
\end{tabular}
\end{minipage}
\vspace{3mm}
    \caption{Response surface~A. Left: censoring setting; Right: case‐control setting.}
    \label{tab:surfaceA}
\end{table}

\begin{table}[t]
    \centering
\begin{minipage}[t]{0.48\textwidth}
\centering
\begin{tabular}{|l|rrr|}
\hline
Censoring & IPW & DM & Efficient \\
\hline
MSE & 327.49 & 4.15 & 1.14 \\
Bias & -17.52 & -1.58 & -0.28 \\
Cov. ratio & 0.00 & 0.00 & 0.01 \\
\hline
\end{tabular}
    \end{minipage}
\hfill
 \begin{minipage}[t]{0.48\textwidth}
 \centering
\begin{tabular}{|l|rrr|}
\hline
Case-control & IPW & DM & Efficient \\
\hline
MSE & 46.15 & 3.34 & 3.77 \\
Bias & 2.66 & 0.41 & 0.93 \\
Cov. ratio & 0.42 & 0.21 & 0.43 \\
\hline
\end{tabular}
\end{minipage}
\vspace{3mm}
    \caption{Response surface~B. Left: censoring setting; Right: case‐control setting.}
    \label{tab:surfaceB}
\end{table}

\begin{figure}[t]
    \centering
    \includegraphics[width=0.9\linewidth]{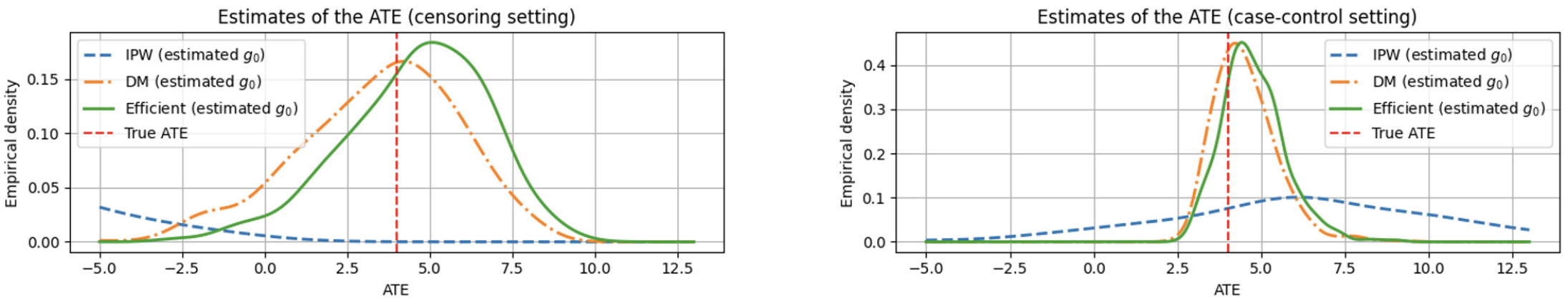}
    \caption{Response surface~A. Left: censoring setting; Right: case‐control setting.}
    \label{fig:surfaceA}
    \vspace{3mm}
    \centering
    \includegraphics[width=0.9\linewidth]{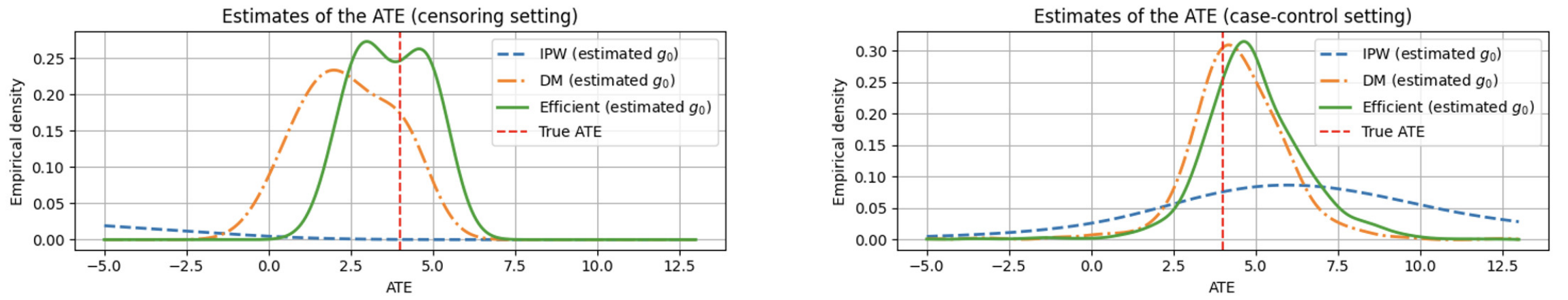}
    \caption{Response surface~B. Left: censoring setting; Right: case‐control setting.}
    \label{fig:surfaceB}
\end{figure}

\section{Additional results of the simulation studies}
\label{appdx:exp}
This section investigates the case in which the expected outcomes and propensity scores follow non-linear models.

All experiments were conducted on a Mac computer equipped with an Apple M2 processor and 24 GB of RAM.

\subsection{Censoring setting}
We generate synthetic data under the censoring setting, where the covariates $X$ are drawn from a multivariate normal distribution as $X \sim \zeta_0(x)$, where $\zeta_0(x)$ is the density of $\mathcal{N}(0, I_p)$, $p$ denotes the number of covariates, and $I_p$ is the $(p \times p)$ identity matrix. We set $p = 10$. The propensity score is given by $g_0(1\mid X) = \mathrm{sigmoid}(X^\top \beta_1 + X^{2\top} \beta_2)$, where $X^2$ is the element-wise square of $X$, and $\beta_1$ and $\beta_2$ are coefficient vectors sampled from $\mathcal{N}(0, 0.5I_{2p})$. The treatment indicator $D$ is sampled according to the propensity score.

The observation indicator $O$ is generated such that $O_i = 1$ with probability $c$ if $D_i = 1$, and $O_i = 0$ if $D_i = 0$, where $c$ is drawn from a uniform distribution over $[0, 1]$ before the experiment begins. The outcome is generated as $Y = (X^\top \beta)^2 + 1.1 + \tau_0 \cdot D + \varepsilon$, where $\varepsilon \sim \mathcal{N}(0, 1)$ and we set $\tau_0 = 3$.

The nuisance parameters are estimated using three-layer perceptrons with hidden layers of 100 nodes. The convergence rates satisfy Assumption~\ref{asm:conv_rate} under standard conditions \citep{SchmidtHieber2020nonparametricregression}. We compare our proposed estimator, $\widehat{\tau}^{\mathrm{cens}\mathchar`-\mathrm{eff}}_n$, with two alternative estimators: the IPW estimator $\widehat{\tau}^{\mathrm{cens}\mathchar`-\mathrm{IPW}}_n$ and the DM estimator $\widehat{\tau}^{\mathrm{cens}\mathchar`-\mathrm{DM}}_n$, as defined in Remarks~\ref{rem:IPW} and \ref{rem:DM}, respectively. Note that all of these estimators are proposed by us, and our objective is not to demonstrate that $\widehat{\tau}^{\mathrm{cens}\mathchar`-\mathrm{eff}}_n$ outperforms the others, although we recommend it in practice. We consider both cases in which the propensity score is either estimated using the method proposed by \citet{Elkan2008learningclassifiers} or assumed to be known.

We set $n = 3000$. We conduct 5000 trials and report the empirical mean squared errors (MSEs) and biases for the true ATE, as well as the coverage ratio computed from the confidence intervals in Table~\ref{tab:table_exp2} for $n = 3000$ and Table~\ref{tab:table_exp3} for $n = 500$. We also present the empirical distributions of the ATE estimates in Figure~\ref{fig:figure_exp2} for $n = 3000$ and Figure~\ref{fig:figure_exp3} for $n = 5000$.

\subsection{Case-control setting}
In the case-control setting, covariates for the treatment and unknown groups are generated from $p = 3$-dimensional normal distributions: $X_{\rmT} \sim \zeta_{\rmT, 0}(x)$ and $X \sim \zeta_0(x) = e_0(1)\zeta_{\rmT, 0}(x) + e_0(0)\zeta_{\rmC}(x)$, where $\zeta_{\rmT, 0}(x)$ and $\zeta_{\rmC}(x)$ are the densities of the normal distributions $\mathcal{N}(\mu_p\bm{1}_p, I_p)$ and $\mathcal{N}(\mu_n\bm{1}_p, I_p)$, respectively. We set $\mu_p = 0.5$, $\mu_n = 0$, and $\bm{1}_p = (1\ 1\ \cdots\ 1)^\top$. The class prior is set as $e_0(1) = 0.3$. By definition, the propensity score $e_0(d \mid x)$ is given by $e_0(1 \mid x) = e_0(1) \zeta_{\rmT, 0}(x) / \zeta_0(x)$. The outcome is generated in the same manner as in the censoring setting: $Y = (X^\top \beta)^2 + 1.1 + \tau_0 D + \varepsilon$, where $\tau_0 = 3$.

Since we use neural networks, we estimate the propensity score $g_0$ using the non-negative PU learning method proposed by \citet{Kiryo2017positiveunlabeledlearning}, which is designed to mitigate overfitting when neural networks are applied. For simplicity, we assume that the class prior $e_0(1)$ is known.

We consider two cases: $(m, l) = (1000, 2000)$ and $(2000, 3000)$. We conduct 5000 trials and report the empirical mean squared errors (MSEs) and biases for the true ATE, along with the coverage ratio computed from the confidence intervals in Table~\ref{tab:table_exp2} for $(m, l) = (1000, 2000)$ and Table~\ref{tab:table_exp3} for $(m, l) = (2000, 3000)$. We also present the empirical distributions of the ATE estimates in Figure~\ref{fig:figure_exp2} for $(m, l) = (1000, 2000)$ and Figure~\ref{fig:figure_exp3} for $(m, l) = (2000, 3000)$.

\section{Empirical analysis using semi-synthetic data}
\label{sec:semi-synthetic}
In this section, we investigate the empirical performance of our estimators using the Infant Health and Development Program (IHDP) dataset. The dataset contains simulated outcomes paired with covariates observed in the real world \citep{Hill2011bayesiannonparametric}.

\subsection{Dataset.}
The sample size is 747, and the covariates consist of 6 continuous variables and 19 binary variables.

\citet{Hill2011bayesiannonparametric} considers two scenarios for the outcome models: response surface~A and response surface~B. Response surface~A generates the potential outcomes $Y_t(1)$ and $Y_t(0)$ according to the following model:
\begin{align*}
    Y_t(0) &\sim \mathcal{N}(X^\top_{t}\bm{\gamma}_A, 1),\\
    Y_t(1) &\sim \mathcal{N}(X^\top_{t}\bm{\gamma}_A + 4, 1),
\end{align*}
where each element of $\bm{\gamma}_A \in \mathbb{R}^{25}$ is randomly drawn from $\{0, 1, 2, 3, 4\}$ with probabilities $(0.5, 0.2, 0.15, 0.1, 0.05)$.

In contrast, response surface~B generates the potential outcomes $Y_t(1)$ and $Y_t(0)$ as follows:
\begin{align*}
    Y_t(0) &\sim \mathcal{N}\left(\exp\left((X_{t} + W)^\top\bm{\gamma}_B\right), 1\right),\\
    Y_t(1) &\sim \mathcal{N}(X^\top_{t}\bm{\gamma}_B - q, 1),
\end{align*}
where $W$ is an offset matrix of the same dimension as $X_t$ with all elements equal to 0.5, $q$ is a constant chosen to normalize the average treatment effect conditional on $d = 1$ to be 4, and each element of $\bm{\gamma}_B \in \mathbb{R}^{25}$ is randomly drawn from $\{0, 0.1, 0.2, 0.3, 0.4\}$ with probabilities $(0.6, 0.1, 0.1, 0.1, 0.1)$.

\subsection{Censoring setting}
We first investigate the censoring setting. The other experimental setups are identical to those in Section~\ref{sec:experiments}. Given $\{(X_i, D_i, Y_i)\}$ from the IHDP dataset, we generate the observation indicator $O$ from a Bernoulli distribution with probability 0.1 if $D_i = 1$, and set $O_i = 0$ if $D_i = 0$. The nuisance parameters are estimated using linear regression and (linear) logistic regression.

We compare our proposed estimator, $\widehat{\tau}^{\mathrm{cens}\mathchar`-\mathrm{eff}}_n$, with two other candidates: the IPW estimator $\widehat{\tau}^{\mathrm{cens}\mathchar`-\mathrm{IPW}}_n$ and the DM estimator $\widehat{\tau}^{\mathrm{cens}\mathchar`-\mathrm{DM}}_n$, as defined in Remarks~\ref{rem:IPW} and \ref{rem:DM}, respectively. All of these estimators are proposed by us. Our aim is not to demonstrate that $\widehat{\tau}^{\mathrm{cens}\mathchar`-\mathrm{eff}}_n$ strictly outperforms the others, although we recommend it in practice. Unlike in Section~\ref{sec:experiments}, we only consider the case in which the propensity score is estimated using the method proposed by \citet{Elkan2008learningclassifiers}.

For each outcome model (response surface~A and B), we conduct 1000 trials and report the empirical mean squared errors (MSEs), biases for the true ATE, and the coverage ratio (Cov. ratio) computed from the confidence intervals in Tables~\ref{tab:surfaceA} and \ref{tab:surfaceB}. We also present the empirical distributions of the ATE estimates in Figures~\ref{fig:surfaceA} and \ref{fig:surfaceB}.

\subsection{Case-control setting}
\label{sec:sim_casecontrol_ihdp}
In the case-control setting, we randomly split the dataset $\calD$ into two subsets. One is used as an unlabeled dataset, and the other is used as a positive dataset by selecting only the treated units from it. The class prior is set as $e_0(1) = 0.1$.

For each outcome model (response surface~A and B), we conduct 1000 trials and report the empirical mean squared errors (MSEs), biases for the true ATE, and the coverage ratio (Cov. ratio) computed from the confidence intervals in Tables~\ref{tab:surfaceA} and \ref{tab:surfaceB}. We also present the empirical distributions of the ATE estimates in Figures~\ref{fig:surfaceA} and \ref{fig:surfaceB}.